%% file: aaai21_conservative_bandit_main.tex
\def\year{2021}\relax
\documentclass[letterpaper]{article} 
\usepackage{aaai21}  
\usepackage{times}  
\usepackage{helvet} 
\usepackage{courier}  
\usepackage[hyphens]{url}  
\usepackage{graphicx} 
\usepackage{natbib}  
\usepackage{caption} 
\usepackage[switch]{lineno}
\usepackage{amsthm}
\newtheorem{theorem}{Theorem}
\newtheorem{lemma}{Lemma}

\newtheorem{fact}{Fact}
\usepackage{amsmath,amssymb}
\usepackage{amsfonts}
\usepackage{bm}
\usepackage{enumerate}
\usepackage{flushend}
\usepackage{mathrsfs}
\usepackage{subfigure}
\usepackage{booktabs}
\usepackage{makecell}
\usepackage{setspace}
\usepackage{ifthen}
\usepackage{xcolor}
\usepackage[ruled,vlined,linesnumbered]{algorithm2e}
\newcommand{\cX}{\mathcal{X}}
\newcommand{\cA}{\mathcal{A}}
\newcommand{\cR}{\mathcal{R}}
\newcommand{\cS}{\mathcal{S}}
\newcommand{\cE}{\mathcal{E}}
\newcommand{\cF}{\mathcal{F}}
\newcommand{\cC}{\mathcal{C}}
\newcommand{\argmax}{\operatornamewithlimits{argmax}}

\newcommand{\mv}{\textup{MV}}

\newcommand{\ex}{\mathbb{E}}

\newcommand{\mbr}[1]{\left[ #1 \right]}

\allowdisplaybreaks

\newcommand{\compilehidecomments}{true}
\ifthenelse{ \equal{\compilehidecomments}{true} }{%
	\newcommand{\longbo}[1]{}
	\newcommand{\siwei}[1]{}
	\newcommand{\yihan}[1]{}
}{
	\newcommand{\longbo}[1]{{\color{red}  [\text{Longbo:} #1]}}
	\newcommand{\siwei}[1]{{\color{blue} [\text{Siwei:} #1]}}
	\newcommand{\yihan}[1]{{\color{teal} [\text{Yihan:} #1]}}
}

\newcommand{\compilefullversion}{false}
\ifthenelse{\equal{\compilefullversion}{false}}{
	\newcommand{\OnlyInFull}[1]{}
	\newcommand{\OnlyInShort}[1]{#1}
}{
	\newcommand{\OnlyInFull}[1]{#1}
	\newcommand{\OnlyInShort}[1]{}
}

\title{A One-Size-Fits-All Solution to Conservative Bandit Problems}
\author {
        Yihan Du,\textsuperscript{\rm 1}
        Siwei Wang,\textsuperscript{\rm 1}
        Longbo Huang\textsuperscript{\rm 1} \\
}
\affiliations {
    \textsuperscript{\rm 1} Tsinghua University \\
    duyh18@mails.tsinghua.edu.cn, \{wangsw2020, longbohuang\}@tsinghua.edu.cn
}
\begin{document}
\maketitle

\begin{abstract}
In this paper, we study a family of conservative bandit problems (CBPs) with \emph{sample-path} reward constraints, i.e., the learner's reward performance must be at least as well as a given baseline at any time. We propose a general one-size-fits-all solution to CBPs and present its applications to three encompassed problems, i.e., conservative multi-armed bandits (CMAB), conservative linear bandits (CLB) and conservative contextual combinatorial bandits (CCCB). Different from previous works which consider high probability constraints on the expected reward, our algorithms guarantee sample-path constraints on the actual received reward, and achieve better theoretical guarantees ($T$-independent additive regrets instead of $T$-dependent) and empirical performance. 
Furthermore, we extend the results and consider a novel  conservative \emph{mean-variance} bandit problem (MV-CBP), which measures the learning performance in both the expected reward and variability. We design a novel algorithm with $O(1/T)$ normalized additive regrets ($T$-independent in the cumulative form) and validate this result through empirical evaluation. 
\end{abstract}

\section{Introduction}

\begin{table*}[t]
	\renewcommand{\arraystretch}{1.6}
	\centering
	\begin{tabular}{|c|c|c|c|c|}
		\hline
		Problem&Algorithm&Regret bound&Type\\
		\hline
		CMAB&\textsf{GenCB-CMAB} (ours)& $O  \big(H \ln T +  \frac{  H }{\alpha}  [ \ln  ( \frac{  H}{\alpha}  )  ]^2   \big)$&E \\
		\hline
		CMAB&\textsf{ConUCB}~\cite{wu2016conservative}& $O  \big( H \ln  ( \frac{T}{\delta}  ) +  \sum_{i=1}^{K} \frac{1}{ \alpha \max \{\Delta_i, \Delta_0-\Delta_i\}  } \ln  ( \bm{T}/ \delta ) \big)$&H \\
		\hline
		CMAB&Lower Bound~\cite{wu2016conservative}& $O  \big( \max  \{ \frac{K}{\alpha }, \sqrt{KT}  \}  \big)$&E \\
		\hline
		\hline
		CLB&\textsf{GenCB-CLB} (ours)& $O  \big( d \ln  ( T  ) \sqrt{T} + \frac{  d^2   }{\alpha}  [ \ln  ( \frac{  d  }{\alpha}  )  ]^2   \big)$&E \\
		\hline
		CLB&\textsf{CLUCB}~\cite{kazerouni2017conservative}& $O  \big( d \ln  ( \frac{T}{\delta}  ) \sqrt{T} + \frac{  d^2   }{\alpha}  [ \ln  ( \frac{  d  }{ \alpha \bm{\delta}  }   )  ]^2   \big)$&H \\
		\hline
		CLB&\textsf{CLUCB2}~\cite{ImprovedConLB_AAAI20}& $O  \big( d \ln  ( \frac{T}{\delta}  ) \sqrt{T} + \frac{  d^2   }{\alpha^2 }  [ \ln  ( \frac{  d }{  \alpha \bm{\delta}  }   )  ]^2   \big)$&H \\
		\hline
		\hline
		CCCB&\textsf{GenCB-CCCB} (ours)& $O  \big( d\ln(KT)\sqrt{T}  +  \frac{  (K+d)^2 }{ \alpha }   [ \ln  ( \frac{ K+d  } {\alpha}  )   ]^2   \big) $&E \\
		\hline
		CCCB&\textsf{C3UCB}~\cite{zhang2019contextual}& $O \big( d\ln(KT)\sqrt{T}  +  \frac{ d }{ \alpha } \sqrt{\frac{d}{K}\ln(\frac{K}{ \delta}\bm{T} )} \big) $&H \\
		\hline
	\end{tabular}
	\caption{Comparison of regret bounds for CBPs. ``Type'' refers to the type of regret bounds. ``E'' and ``H'' denote the expected and high probability bounds, respectively. Here  $H=\sum_{i=1}^{K}\Delta_i^{-1}$. $d$ is the dimension in CLB and CCCB. For high probability bounds,  the convention in the bandit literature is to choose $\delta=1/T$.  Note that our formulation focuses on a sample-path reward constraint, while the other results consider the constraints on the expected reward.}  \label{table:comparison_regret}
\end{table*}

The multi-armed bandit (MAB) problem~\cite{thompson1933,UCB_auer2002}
is a classic online learning model that characterizes the exploration-exploitation trade-off in sequential decision making. While existing bandit algorithms achieve satisfactory regret bounds over the whole learning processes, they can perform wildly and lose much in the initial exploratory phase. This limitation has hindered their applications in  real-world scenarios such as health sciences, marketing and finance, where it is important to guarantee safe and smooth algorithm behavior in initialization.  Hence, studying bandit problems with safe (conservative) exploration contributes to solving this issue.

In this paper, we study the conservative bandit problems (CBPs) with \emph{sample-path} reward constraints. Specifically, a learner is given a set of  regular arms and a  default arm.
At each timestep, the learner chooses a regular arm or the default arm to play and receives a reward according to the played arm. 
The learning's objective is to minimize the expected cumulative regret (equivalently, maximize the expected cumulative reward), while ensuring that the received cumulative reward must stay above a fixed percentage of what one can obtain by always playing the default arm.



CBPs have extensive real-world applications including recommendation systems, company operation and finance. For instance, in finance, investors are offered various financial products including the fixed-income security such as bank deposit (default arm), and the fluctuating equity securities such as stocks (regular arms). While the fixed-income security is a safe and reasonable option, investors want to find better choices to earn higher returns. Meanwhile, compared to the returns they can obtain by simply depositing the money, investors do not want to lose too much when exploring other investment choices. CBPs provide an effective model for such exploration-exploitation trade-off with the safe exploration guarantees.

We propose a general one-size-fits-all solution $\textsf{GenCB}$ for CBPs, and present its applications to three important CBP problems, i.e., conservative multi-armed bandits (CMAB), conservative linear bandits (CLB) and conservative contextual combinatorial bandits (CCCB). We provide theoretical analysis and empirical evaluations for these algorithms, and show that our algorithms outperform existing ones both theoretically and empirically. Table~\ref{table:comparison_regret} presents the comparison of regret bounds between our algorithms and existing 
ones. 
In the table, each regret term contains two components, the first component incurred  by regular arms and the second term due to playing the default arm. One can see that our algorithms possess better regret guarantees. 
%
Moreover, unlike existing algorithms that only provide high probability bounds with $T$-dependent conservative regrets, we not only obtain expected bounds but also have $T$-independent conservative regrets. 
%
\yihan{This paragraph needs a pass.}

Our work distinguishes itself from previous conservative bandit works, e.g., ~\cite{wu2016conservative,kazerouni2017conservative,ImprovedConLB_AAAI20,zhang2019contextual} in two aspects: (i) Previous works consider high probability guarantees on the expected reward. Such models cannot directly handle many risk-adverse tasks, e.g., a start-up does not wish to tolerate any failure probability to reach the basic earning under the debt, or an asset management company must perform better than the promised return. While one can choose a very small 
$\delta$ in previous  algorithms to provide high-probability guarantees, the $\ln (1/ \delta)$-dependent regrets will boost accordingly. Instead, we focus on a certainty  (sample-path) guarantee  on the actual empirical reward. Doing so ensures safe exploration (our regret bounds do not contain $\delta$) and better suits such tasks. 
(ii) Our problem formulation, solution and analysis offer a general framework for studying a family of CBPs, including  CMAB~\cite{wu2016conservative}, CLB~\cite{kazerouni2017conservative,ImprovedConLB_AAAI20} and CCCB~\cite{zhang2019contextual}. Moreover, our algorithms achieve better theoretical and empirical performance than previous schemes.


We also extend our results to the mean-variance setting \cite{Markowitz_1952,sani2012risk}, called \emph{conservative mean-variance bandit problem} (MV-CBP), which focuses on the balance between the expected reward and variability with safe exploration. 
%
Different from the typical CBPs which only consider the expected reward into learning performance, MV-CBP takes into account both the mean and variance of the arms, and 
is more suitable for practical tasks that are sensitive to reward fluctuations, e.g., clinical trials and finance. For example, many risk-adverse investors prefer stable assets (e.g., bonds) with satisfactory returns than volatile assets (e.g., derivatives) with high returns, and they do not want to suffer wild fluctuations when exploring different financial products. 

Note that the mean-variance regret in MV-CBP (formally defined in Eq.~\eqref{eq:mv_pseudo_regret} in next section) consists not only the gap of \emph{mean-variance} (a combination of both measures) between the played arms and the optimal arm, but also an additional variance for playing arms with different means, called \emph{exploration risk}, which requires alternative techniques beyond those in typical CBPs. To tackle this issue, we carefully adapt our solution and analysis for the CBPs and make nontrivial extensions. Our results offer new insight into algorithm design for mean-variance bandit problems.

Our contributions are summarized as follows. 
\begin{itemize}
	\item We study a family of CBPs with sample-path reward constraints, which encompasses previously studied CMAB~\cite{wu2016conservative}, CLB~\cite{kazerouni2017conservative,ImprovedConLB_AAAI20} and CCCB~\cite{kazerouni2017conservative}. We propose a general one-size-fits-all solution \textsf{GenCB} for CBPs, which can translate a standard bandit algorithm into a conservative bandit algorithm and achieve better ($T$-independent conservative regret rather than $T$-dependent) theoretical regret bounds than previous works in the three specific problems.
	\yihan{Clarified this item in our contributions.}
	
	\item We extend the conservative bandit formulation to a novel conservative mean-variance bandit setting, which characterizes the trade-off between the expected reward and variability. We propose an algorithm, \textsf{MV-CUCB}, and prove that it achieves an $O(1/T)$ normalized additive regret for the extended problem. 
	
	\item We conduct extensive experiments for the considered problems. The results match our theoretical bounds and demonstrate that our algorithms achieve the performance superiority compared to existing algorithms. 
\end{itemize}

\subsection{Related Work}
\textbf{Conservative Bandit Literature.}
Recently, there are several works~\cite{wu2016conservative,kazerouni2017conservative,zhang2019contextual,ImprovedConLB_AAAI20} studying bandit problems with conservative exploration constraints. Under the constraints on the expected rewards, \cite{wu2016conservative} propose an algorithm \text{ConUCB} for CMAB. \cite{kazerouni2017conservative} design an algorithm \text{CLUCB} for CLB and \cite{ImprovedConLB_AAAI20} further propose an improved algorithm \text{CLUCB2}. \cite{zhang2019contextual} present an algorithm \text{C3UCB} for CCCB. Under the stage-wise constraints, \cite{Safe_Linear_Stochastic_Bandits_Khezeli_AAAI20} restrict the expected reward at any timestep to stay above a given baseline. \cite{LB_SafetyConstraints_NIPS2019} confine the played arm at any timestep to stay in a given safe set. Under the interleaving constraint, \cite{Interleaving_AISTATS19} require the chosen action at any timestep to perform better than the default action when interleaving in the combinatorial semi-bandit setting. \cite{bubeck2013bounded_regret} study the standard $K$-armed bandit problem with knowledge of the highest expected reward and the smallest gap, \cite{Thresholding_Bandit_Problem_locatelli16} consider the  thresholding pure exploration problem, and the settings and methods in both works are different from ours.   

\textbf{Mean-variance Bandit Literature.}
\cite{sani2012risk} open the mean-variance bandit literature which considers both the expected reward and variability into performance measures, and a series of follow-ups \cite{risk-averse_ALT2013,risk-averse_CDC2019,risk_convex2019} have emerged recently. To our best knowledge, this paper is the first to study the mean-variance bandit problem with conservative exploration.

\section{Problem Formulation}
In this section, we first review previous standard (non-conservative) bandit problems (SBPs) and then give the formulation of the  Conservative Bandit Problems (CBPs). 

\textbf{Standard Bandit Problems (SBPs).} In a standard bandit problem,  a learner is given a set of arms $\cX$, where each arm $x\in\cX$ has an unknown reward distribution in $[0,1]$ with mean of $\mu_{x}$. Each arm $x$ at timestep $t$ has a random reward $r_{t,x}=\mu_{x}+\eta_{t,x}$, where $\eta_{t,x}$ is an independent random noise with respect to $t$. At each timestep $t$, the learner plays an arm $x_t$ and only observes the reward $r_{t,x_t}$ of the chosen arm. 
Let $x_*=\argmax_{x \in \cX} \mu_{x}$ denote the optimal arm.  
The learning performance over a time horizon $T$ is measured by \emph{expected cumulative regret} 
\begin{align}
\mathbb{E}[\cR_T]= \mu_{x_*} T - \mathbb{E}\left[\sum_{t=1}^{T}\mu_{x_t}\right] = \sum_{x \neq x_*} \mathbb{E}[N_{x}(T)] \Delta_x, \label{eq:regret_mean}
\end{align}
where $\Delta_x=\mu_{x_*} - \mu_{x}$ and $N_{x}(T)$ is the number of times arm $x$ was played over time $T$. The regret characterizes the loss due to not always playing the optimal arm. The goal of standard bandit algorithms is to minimize Eq.~\eqref{eq:regret_mean}.

\textbf{Conservative Bandit Problems (CBPs).}
The CBPs provide an alternative default arm $x_0$ to play. In this case, since playing $x_0$ is a default (baseline) policy that the learner is familiar with, for ease of analysis we assume that $x_0$ has a known constant reward $0<\mu_0<\mu_{x_*}$ as previous works~\cite{wu2016conservative,kazerouni2017conservative,zhang2019contextual} do.\footnote{This assumption can be relaxed to that $x_0$ has a random reward within a known interval $[r_0^\ell, r_0^h]$ ($r_0^\ell>0$) by sightly changing the right-hand-side of the $\mathtt{if}$ statements in our algorithms, and our analysis procedure still works. While previous works can remove this assumption by estimating $\mu_0$, this is due to that their constraints are imposed on the expected reward.}



%
Then, during the learning process, the learner is required to ensure that the cumulative reward under the chosen policy is lower bounded by a fraction of the reward from always pulling the default arm. Specifically, given a parameter  $\alpha \in (0,1)$, for any timestep $t$,  the learner's cumulative empirical reward should be least $1-\alpha$ fraction of the reward of always playing $x_0$, i.e., 
\begin{align}
\sum_{s=1}^{t} r_{s,x_s} \geq (1-\alpha) \mu_0 t , \quad \forall t \in \{1, \dots, T\} . \label{eq:perf_constraint}
\end{align}
Here  $\alpha$ controls the strictness of the constraint, i.e., how conservative we want the leaner to behave, and can be viewed as the weight we place on safety in exploration.  
The goal of conservative bandit algorithms is to minimize the expected cumulative regret (Eq.~\eqref{eq:regret_mean}) while satisfying the reward constraint (Eq.~\eqref{eq:perf_constraint}). 

We note that constraint (\ref{eq:perf_constraint}) is a \emph{sample-path} reward constraint, which is different from the high-probability constraints on the expected reward  in prior works~\cite{wu2016conservative,kazerouni2017conservative,ImprovedConLB_AAAI20,zhang2019contextual}. This setting is particularly useful when the practical tasks cannot tolerate higher losses than the baseline with certainty, e.g., health care and investment. On the other hand, it also imposes new challenges in algorithm design and regret analysis. \yihan{Changed the detailed expression to this one.}

Our formulation is a general framework which encompasses various bandit problems from the prospective of conservative exploration. For example, in CMAB which studies a conservative version of the classic $K$-armed bandit problem~\cite{thompson1933,UCB_auer2002,agrawal2012analysis}, $\cX=[K]$ and $\mu_x$ is an arbitrary value.\footnote{$[K] \overset{\textup{def}}{=} \{1, \dots, K\}$.} In CLB which considers the linear bandit problem~\cite{Dani2008_linear_bandit,linear_bandit_NIPS2011} with  conservative exploration, $\cX$ is a compact subset of $\mathbb{R}^d$ and each arm $x \in \mathbb{R}^d$ has an expected reward $\mu_x=x^\top \theta^*$, where $\theta^* \in \mathbb{R}^d$ is an unknown parameter. In CCCB which investigates the contextual combinatorial bandit problem~\cite{Contextual_Combinatorial_Bandit} with the safe exploration requirement, there is a set of base arms $[K]$ and $\cX$ is a collection of subsets of base arms, which represents certain combinatorial structure (e.g., matchings and paths). For each $x \in \cX$, $\mu_x$ is associated with the expected rewards of its containing base arms. We will analyze the CBPs under specific bandit settings in the next section.

\section{A General Solution to Conservative Bandits}

In this section, we first present a  general solution for CBPs, and its regret analysis. Then, we present its applications to three specific problems, i.e., CMAB, CLB and CCCB, and show that in all three cases, our algorithm achieves tighter bounds than  existing algorithms. 

Algorithm \ref{alg:general_solution}  illustrates the proposed solution to CBPs, called \textsf{GenCB}, which offers a general scheme for translating a standard non-conservative bandit algorithm $\cA_{S}$ into a  conservative bandit algorithm. 
In the algorithm,  $m$ denotes the time horizon of $\cA_{S}$, and the number of times we play the regular arms, $r_S(t)$ denotes the cumulative reward from sampling regular arms, and $N_0(t)$ denotes the number of times $x_0$ is played up to time $t$. 

The main idea of \textsf{GenCB} is to play regular arms as much as possible while ensuring the sample-path reward constraint in the worst case, since playing the default arm cannot provide any information for identifying the optimal arm.
At each  time, \textsf{GenCB} checks if playing a regular arm can satisfy the sample-path reward constraint in the worst case (this pull feedbacks zero reward). 
If it can, we play a regular arm $x_t$ according to $\cA_{S}$, observe reward $r_{t,x_t}$ and update the statistical information. Otherwise, we choose the default arm. 

Different from previous conservative algorithms \cite{wu2016conservative,kazerouni2017conservative,zhang2019contextual}, \textsf{GenCB} guarantees the  constraint with certainty rather than with high probability, and \textsf{GenCB} uses the received cumulative reward rather than the lower confidence bound to check the constraint. Doing so makes our algorithm less conservative and boosts its empirical performance significantly (see Section~\ref{sec:experiments} for empirical comparisons). 

Next, we present the regret analysis for  \textsf{GenCB}.  Note that, the regret for CBPs can be decomposed into (i) the regret incurred by regular arms, and (ii) the regret due to playing the default arm, i.e., conservative regret. Since the analysis of the former is similar to that in SBPs, as in the conservative bandit literature~\cite{wu2016conservative,kazerouni2017conservative,ImprovedConLB_AAAI20,zhang2019contextual}, we mainly focus the conservative regret. 
We remark that our analysis is different from those in prior works, and can be applied to several specific CBPs including CMAB, CLB and CCCB. 
We give the regret bound of \textsf{GenCB} as follows.

\begin{algorithm}[t]
	\caption{General Solution to Conservative Bandits (\textsf{GenCB})} \label{alg:general_solution}
	\KwIn{Standard bandit problem and algorithm $\cA_{S}$, regular arms $\cX$, default arm $x_0$ with reward $\mu_0$, parameter $\alpha$.}
	$\forall t \geq 0, N_0(t) \leftarrow 0, r_S(t) \leftarrow 0$. $m \leftarrow 0$\;
	\For{$t=1, 2, \dots$}
	{
		\If{$r_S(t-1)+N_0(t-1)\mu_0 \geq (1-\alpha) \mu_0 t$} 
		{
			$m \leftarrow m+1$\;
			Play an arm $x_t$ according to $\cA_{S}$, observe $r_{t,x_t}$ and update the statistical information\;
			$N_0(t) \leftarrow N_0(t-1)$\;
			$r_S(t) \leftarrow r_S(t-1)+r_{t,x_t}$\;
			
		}
		\Else
		{
			Play $x_0$ and receive reward $\mu_0$\;
			$N_0(t) \leftarrow N_0(t-1)+1$\;
			$r_S(t) \leftarrow r_S(t-1)$\;
			
		}
	}
\end{algorithm}

\begin{theorem} \label{thm:general}
	Given a standard bandit problem and a corresponding algorithm $\cA_{S}$ with sublinear regret $\mathbb{E}[\cR_T(\cA_{S})] \leq B(T)$, \textsf{GenCB} (Algorithm~\ref{alg:general_solution}) guarantees the sample-path reward constraint Eq.~\eqref{eq:perf_constraint} and achieves a regret bound 
	$$ \mathbb{E}[\cR_T(\textsf{GenCB})] \leq B(T)+C \Delta_0,$$ 
	where $C$ is a problem-specific constant independent of $T$ and $\Delta_0=\mu_{x_*}-\mu_{x_0}$.
\end{theorem}
\begin{proof}\longbo{skip the proof for now}
	First, it can be seen from the algorithm that the   sample-path reward constraint Eq.~\eqref{eq:perf_constraint} can be guaranteed. Next, we prove the regret bound of \textsf{GenCB}. 
	\yihan{Removed the proof of the satisfaction of the constraint.}
	We use $\cS_t$ to denote the set of timesteps up to time $t$ during which we play regular arms and use $m_t$ to denote its size. 
	Let $\tau$ denote the last timestep  we play $x_0$, 
	i.e., $\tau$ is the last timestep such that $r_S(\tau-1)+N_0(\tau-1)\mu_0 < (1-\alpha) \mu_0 \tau$ holds. Rearranging the terms, and subtracting $(1-\alpha)\mu_0N_0(\tau-1)$ from  both sides (note that $\tau=N_0(\tau-1)+m_{\tau-1}+1$), we have
	\begin{align}
	\nonumber\!\!	\alpha \mu_0 N_0(\tau-1) < & (1-\alpha) \mu_0 (m_{\tau-1}+1) - r_S(\tau-1)
		\\ 
	\nonumber	= & (1-\alpha) \mu_0 (m_{\tau-1}+1) \\ & + \sum_{t \in \cS_{\tau-1}} (\mu_{x_t}  - r_{t,x_t}) - \sum_{t \in \cS_{\tau-1}} \mu_{x_t} . 
	\label{eq:proof_1}
	\end{align}
	$\sum_{t \in \cS_{\tau-1}} (\mu_{x_t}  - r_{t,x_t})$ is the deviation between the sum of $m_{\tau-1}$ sample results and their means. Using the Azuma-Hoeffding inequality, $\sum_{t \in \cS_{\tau-1}} (\mu_{x_t}  - r_{t,x_t})$ can be upper bounded by $F \sqrt{m_{\tau-1} \ln(m_{\tau-1})}$ with high confidence, for fixed $m_{\tau-1}$ and some constant $F$ that varies in different settings. 
	Then, by probabilistic calculations, we can have $\ex[\sum_{t \in \cS_{\tau-1}} (\mu_{x_t}  - r_{t,x_t})] \leq \ex[F\sqrt{m_{\tau-1} \ln(m_{\tau-1})}] + 1$. Taking expectation on both sides of Eq. \eqref{eq:proof_1}, setting $m=\ex[m_{\tau-1}+1]$ and replacing $\ex[\sum_{t \in \cS_{\tau-1}} \mu_{x_t} ]$ with $\mu_{x_*}\ex[m_{\tau-1}]-\ex[\cR_{m_{\tau-1}}(\cA_{S})]$, we have
	\begin{align*}
	 \alpha \mu_0 \ex[N_0(\tau-1)] < & -(\Delta_0+\alpha \mu_0) m  + \ex[B(m_{\tau-1})] \\ & +  \ex[F \sqrt{m_{\tau-1} \ln(m_{\tau-1})}]  + 1 
	\\
	 \overset{\textup{(a)}}{<} &   -(\Delta_0+\alpha \mu_0) m + B(m) + 2 \\ & \quad + F \sqrt{m \ln(\ex[N_0(\tau-1)]+m) } ,
	\end{align*}
	where (a) comes from Jensen's inequality. 
	Note that since $B(m)$ and $F \sqrt{m \ln(\ex[N_0(\tau-1)]+m) }$ are sublinear with respect to $m$, for any $m\ge 2$, the right-hand-side can be  upper bounded by $G [\ln(\sqrt{\ex[N_0(\tau-1)]})]^2 $ where $G$ is a constant factor that only depends on problem parameters. 
	\yihan{Revised this sentence to explain $G$.}
	Then, we obtain $\ex[N_0(\tau-1)] \leq  \frac{G}{\alpha \mu_0} [\ln (\frac{G}{\alpha \mu_0})]^2 $. Thus, $\ex[N_0(T)]=\ex[N_0(\tau)]=\ex[N_0(\tau-1)]+1 \leq C$, where $C\triangleq\frac{G}{\alpha \mu_0} [\ln (\frac{G}{\alpha \mu_0})]^2+1$ is independent of $T$. 
	
	Combining the regrets for $\cA_{S}$ and $x_0$, we obtain that $\ex[\cR_T(\textsf{GenCB})] 
		\leq  B(T) + C \Delta_0$.
\end{proof}

\textbf{Remark 1.} Theorem \ref{thm:general} shows that \textsf{GenCB} provides a general algorithmic and analytical framework for  translating a standard bandit problem into a conservative bandit algorithm, and only generate an additional $T$-independent regret due to the reward constraint. 
To the best of our knowledge, this is the first general analysis procedure which works for a family of CBPs with sample-path reward constraints, and it provides an expected regret bound (rather than high probability bounds in \cite{wu2016conservative,kazerouni2017conservative,ImprovedConLB_AAAI20,zhang2019contextual}) with $T$-independent conservative regret.
\yihan{Revised this paragraph. Should we merge this paragraph to the Remark1?}

Below, we apply \textsf{GenCB} to three widely studied CBPs, i.e., CMAB, CLB and CCCB. 
Here we only present the main theorems, and defer the algorithm pseudo-codes and proofs to the supplementary material~\cite{supp_arxiv_version}.

\subsection{Application to Conservative Multi-Armed Bandits (CMAB)}
The conservative multi-armed bandit (CMAB) problem is a variation of the classic $K$-armed bandit model with conservative exploration~\cite{wu2016conservative}, which has extensive applications including clinical trials, online advertising and wireless network. In CMAB,  $\cX=[K]$ and $\mu_{i}$ ($1 \leq i \leq K$) can be an arbitrary value. Without lose of generality, we assume $\mu_1\geq\mu_2\geq\cdot\cdot\cdot\geq\mu_K$ and denote $\mu_{*}\triangleq\mu_1$. 

We apply the \textsf{GenCB} algorithm with the \textsf{UCB} algorithm~\cite{UCB_auer2002} to this setting, by replacing Line 5 in Algorithm~\ref{alg:general_solution} with 
$x_t \leftarrow \argmax_{i \in [K]} \left(\hat{\mu}_i + \sqrt{ 2 \ln m/ N_i(t-1)  } \right)$,
where $\hat{\mu}_i$ is the reward empirical mean for arm $i$, and name this version of the algorithm \textsf{GenCB-CMAB}.


The main idea of \textsf{GenCB-CMAB} is to 
play the arm with the maximum upper confidence bound whenever the reward constraint is satisfied (otherwise we play the default arm).
The regret bound for \textsf{GenCB-CMAB} is summarized below. 

\begin{theorem} \label{thm:gen_con_ucb}
	For the conservative multi-armed bandit problem, \textsf{GenCB-CMAB} guarantees the sample-path reward constraint Eq.~\eqref{eq:perf_constraint} and achieves the regret bound 
	$$ \! O \!\! \left( \! H \! \ln T \!+\!  \frac{  H \Delta_0}{\alpha \mu_0 (\Delta_0+ \alpha \mu_0)} \!\! \left[ \ln \! \left( \! \frac{  H}{\alpha \mu_0 (\Delta_0+ \alpha \mu_0)} \right) \right]^2  \right) \!\!,$$
	where $H=\sum_{i>1} \Delta_i^{-1}$.
\end{theorem}

\textbf{Remark 2.} The first term owes to playing the regular arms, which is similar to the result in standard MAB~\cite{UCB_auer2002}, and the second term is caused by the default arm, i.e., the conservative regret, which is the main focus in conservative bandit study.  Compared to the existing algorithm \textsf{ConUCB} \cite{wu2016conservative},
\textsf{GenCB-CMAB} only incurs a $T$-independent conservative regret rather than $\ln T$ (see Table~\ref{table:comparison_regret}). 
Our result also matches the regret lower bound derived in \cite{wu2016conservative}  for CMAB with expected reward constraints, which also holds for our sample-path reward constraint setting. 
\yihan{Explained that the lower bound also holds for our setting.}

\subsection{Application to Conservative Linear Bandits (CLB)}
The conservative linear bandit (CLB)~\cite{kazerouni2017conservative,ImprovedConLB_AAAI20} problem considers the linear bandit problem~\cite{Dani2008_linear_bandit,linear_bandit_NIPS2011} with safe exploration. In CLB where there is a linear structure among arms, $\cX$ is a compact subset of $\mathbb{R}^d$ and $\mu_x=x^{\top} \theta^*$, where $\theta^* \in \mathbb{R}^d$ is an unknown parameter. We make the common assumptions, i.e., $\|x\|_2 \leq L , \forall x \in \cX$ and $\|\theta^*\|_2 \leq S$, as previous linear bandit papers~\cite{Dani2008_linear_bandit,linear_bandit_NIPS2011,kazerouni2017conservative} do.

For CLB, we apply \textsf{GenCB} with the \textsf{LinUCB} algorithm~\cite{linear_bandit_NIPS2011} by replacing Line 5 in Algorithm~\ref{alg:general_solution} with $(x_t, \tilde{\theta}_t) \leftarrow \argmax_{(x, \theta) \in \cX \times \cC_t} x^{\top} \theta $. Here  $\cC_{t}=\{ \theta\in \mathbb{R}^d : \| \theta-\hat{\theta}_{t-1} \|_{V_{t-1}} \leq \sqrt{ d \ln ( 2 m^2 (1+mL^2/\lambda )  ) } + \sqrt{\lambda} S \}$ is a confidence ellipsoid that contains $\theta^*$ with high probability, and we define $\hat{\theta}_t= V_t^{-1} b_t$, $V_t=\lambda I + \sum_{s=1}^{t} x_s x_s^{\top}$, $b_t=\sum_{s=1}^{t}r_{s,x_s} x_s$ and $\lambda \geq \max\{1, L^2\}$.\footnote{$\|x\|_{V} \overset{\textup{def}}{=} \sqrt{x^{\top} V x}, \forall x \in \mathbb{R}^d, \forall V \in \mathbb{R}^{d \times d}$.} We name this version of the algorithm \textsf{GenCB-CLB}, whose  key idea   is to play a regular arm according to the optimism in the face of uncertainty principle while ensuring the sample-path reward constraint. 
Below, we have the regret bound of \textsf{GenCB-CLB}.  
\begin{theorem} \label{thm:gen_con_linucb}
	For the conservative linear bandit problem, \textsf{GenCB-CLB}  guarantees the sample-path reward constraint Eq.~\eqref{eq:perf_constraint} and has the regret bound 
	$$ O \! \left( \! d \ln \left( \frac{L T}{\lambda} \right) \sqrt{T} \! + \! \frac{  d^2 S^2 \lambda \Delta_0 }{\alpha \mu_0 \tilde{\Delta}_0} \!\! \left[ \ln \left( \frac{  d S \sqrt{\lambda} }{\alpha \mu_0 \tilde{\Delta}_0} \right) \right]^2  \right) \!\! , $$
	where $\tilde{\Delta}_0 = \Delta_0+ \alpha \mu_0$.
\end{theorem}
\textbf{Remark 3.} 
Similarly, the first term is aligned with the result in standard linear bandits~\cite{Dani2008_linear_bandit,linear_bandit_NIPS2011}, and the second term is the conservative regret due to the default arm.
While the existing algorithms $\textsf{CLUCB}$~\cite{kazerouni2017conservative} and $\textsf{CLUCB2}$~\cite{ImprovedConLB_AAAI20} have $\ln(1/ \delta)$-dependent conservative regrets with high probability (do not contain $T$ either), these results are of $\ln T$ order when making the convention $\delta=1/T$. In contrast, we provide an expected bound with a $T$-independent conservative regret.

\subsection{Application to Conservative Contextual Combinatorial Bandits (CCCB)}
The conservative contextual combinatorial bandit (CCCB) problem~\cite{zhang2019contextual}  investigates the contextual combinatorial bandit problem under the safe exploration requirement. 
In CCCB, $\cX$ is a collection of subsets of \emph{base arms} $x_1, \dots, x_K \in \mathbb{R}^d$ and generated from certain combinatorial structure (e.g., matchings and paths).
The learner plays a \emph{super arm} (subset of base arms) $A_t \in \cX$ or the default arm $x_0$ at each timestep.
The expected reward of base arm $x_e$ is $w_{e}^*=x_{e}^{\top} \theta^*$ and that of super arm $A$ is $f(A, \boldsymbol{w}^*)$, where $\theta^*$ is an unknown parameter and $f$ satisfies two mild assumptions, i.e., monotonicity and Lipschitz continuous with parameter $P$ \cite{Contextual_Combinatorial_Bandit,zhang2019contextual}.
Similar to CLB, we assume $\|x\|_2 \leq L , \forall x \in \cX$ and $\|\theta^*\|_2 \leq S$.
At timestep $t$, the random reward of a base arm $x_e$ and a super arm  $A$  are  $w_{t,e}=w_{e}^*  + \eta_{t,e} \in [0,1]$ and $r_{t,A}=f(A, \boldsymbol{w}^*)+\eta_{t,A} \in [0, K]$, respectively. After pulling  super arm $A_t$, we receive the random reward $r_{t,A_t}$ and observe a semi-bandit feedback, i.e., $w_{t,e}$ for each $e \in A_t$.

For CCCB, we apply \textsf{GenCB} with the \textsf{C2UCB} algorithm~\cite{Contextual_Combinatorial_Bandit}, by replacing Line 5 in Algorithm~\ref{alg:general_solution} with $A_t \leftarrow \argmax_{A \in \cX } f(A, \bar{\boldsymbol{w}}_t) $. Here  $\bar{w}_{t,e}= x_{e}^{\top} \hat{\theta}_{t-1}+ (\sqrt{ d \ln ( 2 m^2 (1+mKL^2/\lambda )  ) }+ \sqrt{\lambda} S) \|x_{e}\|_{V_{t-1}^{-1}}$ is the upper confidence bound of $w_{e}^*$, and we define $\hat{\theta}_t= V_t^{-1} b_t$,  $V_t=\lambda I + \sum_{s=1}^{t} \sum_{e \in A_s} x_{e} x_{e}^{\top}$, $b_t=\sum_{s=1}^{t} \sum_{e \in A_s} w_{s,e} x_{e}$ and $\lambda \geq \max\{1, L^2\}$. 
The key idea here is to play a super arm with the maximum upper confidence bound according to the historical observations on base arms. 
Theorem \ref{thm:gen_con_c2ucb} below gives the regret bound of \textsf{GenCB-CCCB}.


\begin{theorem} \label{thm:gen_con_c2ucb}
	For the contextual   combinatorial bandit problem, \textsf{GenCB-CCCB} ensures the sample-path reward constraint Eq.~\eqref{eq:perf_constraint} and achieves the regret bound 
	$$O \! \left( \! Pd \ln \left( \frac{KLT}{\lambda} \right)\sqrt{T} \! + \! \frac{  D^2 }{ \alpha \mu_0 \tilde{\Delta}_0 } \!\! \left[ \ln \left( \frac{ D  } {\alpha \mu_0 \tilde{\Delta}_0 } \right)  \right]^2  \right), $$
	where $D = K+P\sqrt{\lambda}Sd$ and $\tilde{\Delta}_0 = \Delta_0+ \alpha \mu_0$.
\end{theorem}
\textbf{Remark 4.} The first term is consistent with the result in standard contextual combinatorial bandits~\cite{Contextual_Combinatorial_Bandit}, and the second conservative regret term is due to playing the default arm. Compared to the state-of-the-art algorithm \textsf{C3UCB}~\cite{zhang2019contextual},
\textsf{GenCB-CCCB} provides a $T$-independent conservative regret, while \textsf{C3UCB} incurs a $\ln T$ regret (see Table~\ref{table:comparison_regret}).

\begin{algorithm}[t]
	\caption{\textsf{MV-CUCB}} \label{alg:ConMVUCB}
	\KwIn{Reugular arms $[K]$, default arm $x_0$ with $\mv_0=\rho \mu_0$, parameters $\alpha$, $\rho>\frac{2}{\alpha \mu_0}$.}
	$\forall t \geq 0, \forall 0 \leq i \leq K,  N_i(t) \leftarrow 0$. $m \leftarrow 0$. $\widehat{\mv}_{0}(\cA) \leftarrow 0$\;
	\For{$t=1, 2, \dots$}
	{
		\If{$ (t-1)\widehat{\mv}_{t-1}(\cA) - 2 \geq (1-\alpha) {\mv}_0 t$} 
		{   
			$m \leftarrow m+1$\;
			$x_t  \leftarrow  \argmax \limits_{i \in [K]} \! \left(\widehat{\mv}_i \!+\! (5 \!+\! \rho) \sqrt{ \frac{ \ln (12Km^3)}{ 2 N_i(t-1) } } \right)$\;
			Pull arm $x_t$, observe the random reward $r_{t,x_t}$ and update $\widehat{\mv}_{x_t}$\;
			$N_{x_t}(t) \leftarrow N_{x_t}(t-1) + 1 $ and $\forall 0 \leq i \leq K, i \neq x_t, N_{i}(t) \leftarrow N_{i}(t-1) $\;
		}
		\Else
		{
			Play $x_0$ and receive reward $\mu_0$\;
			$N_0(t) \leftarrow N_0(t-1)+1$ and $\forall 1 \leq i \leq K, N_{i}(t) \leftarrow N_{i}(t-1) $\;	
		}
	}
\end{algorithm}

\section{Conservative Mean-Variance Bandits}
We now extend CBPs to the mean-variance \cite{sani2012risk,risk-averse_ALT2013,risk_convex2019} setting (MV-CBP), which focuses on finding arms that achieve effective trade-off between the expected reward and variability. 
MV-CBP increments the typical conservative bandit model and better suits the tasks emphasizing on reward fluctuations. It also brings additional complications for algorithm design and regret analysis beyond \textsf{GenCB}.

\subsection{Problem Formulation for MV-CBP}
To introduce our MV-CBP formulation, we first review the standard mean-variance bandit setting~\cite{sani2012risk}.
Each arm $x \in [K]$ is associated with a measure mean-variance, which is formally defined as $\mv_x=\rho \mu_x - \sigma_x^2$, where $\sigma_x^2$ is the reward variance and $\rho$ is a weight parameter.
Let $x^{\mv}_*=\argmax_{x \in [K]} \mv_x$ denote the mean-variance optimal arm.
Given i.i.d. reward samples $\{Z_{x,s}\}_{s=1}^{t}$ of arm $x$, we define the \emph{empirical mean-variance} $\widehat{\mv}_{x,t}=\rho \hat{\mu}_{x,t}-\hat{\sigma}_{x,t}^2$, where $\hat{\mu}_{x,t}=\frac{1}{t}\sum_{s=1}^{t} Z_{x,s}$ and $\hat{\sigma}_{x,t}^2=\frac{1}{t} \sum_{s=1}^{t} (Z_{x,s}-\hat{\mu}_{x,t})^2$.

For an algorithm $\cA$ and its sample path $\{r_{t,x_t}\}_{t=1}^{T}$ over time horizon $T$, we define the empirical mean-variance $\widehat{\mv}_T(\cA)=\rho \hat{\mu}_{T}(\cA)-\hat{\sigma}_{T}^2(\cA)$, where $\hat{\mu}_{T}(\cA)=\frac{1}{T}\sum_{t=1}^{T} r_{t,x_t}$ and $\hat{\sigma}_{T}^2(\cA)=\frac{1}{T} \sum_{t=1}^{T} (r_{t,x_t}-\hat{\mu}_{T}(\cA))^2$.
Naturally, for algorithm $\cA$ over time $T$, we define the \emph{mean-variance regret} $\cR^{\mv}_T(\cA)=\widehat{\mv}_{x_*,T}-\widehat{\mv}_T(\cA)$, which is the difference of the mean-variance performance between $\cA$ and what we could have achieved by always playing $x^{\mv}_*$. 

Due to the difficulty of the $\cR^{\mv}_T(\cA)$ metric, we follow the mean-variance bandit literature and use a more tractable mesure  \emph{mean-variance pseudo-regret}~\cite{sani2012risk} defined as:
	\begin{align}
	\!\!\! \widetilde{\cR}^{\mv}_T(\cA) \!\!=\! \frac{1}{T} \!\!\! \sum_{x \neq x_*} \!\! N_{x,T} \Delta^{\mv}_{x} \!\!+\!\! \frac{2}{T^2} \!\! \sum_{x \in \cX}  \sum_{y \neq x} \!\! N_{x,T} N_{y,T} \Gamma_{x,y}^2, \label{eq:mv_pseudo_regret}
	\end{align} 
where $N_{x,T}$ is a shorthand for $N_{x}(T)$, $\Delta^{\mv}_{x}=\widehat{\mv}_{x_*}-\widehat{\mv}_{x}$ and $\Gamma_{x,y}=\mu_{x}-\mu_{y}$. 
It has been shown that any bound on $\widetilde{\cR}^{\mv}_T(\cA)$ immediately translates into an bound on $\cR^{\mv}_T(\cA)$ (Lemma 1 in \cite{sani2012risk}). Thus, most theoretical analysis \cite{sani2012risk,risk-averse_ALT2013,risk_convex2019} on  mean-variance bandits has been done via $\widetilde{\cR}^{\mv}_T(\cA)$. 
Note that, in MV-CBP the measures $\widehat{\mv}_T(\cA)$ and $\cR^{\mv}_T(\cA)$ are both \emph{normalized} quantities over $T$.

In addition to minimizing the regret, the learner is also required to guarantee the following mean-variance constraint: 
	\begin{align}
	\widehat{\mv}_t(\cA) \geq (1-\alpha) {\mv}_0 , \quad \forall t \in \{1, \dots, T\}. \label{eq:mv_perf_constraint}
	\end{align}
Here ${\mv}_0$ denotes the mean-variance of our default arm $x_0$ with known constant reward $\mu_0$ and zero variance. 
The goal in MV-CBP is to minimize Eq.~\eqref{eq:mv_pseudo_regret} while satisfying Eq.~\eqref{eq:mv_perf_constraint}.



\subsection{Algorithm for MV-CBP}
We propose a novel algorithm named \textsf{MV-CUCB} for MV-CBP (illustrated in Algorithm \ref{alg:ConMVUCB}). 
The main idea is to  compute the upper confidence bound of mean-variance for each arm and select one according to the optimism principle whenever the constraint is not violated. 
Theorem \ref{thm:con_mv_ucb} summarizes the performance results of \textsf{MV-CUCB} (see the supplementary material~\cite{supp_arxiv_version} for its proof).
\begin{theorem} \label{thm:con_mv_ucb}
	For the conservative mean-variance multi-armed bandit problem with $\alpha \mv_0>2$, \textsf{MV-CUCB} (Algorithm~\ref{alg:ConMVUCB}) ensures the mean-variance constraint Eq.~\eqref{eq:mv_perf_constraint} and achieves the following regret bound:\footnote{$\tilde{O}$ omits the logarithmic terms that are independent of $T$.}
	\begin{align*}
		\tilde{O} & \Bigg(  \frac{\rho^2 \ln (KT)}{T} \left(H_1+H_2+ \frac{\rho^2 \ln (KT)}{T} H_3 \right) \\ & \!\!\! + \! \frac{  \rho^3 K^2 (H_1^{\mv} \!\!\!+\! 4H_2^{\mv}) \!+\! ( \rho^4 K H_3^{\mv} \!\!\!+\! \rho K)\tilde{\Delta}_0^{\mv}  }{ (\alpha {\mv}_0-2) \tilde{\Delta}_0^{\mv} T } \!\cdot\! \Delta_0^{\mv} \! \Bigg) ,
	\end{align*}
	where $H_1^{\mv}=\sum_{i>1}(\Delta_i^{\mv})^{-1}$, $H_2^{\mv}=\sum_{i>1}(\Delta_i^{\mv})^{-2}$, $H_3^{\mv}=\sum_{i>1}\sum_{j>1, j \neq i}(\Delta_i^{\mv} \Delta_j^{\mv})^{-2}$ and $\tilde{\Delta}_0^{\mv}=\Delta_0^{\mv} + \alpha {\mv}_0$.
\end{theorem}

\textbf{Remark 5.} 
Since a pull of $x_0$ not only accumulates $\mv_0$ but also causes an exploration risk  (bounded by $2$ for reward distributions in $[0,1]$) due to the switch between different arms, we need the mild assumption $\alpha \mv_0>2$ to guarantee that a pull of $x_0$ will not violate the constraint.
Recall that the result is a normalized regret over $T$, the first term of $O(\ln T/ T)$ order owes to regular arms, which agrees with the previous non-conservative mean-variance result~\cite{sani2012risk}. The second term is the conservative regret for satisfying the constraint, which is of only $O(1/T)$ order and independent of $T$ in the cumulative form.  
To our best knowledge, Theorem~\ref{thm:con_mv_ucb} is the first result for conservative bandits with mean-variance objectives.

\section{Experiments} \label{sec:experiments}
We conduct experiments for our algorithms in four problems, i.e., CMAB, CLB, CCCB and MV-CBP, with a wide range of parameter settings. Due to space limit, only partial results are presented here (see the supplementary material~\cite{supp_arxiv_version} for full results).

In all experiments, we assume the rewards to take  i.i.d. Bernoulli values. \yihan{Revised this sentence.}
For CMAB, we set $K \in \{24, 72, 144\}$, $\alpha \in \{0.05, 0.1, 0.15\}$, $\mu_0=0.7$ and $\mu_1, \dots, \mu_K$ as an arithmetic sequence from $0.8$ to $0.2$. For CLB and CCCB, we set $d \in \{5, 7, 9\}$, $\alpha \in \{0.01, 0.02, 0.03\}$, $K=2d$ and $f(A, \boldsymbol{w}^*)=\sum_{e \in A} w_e^*$.
For MV-CBP, we use the same parameter settings as CMAB and additionally set $\rho \in \{10, 30, 60\}$.
For each algorithm, we perform $50$ independent runs and present the average (middle curve), maximum (upper curve) and minimum (bottom curve) cumulative regrets across runs. For each figure, we also zoom in the initial exploratory phase in the sub-figure to compare algorithm performance in this phase. \yihan{Added this sentence.}

\textbf{Experiments for CBPs.}
In the experiments for CMAB (Figure~\ref{fig:mab}), CLB (Figure~\ref{fig:lb}) and CCCB (Figure~\ref{fig:clb}), we compare  \textsf{GenCB-CMAB}, \textsf{GenCB-CLB} and \textsf{GenCB-CCCB} to previous CBP algorithms \textsf{CUCB} \cite{wu2016conservative}, \textsf{CLUCB} \cite{kazerouni2017conservative} and \textsf{C3UCB} \cite{zhang2019contextual}, the standard bandit algorithms \textsf{UCB} \cite{UCB_auer2002}, \textsf{LinUCB} \cite{linear_bandit_NIPS2011} and \textsf{C2UCB} \cite{Contextual_Combinatorial_Bandit}, and the conservative baseline $(1-\alpha)\mu_0$, respectively. 

We see that, in  the exploration phase, existing non-conservative algorithms suffer higher losses than the baseline, while our algorithms and previous CBP algorithms achieve similar performance as (or better than) the baseline due to the conservative constraints. \yihan{Revised this sentence.} 
However, since previous CBP algorithms use lower confidence bounds (rather than the empirical  rewards in ours) to check the constraints, they  are forced to play the default arm more  and act more conservatively compared to ours. 

In the exploitation phase, when compared to non-conservative algorithms, our algorithms have additional regrets that keep constant as $T$ increases, which matches our $T$-independent conservative regret bounds. 
Compared to previous CBP algorithms, our schemes achieve significantly better performance, since we play the default arm less and enjoy a lower conservative regret.

\textbf{Experiments for MV-CBP.}
In the experiments for MV-CBP (Figure~\ref{fig:mv}), we present the mean-variance regret in the cumulative form $T \cdot \widetilde{\cR}^{\mv}_T$ for clarity of comparison. 
Since \textsf{MV-CUCB} is the first algorithm for MV-CBP, we compare it with the standard mean-variance bandit algorithm \textsf{MV-UCB} and the baseline $(1-\alpha)\mv_0 T$. We can see that, in the exploration phase, \textsf{MV-UCB} suffers from a higher regret than the baseline while \textsf{MV-CUCB} follows the baseline closely. One also sees that \textsf{MV-CUCB} achieves this with only an additional constant overall regret compared to \textsf{MV-UCB}, which matches our $T$-independent bound of conservative regret.


\begin{figure} 
	\centering    
	\subfigure[CMAB ($K=72, \alpha=0.05$)] {
		\label{fig:mab}     
		\includegraphics[width=0.46\columnwidth]{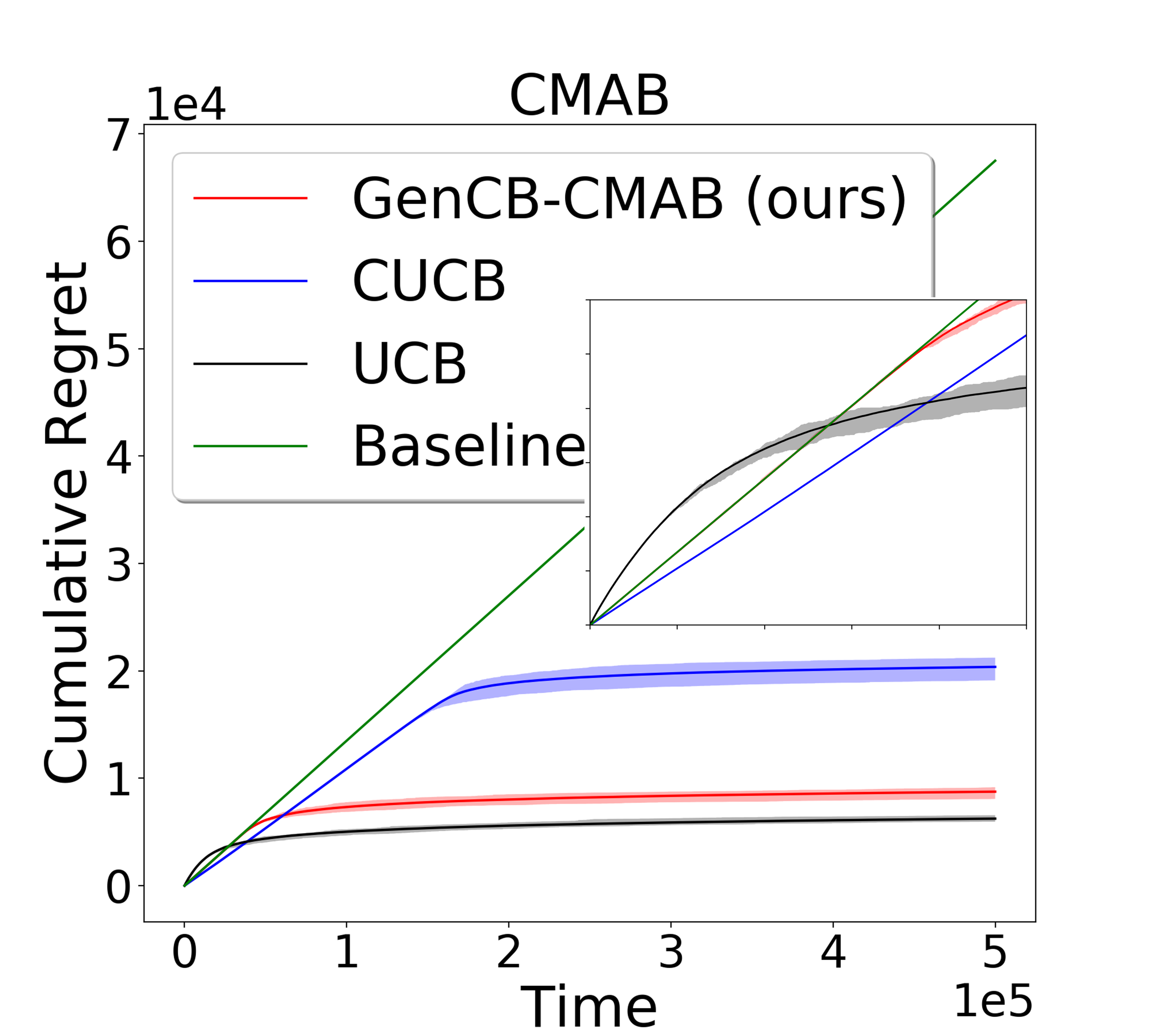}  
	}     
	\subfigure[CLB ($d=7, \alpha=0.01$)] { 
		\label{fig:lb}     
		\includegraphics[width=0.46\columnwidth]{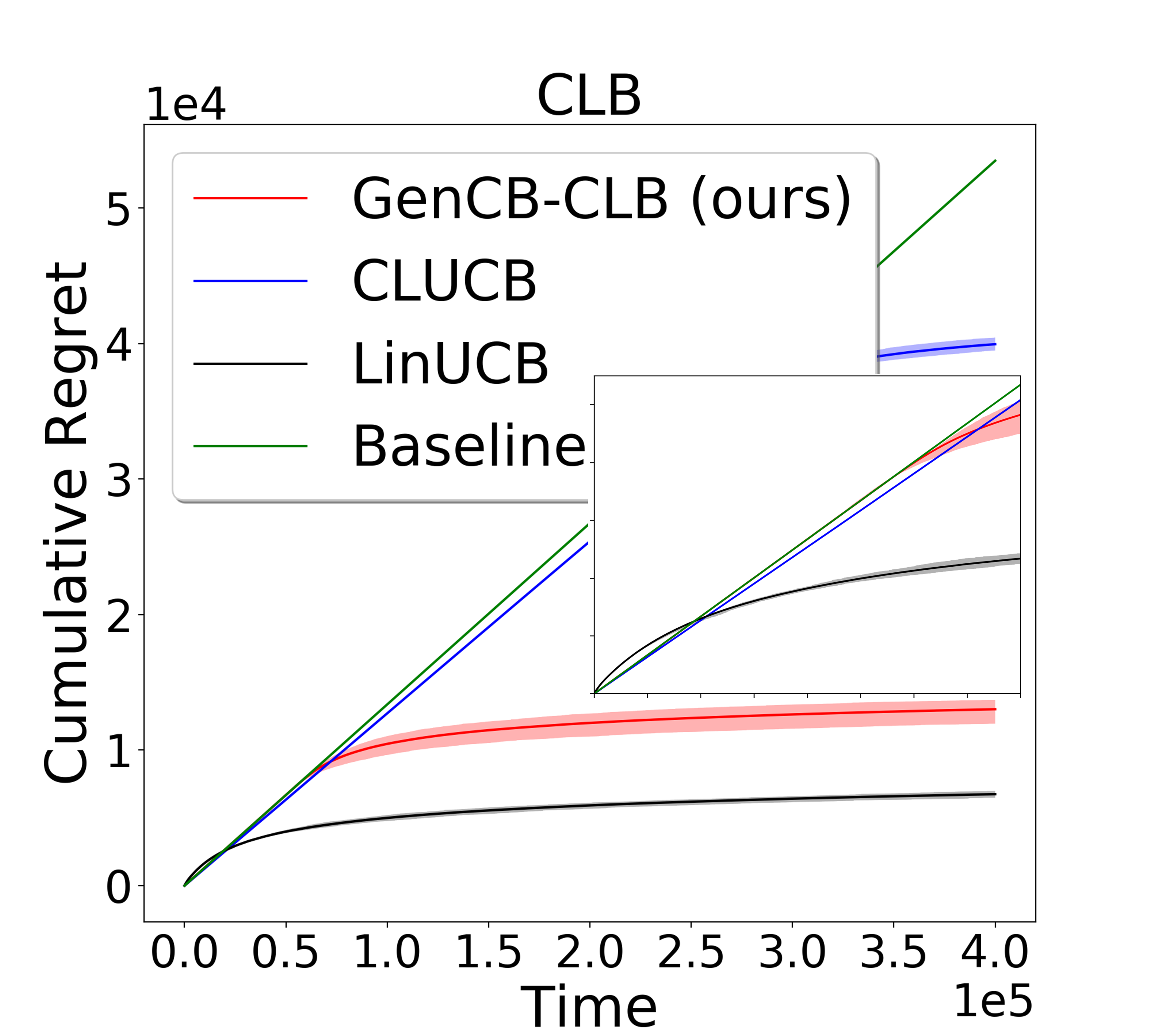} 
	}    
	\subfigure[CCCB ($d=7, \alpha=0.01$)] { 
		\label{fig:clb}     
		\includegraphics[width=0.46\columnwidth]{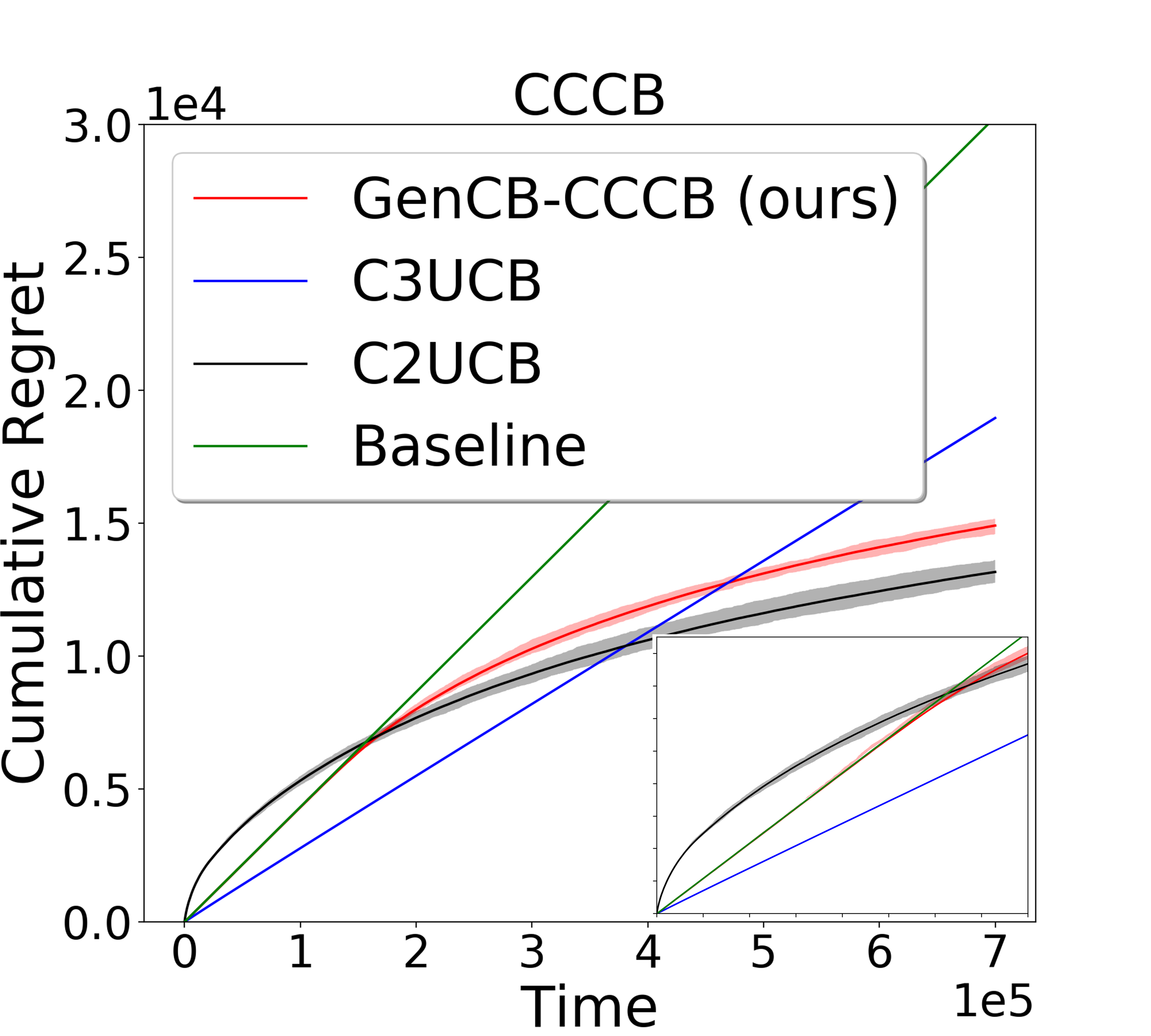}  
	}   
	\subfigure[MV-CBP ($K=24, \alpha=0.05, \rho=10$)] { 
	\label{fig:mv}     
	\includegraphics[width=0.46\columnwidth]{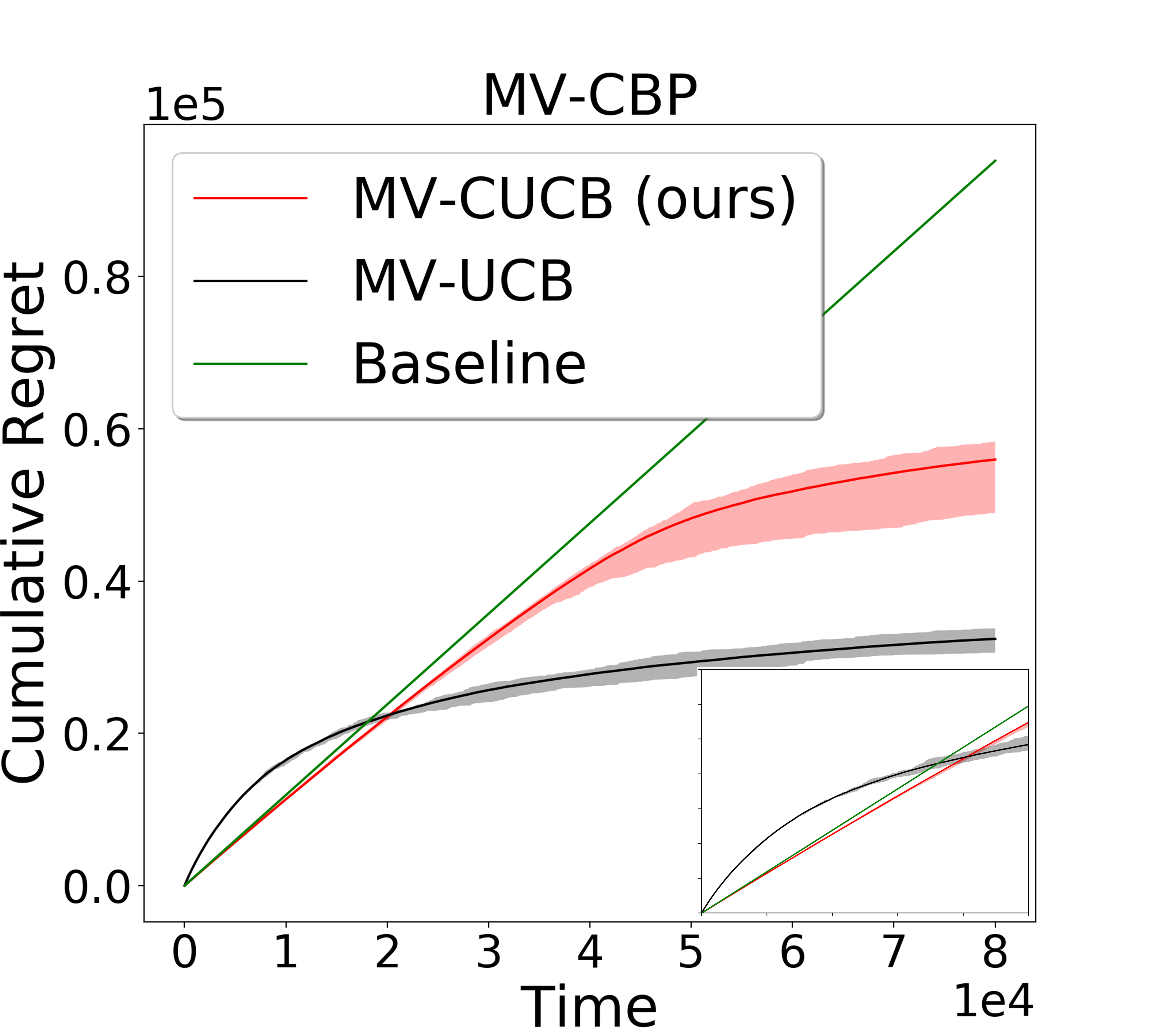}    
	}   
	\caption{Experiments for the studied problems, i.e., CMAB, CLB, CCCB and MV-CBP.}     
	\label{fig:general_label}     
\end{figure}

\section{Conclusion and Future Works}
In this paper, we propose a general solution to a family of conservative bandit problems (CBPs) with sample-path reward constraints, and present its applications to three encompassed problems, i.e., conservative multi-armed bandits (CMAB), conservative linear bandits (CLB) and conservative contextual combinatorial bandits (CCCB). We show that our algorithms outperform existing ones both theoretically (incurs $T$-independent conservative regrets rather than $T$-dependent) and empirically. 
Moreover, we study a novel extension of CBPs to the mean-variance setting (MV-CBP) and develop an algorithm with $O(1/T)$ normalized conservative regret ($T$-independent in the cumulative form). We also validate this result through empirical evaluation. 

There are several directions worth further investigation. One is to consider more general conservative mean-variance bandits other than the $K$-armed setting, e.g., a contextual extension. Another direction is to consider other practical conservative constraints which capture the safe exploration requirement in real-world applications. 

\clearpage

\section*{Acknowledgements}
This work is supported in part by the National Natural Science Foundation of China Grant 61672316, the Zhongguancun Haihua Institute for Frontier Information Technology and the Turing AI Institute of Nanjing.

\section*{Ethical Impact}
In this paper, we study a family of conservative bandit problems and present algorithms with theoretical guarantees and experimental results. While our work mainly focuses on the theoretical analysis, it may have potential social impacts on the applications including finance and clinical trials. For example, our algorithms may help risk-adverse investors choose financial products, with the objective of obtaining high cumulative returns while guaranteeing a certain baseline during exploration. We believe that this work does not involve any ethical issue.

\bibliographystyle{aaai21}
\bibliography{aaai21_conservative_bandit_ref}
\clearpage

\input{aaai21_conservative_bandit_supp.tex}

\end{document}

%% file: aaai21_conservative_bandit_supp.tex
\appendix
\section*{Supplementary Material}
\setcounter{equation}{0}

\section{More Experimental Results} \label{apx:more_experiments}
In this section, we present more experimental results for CMAB, CLB, CCCB and MV-CPB, which are shown in Figure~\ref{fig:apx_cmab}, \ref{fig:apx_clb}, \ref{fig:apx_cccb} and \ref{fig:apx_mv}, respectively.
The parameter settings of our experiments are described in Section~\ref{sec:experiments} of the main paper.

\begin{figure}[h!]
	\centering    
	\subfigure[$K=24, \alpha=0.05$] {    
		\includegraphics[width=0.45\columnwidth]{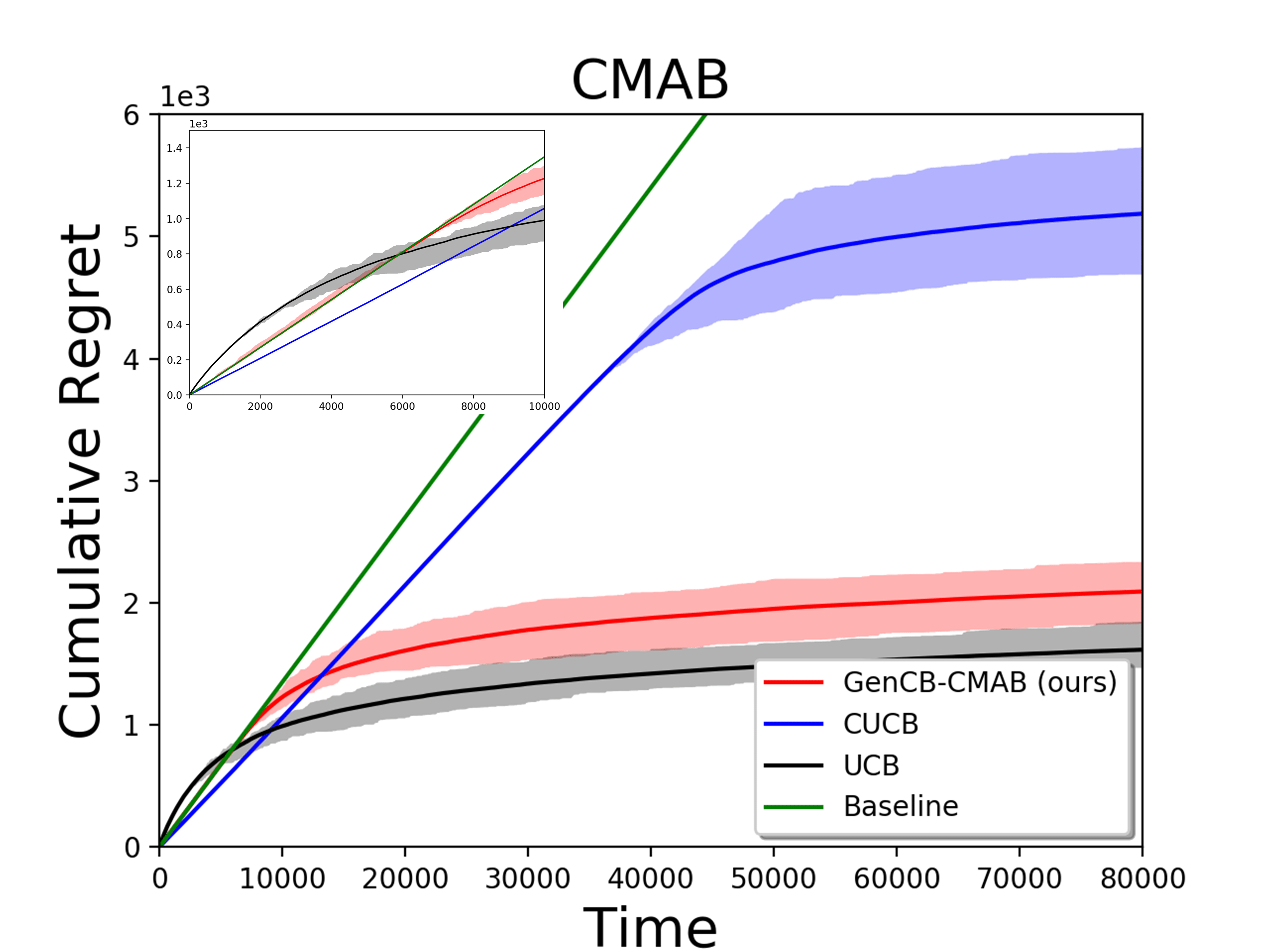}  
	}     
	\subfigure[$K=144, \alpha=0.05$] {     
		\includegraphics[width=0.45\columnwidth]{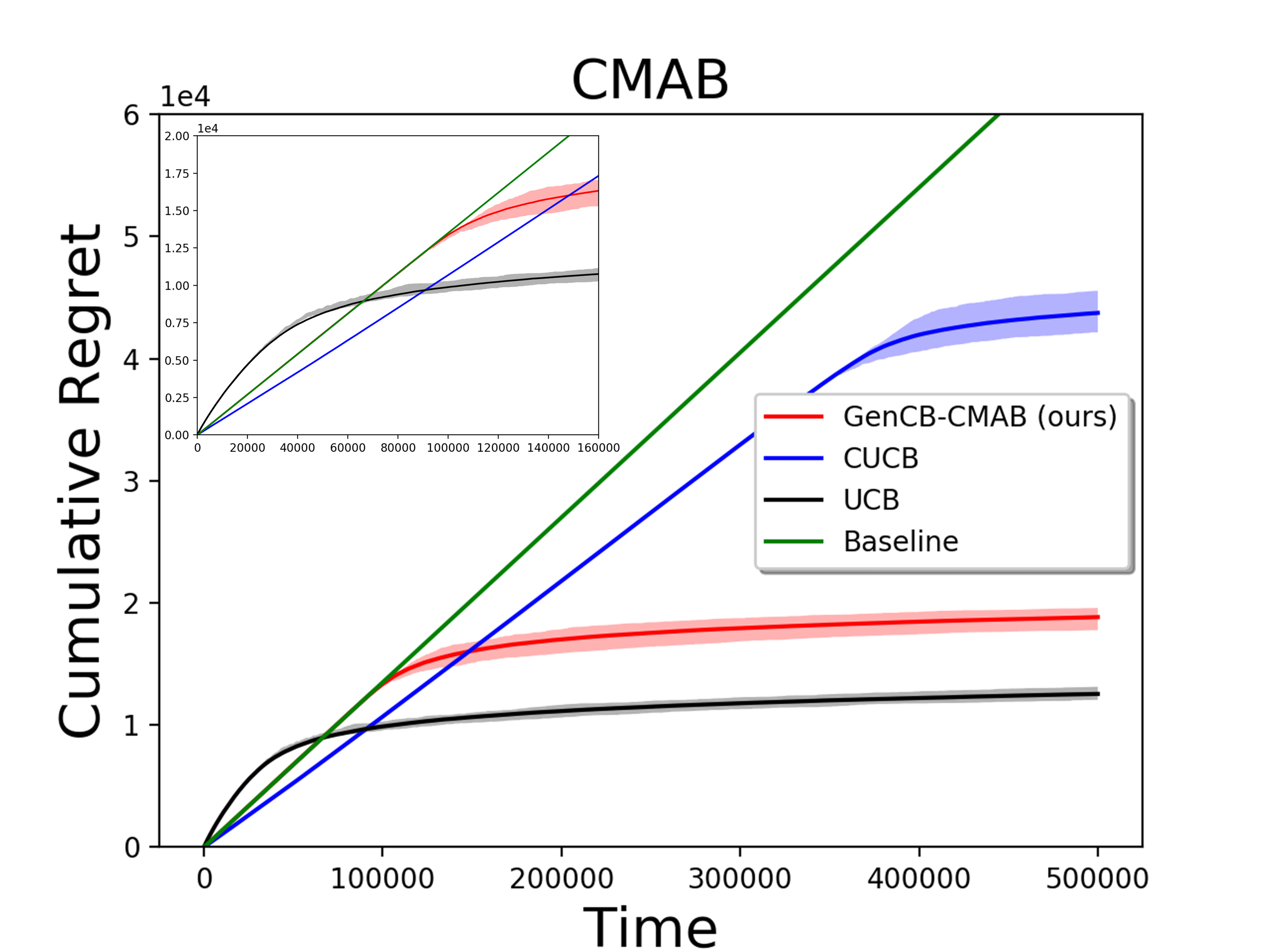}
	}   
	\subfigure[$K=72, \alpha=0.1$] {    
		\includegraphics[width=0.45\columnwidth]{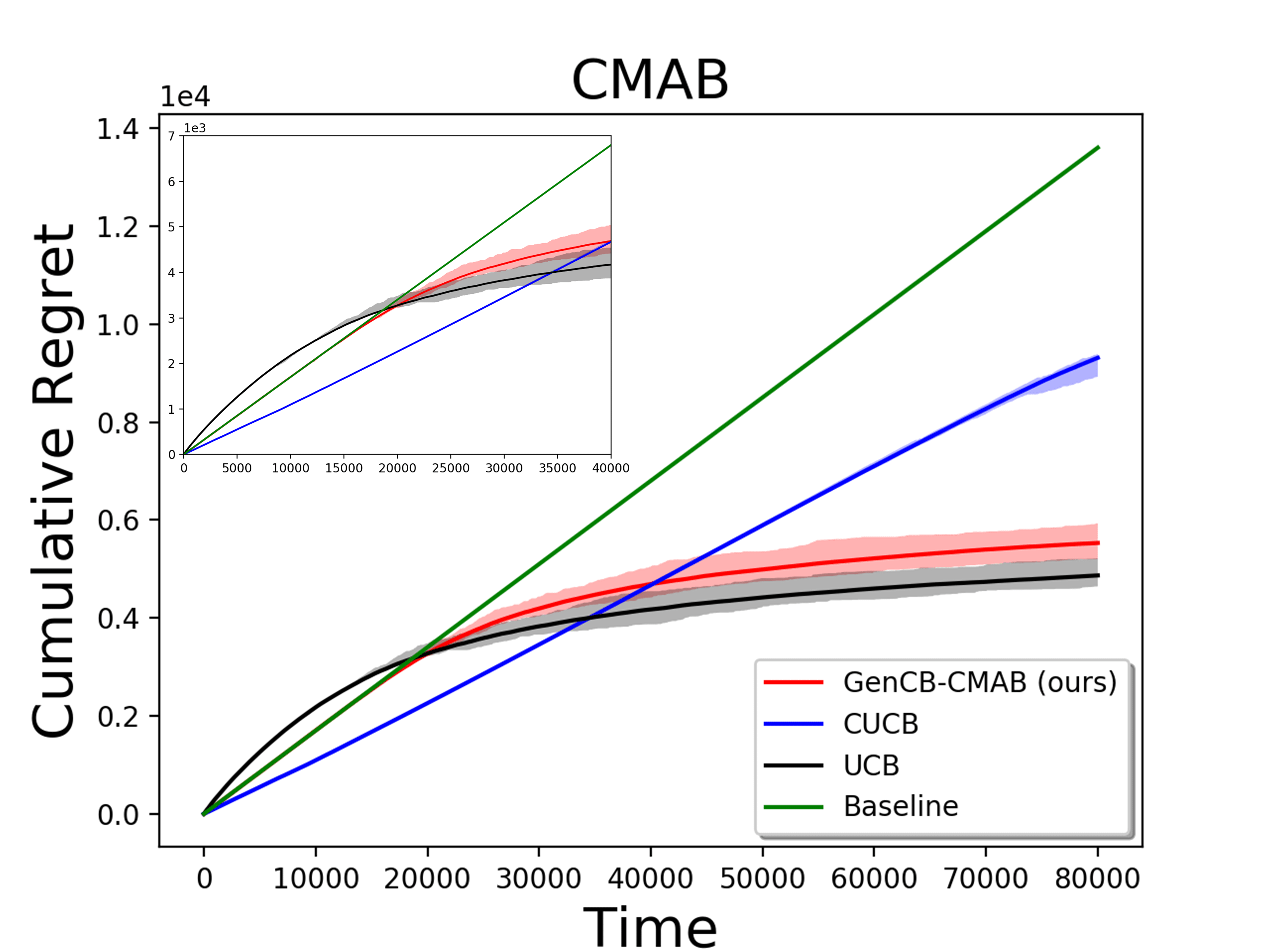}
	}   
	\subfigure[$K=72, \alpha=0.15$] {    
		\includegraphics[width=0.45\columnwidth]{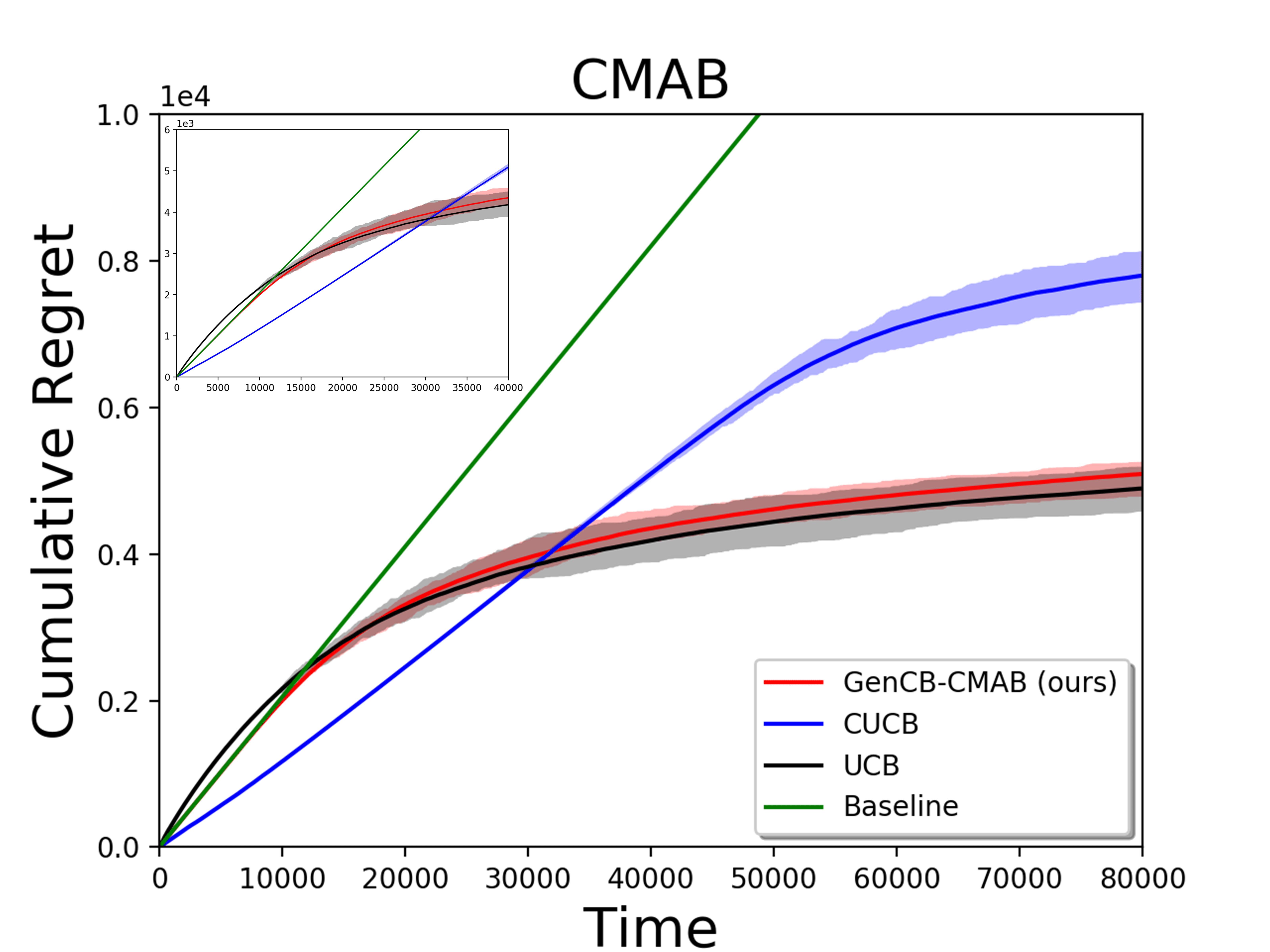}
	}   
	\caption{Experiments for CMAB.}  
	\label{fig:apx_cmab}   
\end{figure}

\begin{figure}[h!]
	\centering    
	\subfigure[$d=5, \alpha=0.01$] {    
		\includegraphics[width=0.45\columnwidth]{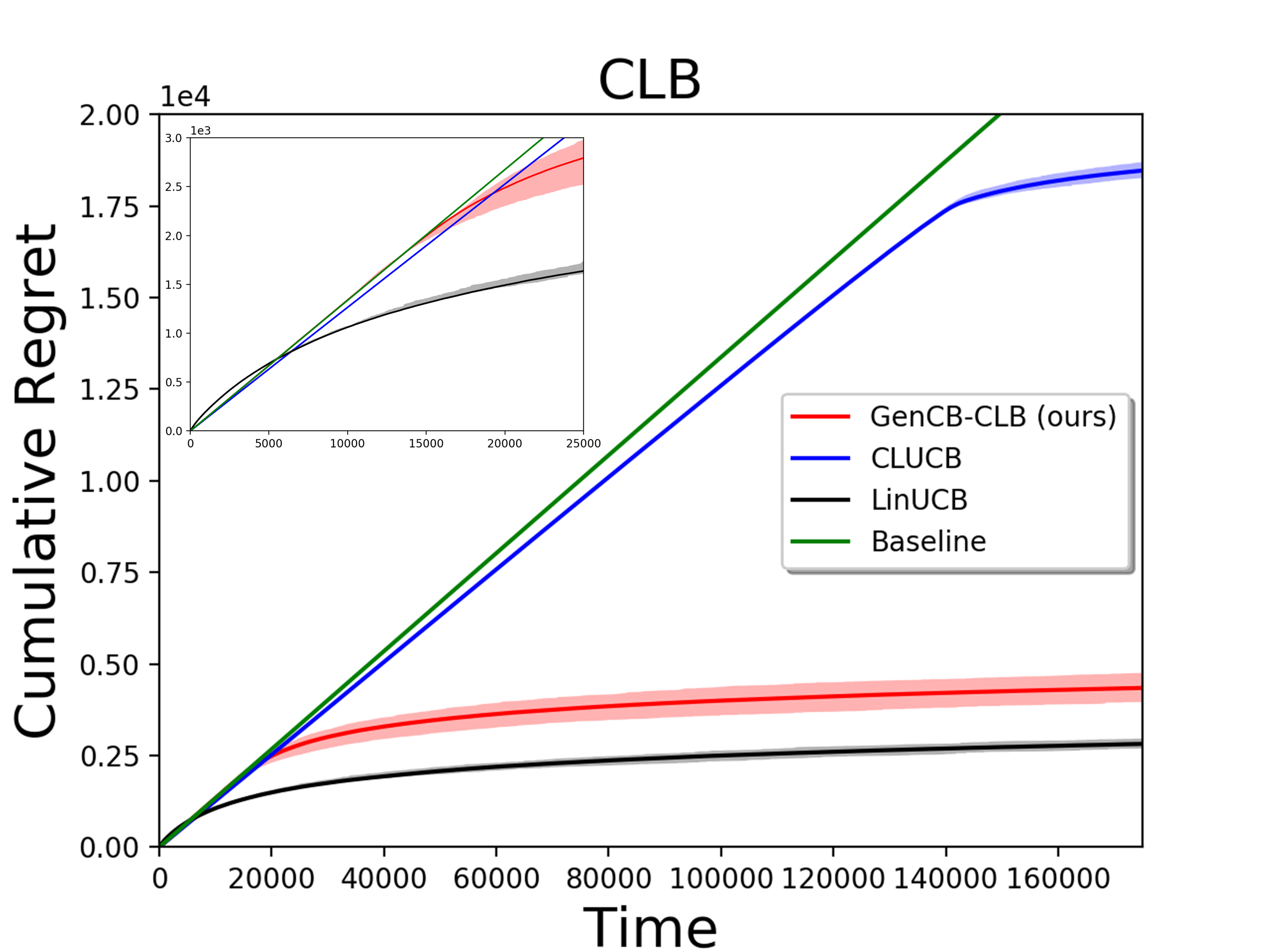}  
	}     
	\subfigure[$d=9, \alpha=0.01$] {     
		\includegraphics[width=0.45\columnwidth]{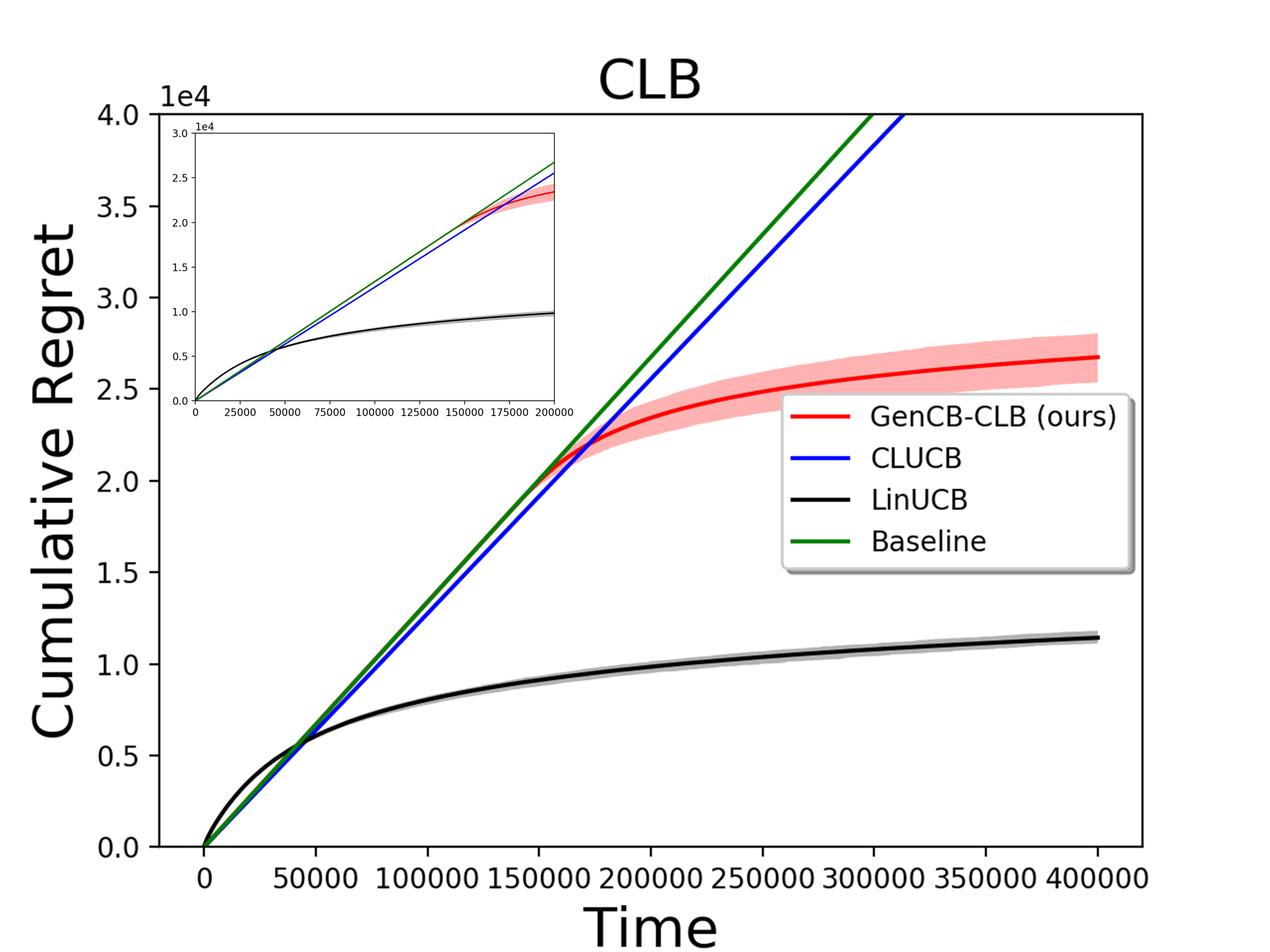}
	}   
	\subfigure[$d=5, \alpha=0.02$] {    
		\includegraphics[width=0.45\columnwidth]{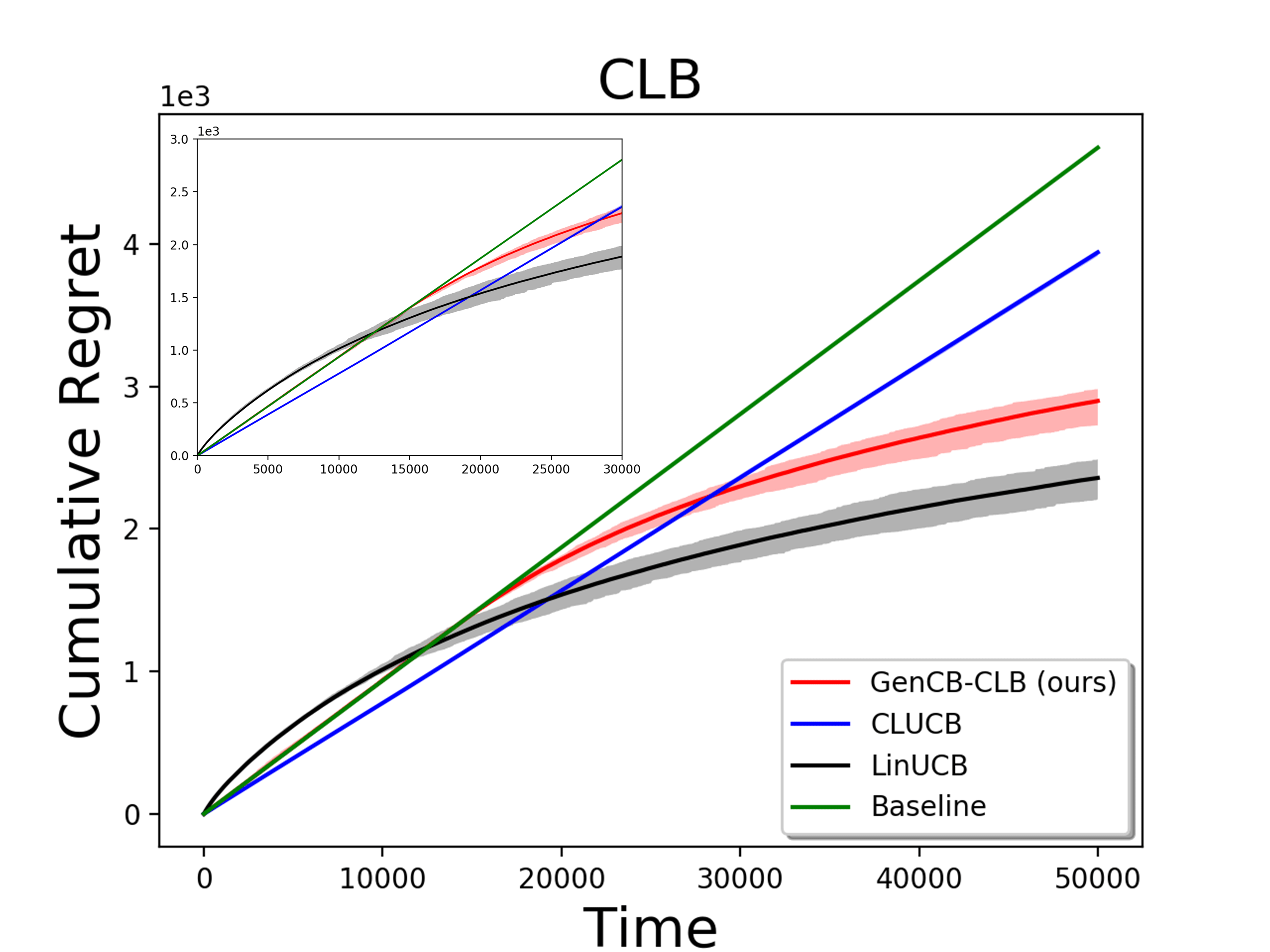}
	}   
	\subfigure[$d=5, \alpha=0.03$] {    
		\includegraphics[width=0.45\columnwidth]{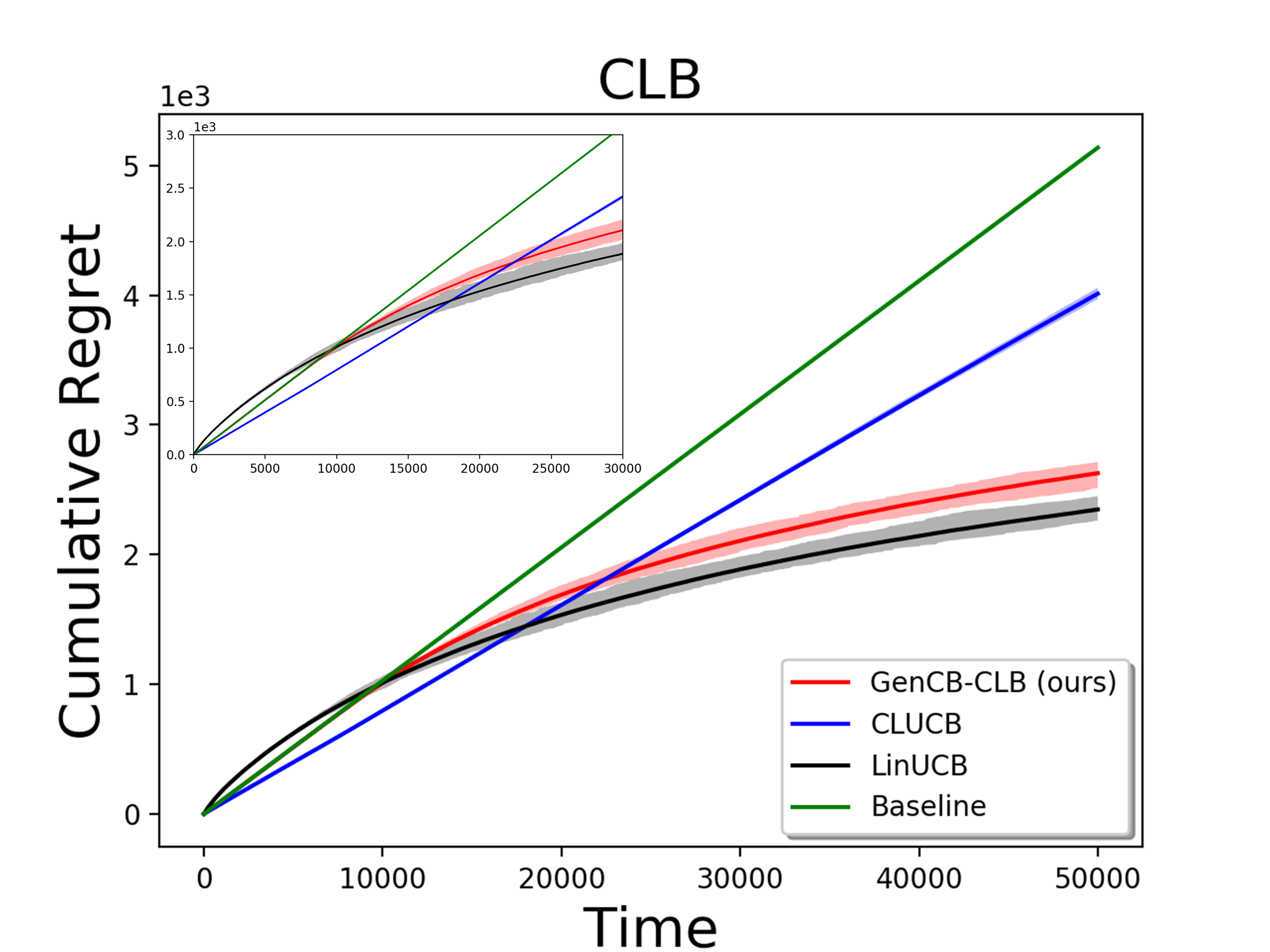}
	}   
	\caption{Experiments for CLB.}  
	\label{fig:apx_clb}   
\end{figure}

\begin{figure}[h!]
	\centering    
	\subfigure[$d=5, \alpha=0.01$] {    
		\includegraphics[width=0.45\columnwidth]{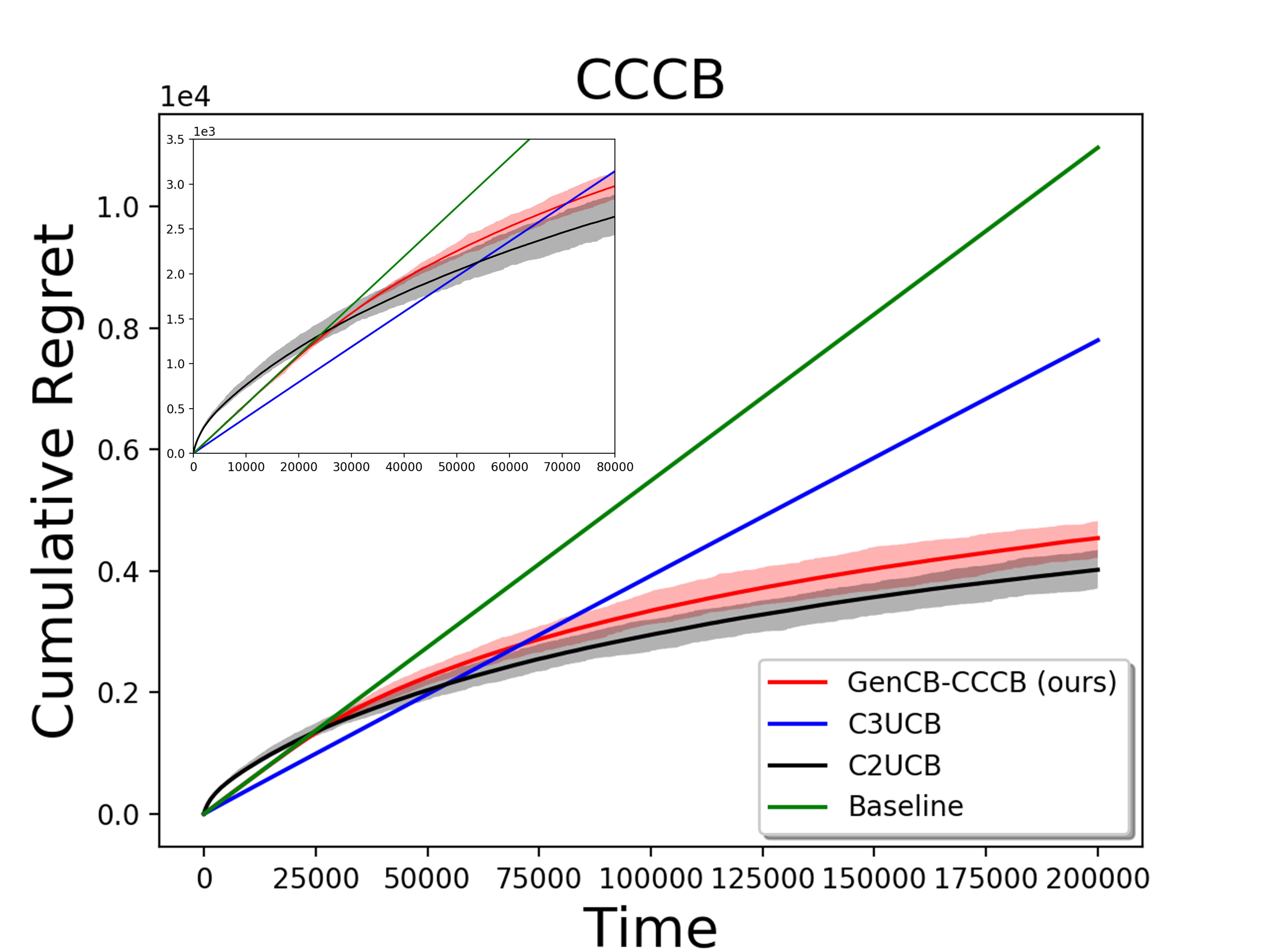}  
	}     
	\subfigure[$d=9, \alpha=0.01$] {     
		\includegraphics[width=0.45\columnwidth]{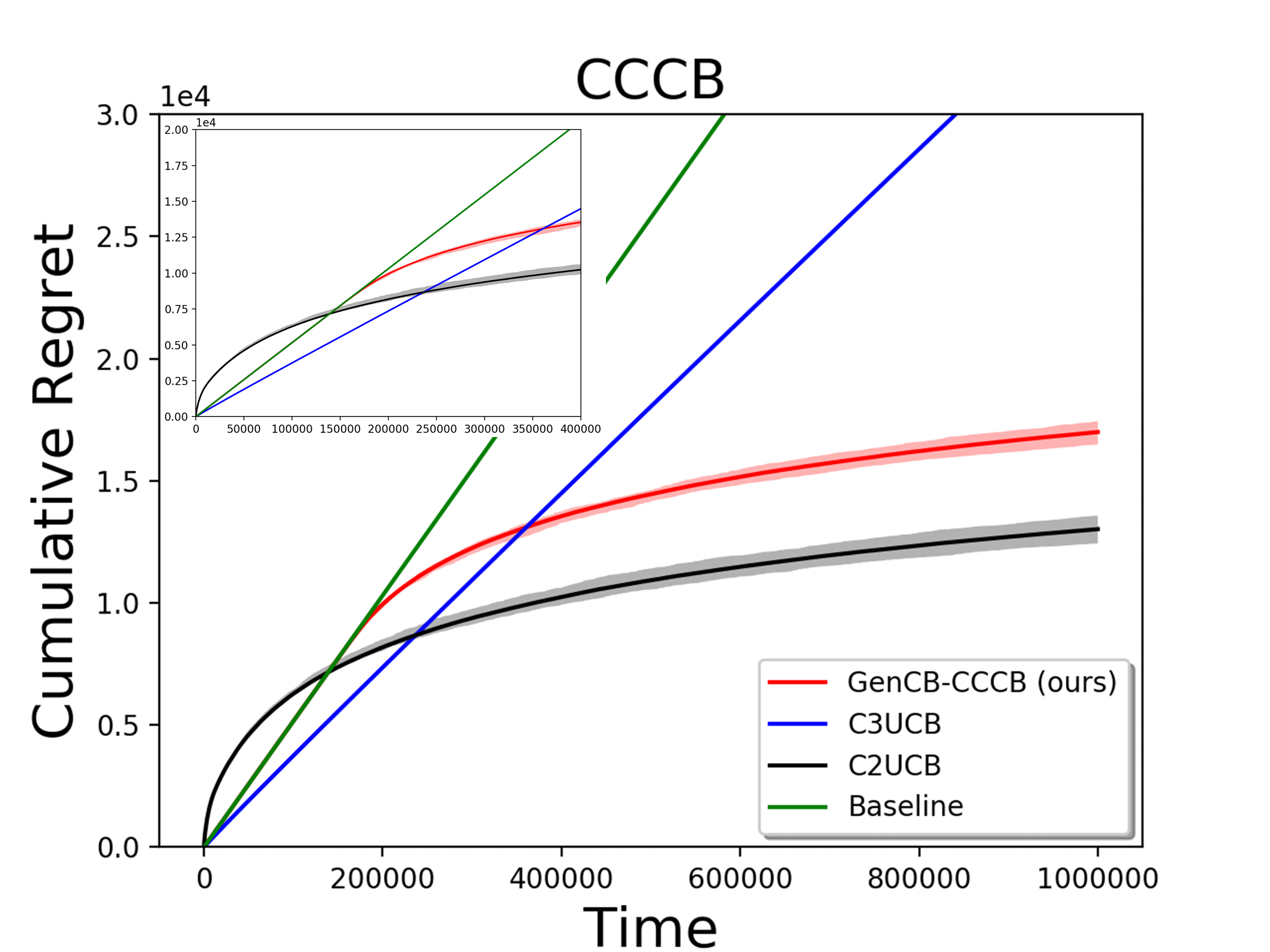}
	}
	\subfigure[$d=5, \alpha=0.02$] {     
		\includegraphics[width=0.45\columnwidth]{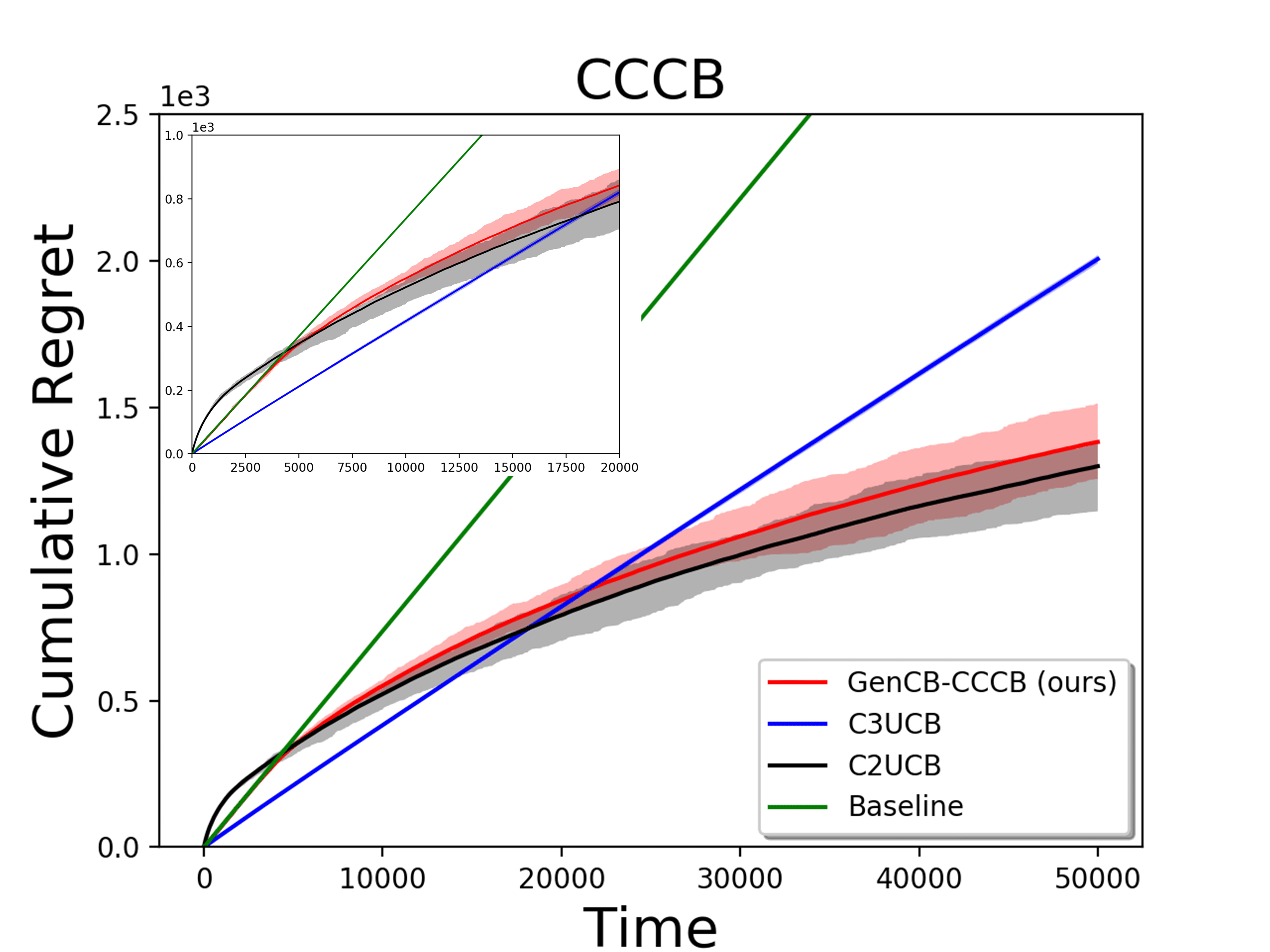}
	}   
	\subfigure[$d=5, \alpha=0.03$] {    
		\includegraphics[width=0.45\columnwidth]{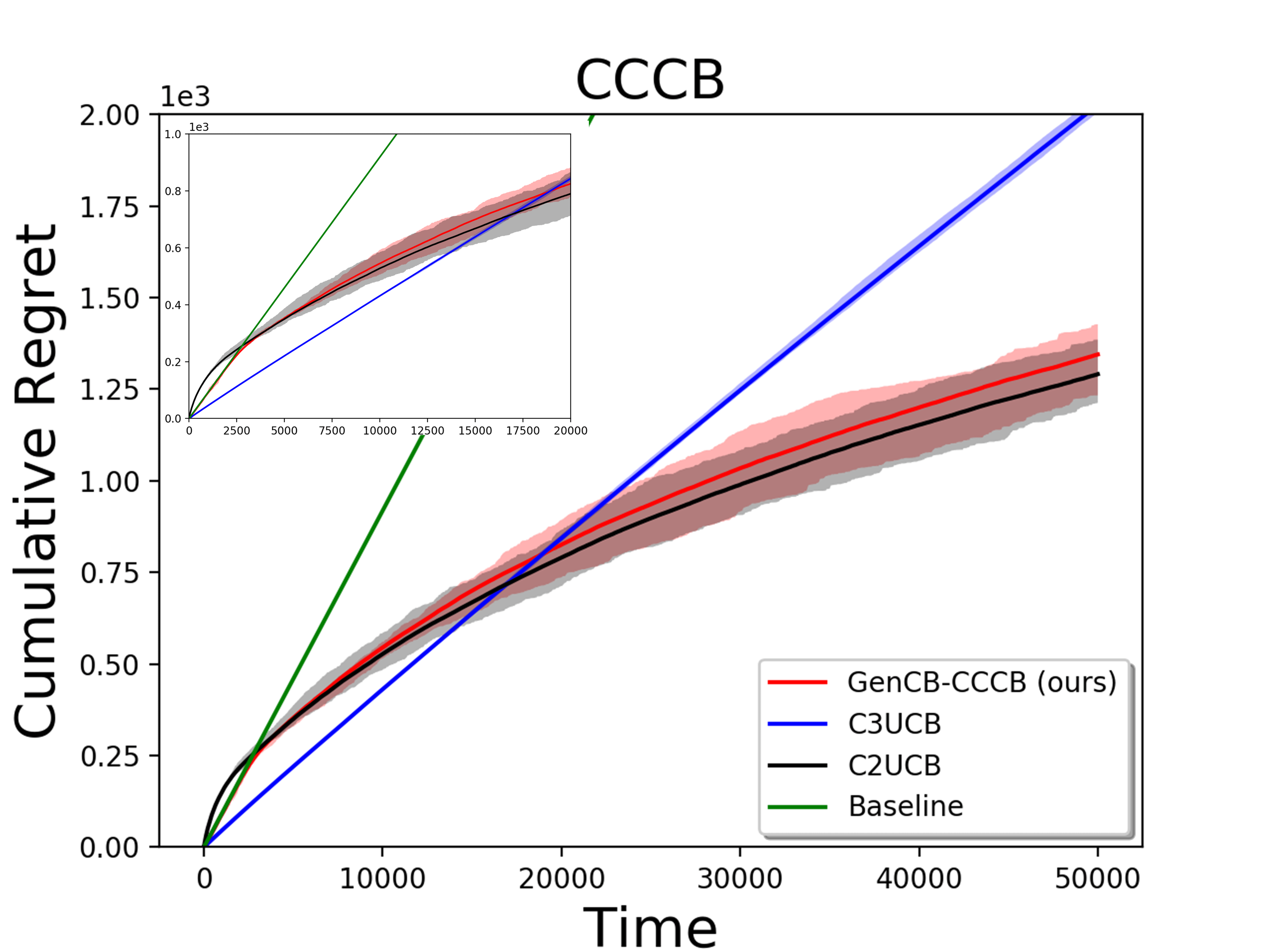}
	}   
	\caption{Experiments for CCCB.}  
	\label{fig:apx_cccb}   
\end{figure}

\begin{figure*}[h!]
	\centering    
	\subfigure[$K=24, \alpha=0.05, \rho=30$] {     
		\includegraphics[width=0.5\columnwidth]{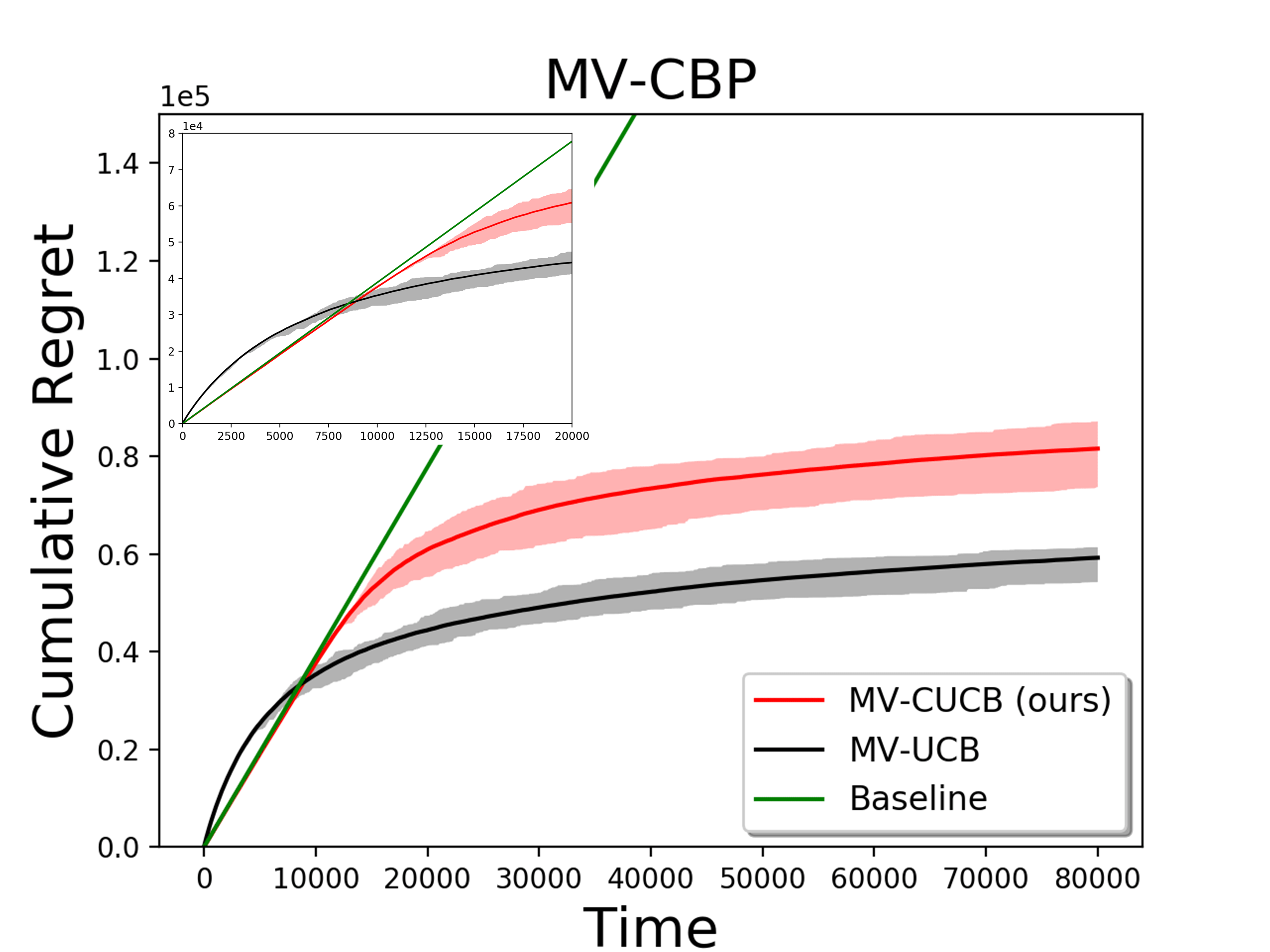}
	}    
	\subfigure[$K=24, \alpha=0.05, \rho=60$] {     
		\includegraphics[width=0.5\columnwidth]{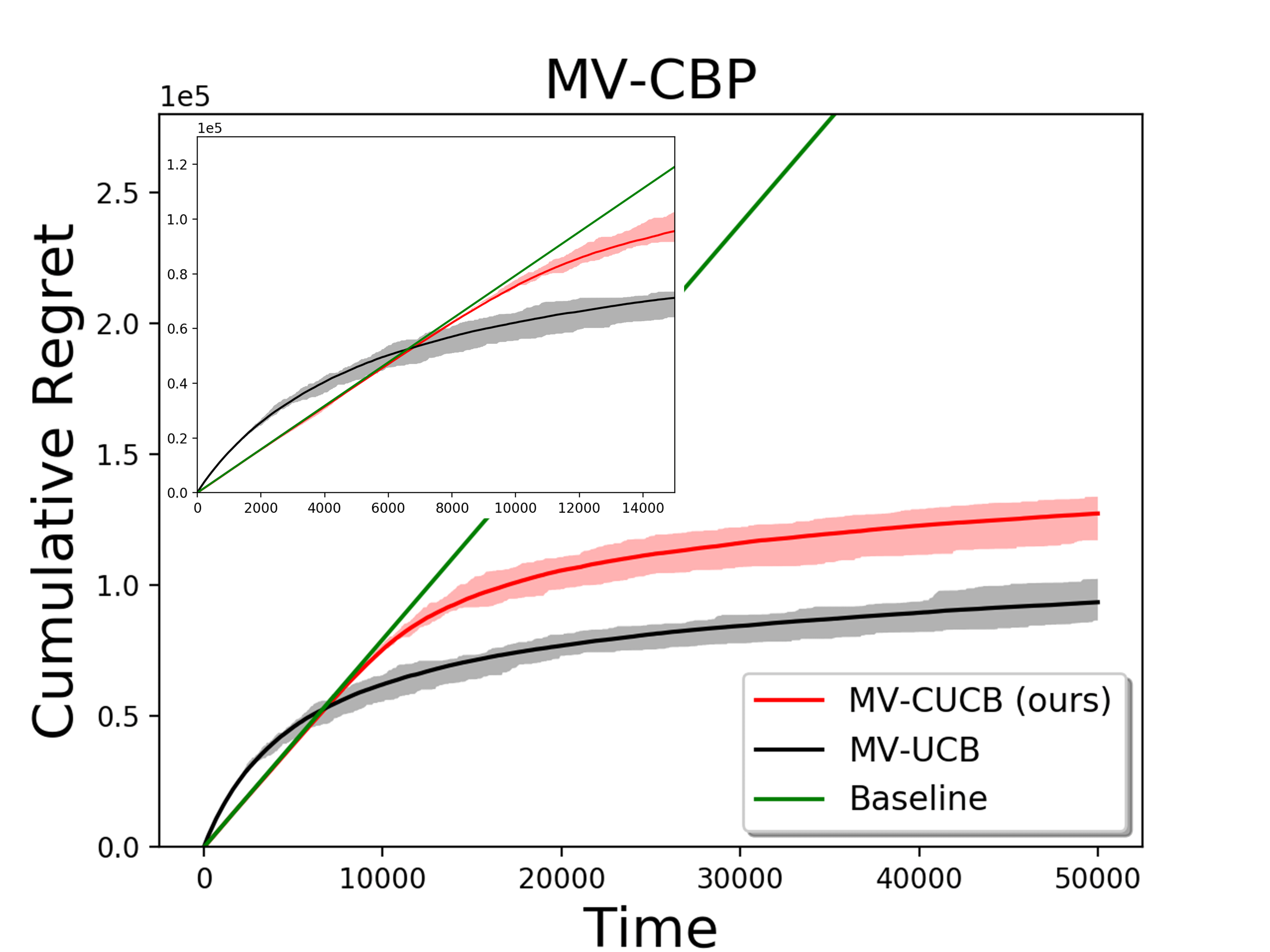}
	}   
	\subfigure[$K=72, \alpha=0.05, \rho=60$] {    
		\includegraphics[width=0.5\columnwidth]{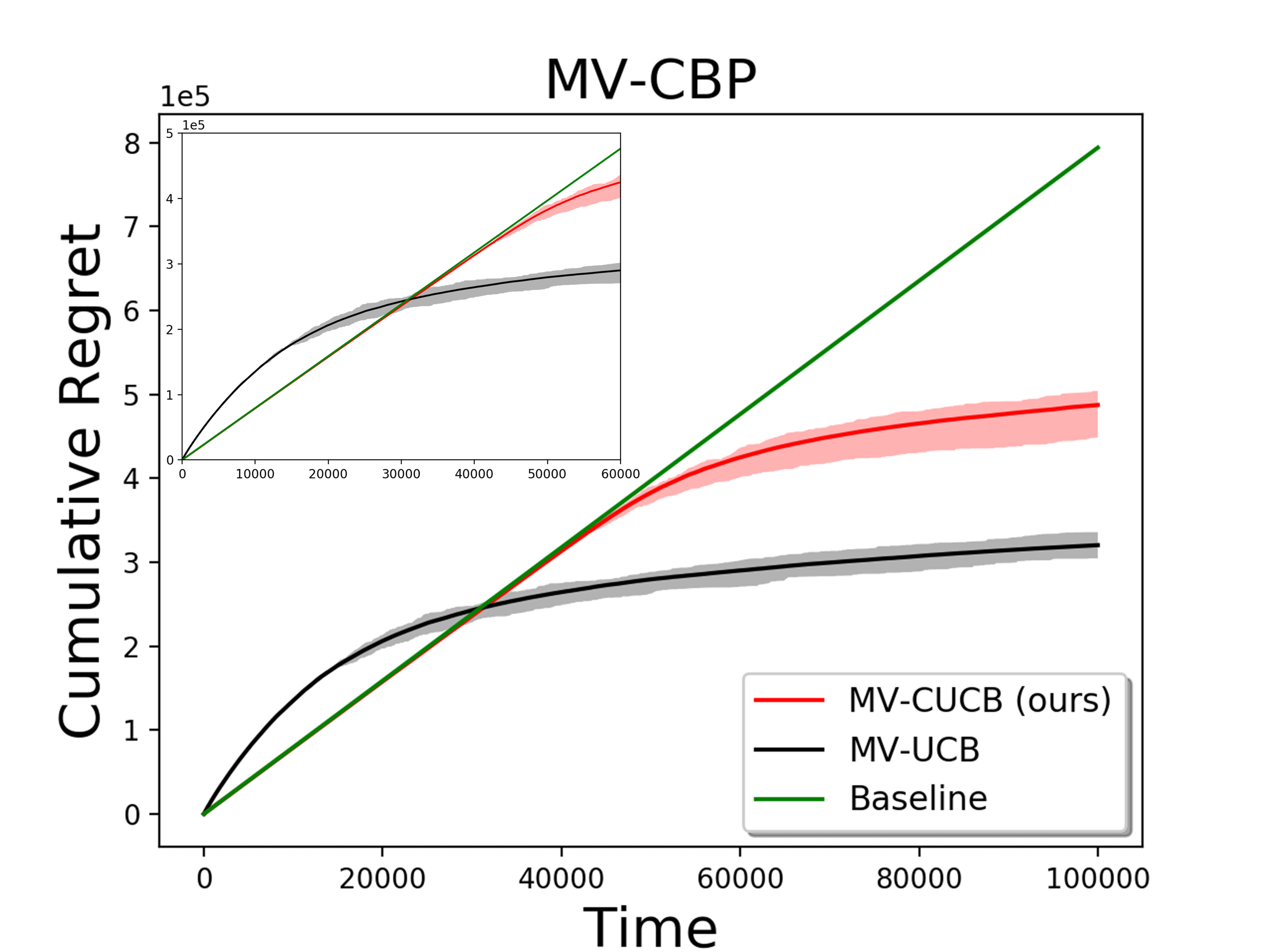}
	} 
	\subfigure[$K=144, \alpha=0.05, \rho=60$] {    
		\includegraphics[width=0.5\columnwidth]{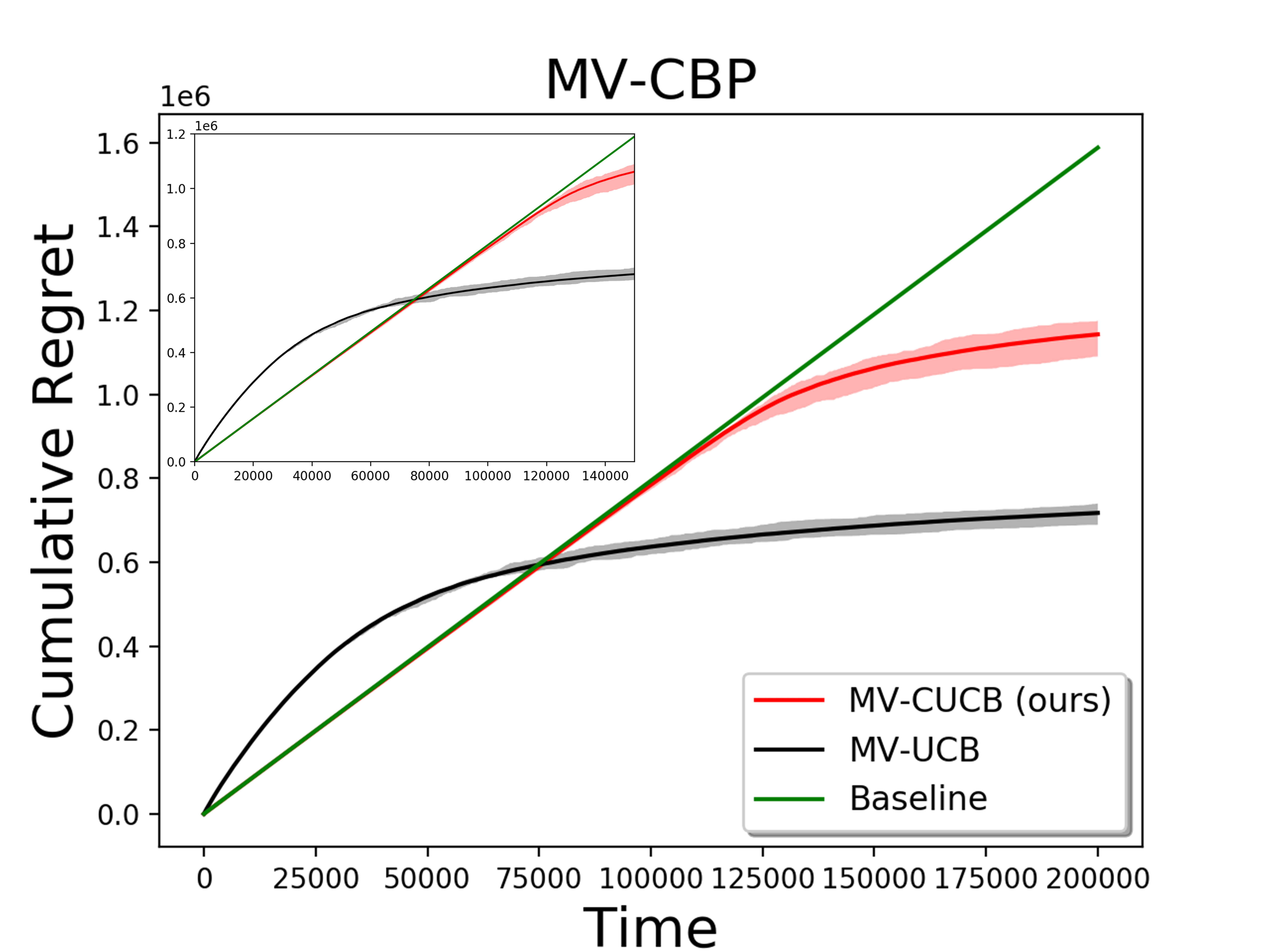}
	}   
	\subfigure[$K=72, \alpha=0.1, \rho=60$] {    
		\includegraphics[width=0.5\columnwidth]{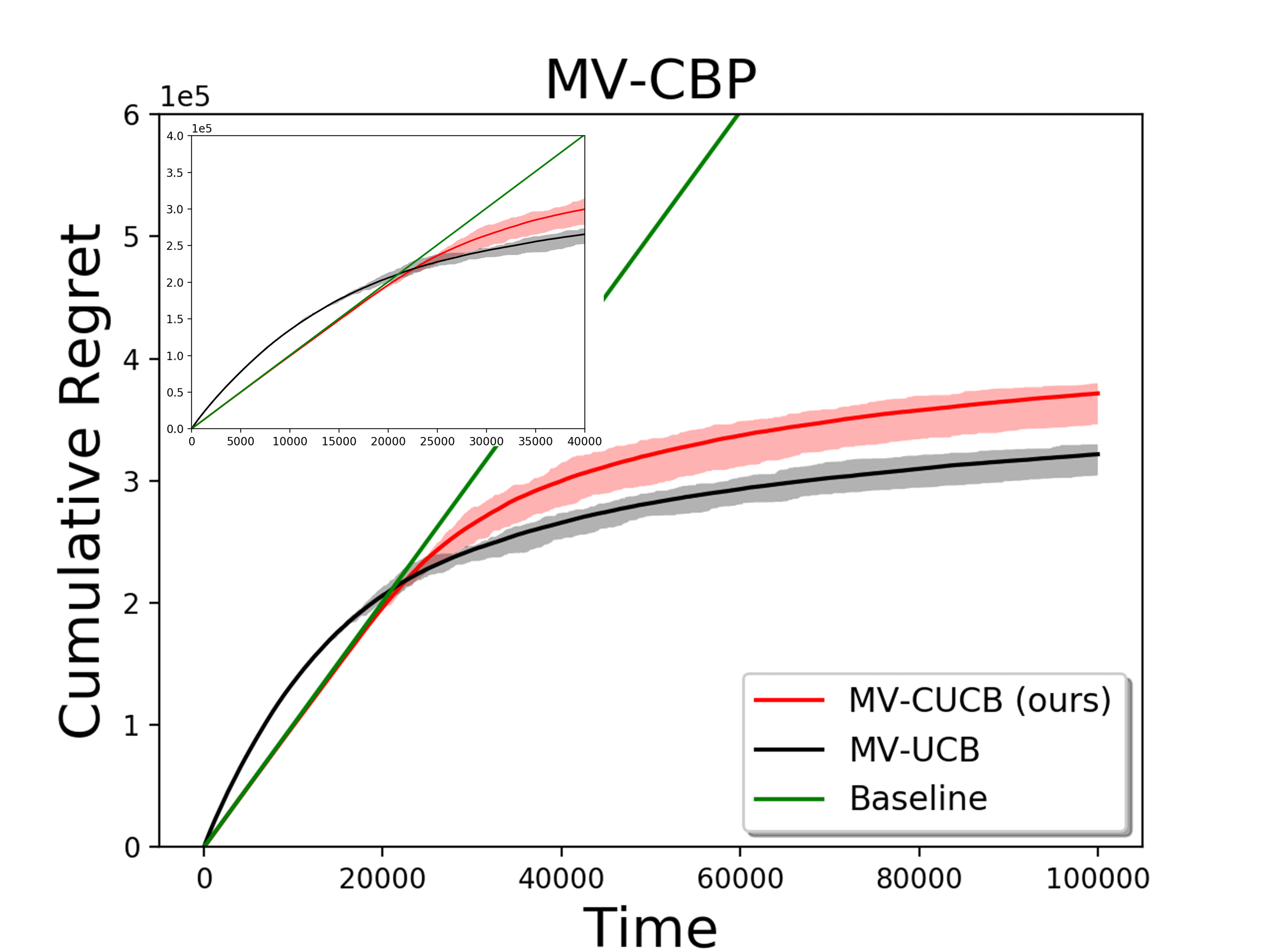}
	} 
	\subfigure[$K=72, \alpha=0.15, \rho=60$] {    
		\includegraphics[width=0.5\columnwidth]{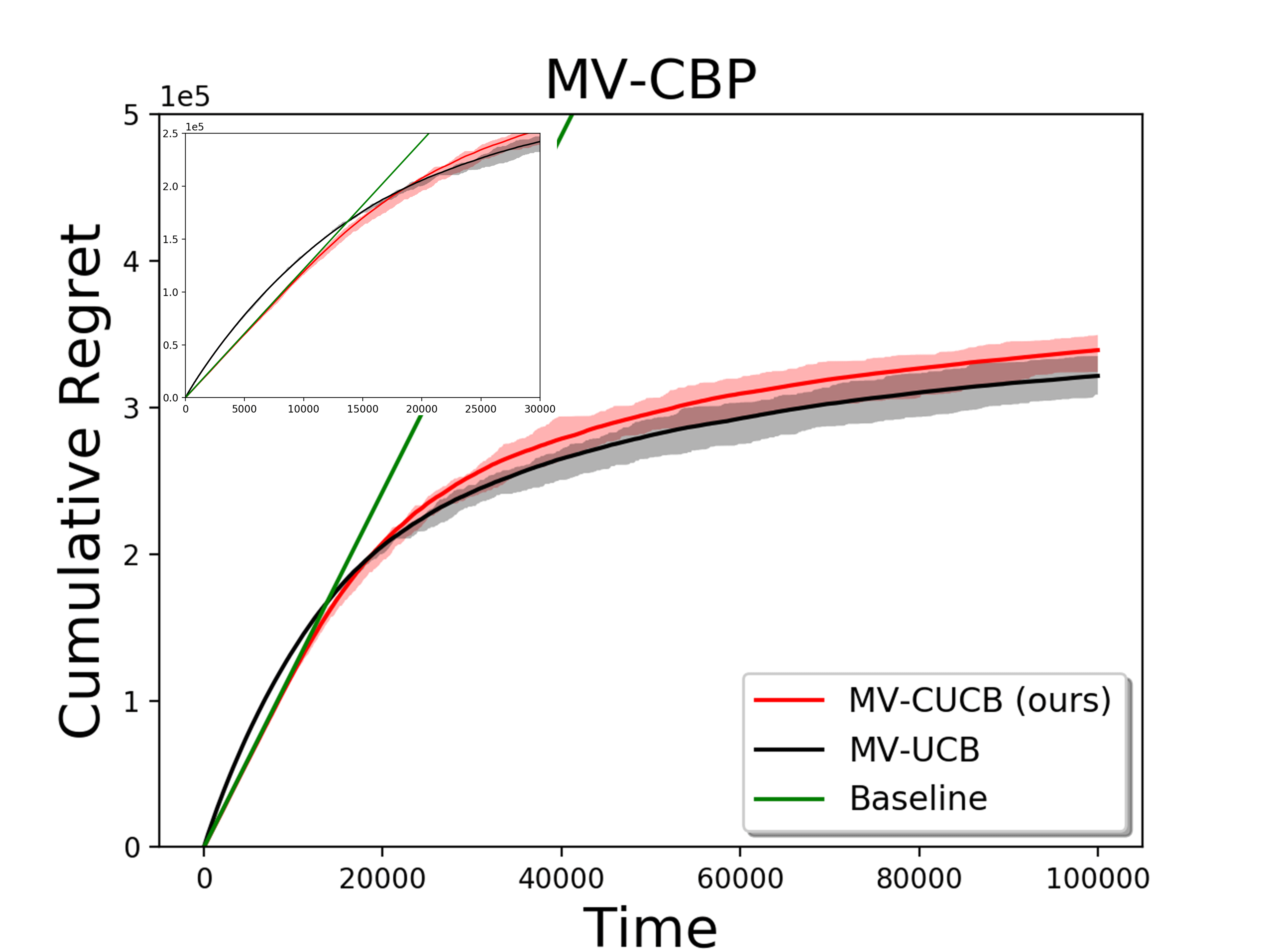}
	} 
	\caption{Experiments for MV-CBP.}  
	\label{fig:apx_mv}   
\end{figure*}

\section{Technical Tools} 
We present some technical tools (Lemmas~\ref{lemma:g_m_ub},\ref{lemma:g_m_ub_mv} and Facts~\ref{fact:clucb_lemma9},\ref{fact:clucb_lemma10}) below.
\begin{lemma} \label{lemma:g_m_ub}
	For $m \geq 2$, $c_1=2$, $c_2=2\mathbb{E}[N_0(\tau-1)] \geq 4$, $c_3=\Delta_0+ \alpha \mu_0 \in (0,1)$, $c_4=8H$ where $H\geq 2$, define function $g_1(m)=-c_3m+c_1\sqrt{m}\ln(c_2 m)+c_4\ln(m)$.
	Then, $g_1(m)$ can be upper bounded by
	$$
	g_1(m) \leq \frac{132 c_1 c_4 }{c_3}  \left [ \ln \left( \frac{10 \sqrt{c_2c_4} }{c_3} \right) \right]^2 .
	$$
\end{lemma}

\begin{proof}
	Taking the derivative of $g_1(m)$, we obtain
	$$g_1'(m)=-c_3 + \frac{c_1 \ln(c_2 m) +2c_1}{2 \sqrt{m}} + \frac{c_4}{m} $$
	Let $\tilde{m}_1=\frac{c_4}{c_3}, \tilde{m}_2=\frac{100c_4 [\ln(c_2 c_4 / c_3^2)]^2}{c_3^2}$. Then, we have
	$$ g'_1(\tilde{m}_1)= \frac{c_1 \ln(c_2 \tilde{m}_1) +2c_1}{2 \sqrt{\tilde{m}_1}} >0 $$
	and
	\begin{align*}
		g'_1(\tilde{m}_2)= & -c_3 + \frac{c_1 \ln( \frac{100c_2c_4 [\ln(c_2 c_4 / c_3^2)]^2}{c_3^2}) +2c_1}{ 10 \sqrt{c_4} \ln(c_2 c_4 / c_3^2)} \cdot \frac{c_3}{2} \\& + \frac{c_3}{100 [\ln(c_2 c_4 / c_3^2)]^2} \cdot c_3 .
	\end{align*}
	Since 
	\begin{align*}
	& c_1 \ln( \frac{100c_2c_4 [\ln(c_2 c_4 / c_3^2)]^2}{c_3^2}) +2c_1
	\\
	 = & c_1 \ln( \frac{100c_2 c_4 }{c_3^2}) +  2 c_1 \ln(  \ln(\frac{c_2 c_4}{c_3^2})) +2c_1
	\\
	\leq & c_1 \ln( \frac{100c_2c_4 }{c_3^2}) +  2 c_1 \ln( \frac{c_2 c_4}{c_3^2}  )  +2c_1
	\\
	\leq & 3 c_1 \ln( \frac{100c_2c_4 }{c_3^2}) +  2c_1
	\\
	= & 6 \ln( \frac{c_2c_4 }{c_3^2}) + 6 \ln(100) + 4 
	\\
	< & 10 \sqrt{c_4} \ln(\frac{c_2c_4 }{c_3^2}) ,
	\end{align*}
	we have
	$$
	\frac{c_1 \ln( \frac{100c_2c_4 [\ln(c_2 c_4 / c_3^2)]^2}{c_3^2}) +2c_1}{ 10 \sqrt{c_4} \ln(c_2 c_4 / c_3^2)} < 1.
	$$
	In addition, it is clear that $\frac{c_3}{100 [\ln(c_2 c_4 / c_3^2)]^2} <\frac{1}{2} $. Thus, we have  
	\begin{align*}
		g'_1(\tilde{m}_2)  = & -c_3 + \frac{c_1 \ln( \frac{100c_2c_4 [\ln(c_2 c_4 / c_3^2)]^2}{c_3^2}) +2c_1}{ 10 \sqrt{c_4} \ln(c_2 c_4 / c_3^2)} \cdot \frac{c_3}{2} \\& + \frac{c_3}{100 [\ln(c_2 c_4 / c_3^2)]^2} \cdot c_3 <0 .
	\end{align*}
	Thus,
	\begin{align*}
	g_1(m) \leq & -c_3\tilde{m}_1+c_1\sqrt{\tilde{m}_2}\ln(c_2 \tilde{m}_2)+c_4\ln(\tilde{m}_2)
	\\
	\leq & ( c_1\sqrt{\tilde{m}_2}+c_4 )\ln(c_2 \tilde{m}_2)
	\\
	\leq & ( \frac{10 c_1 \sqrt{c_4}}{c_3} \ln(\frac{c_2c_4 }{c_3^2}) +c_4 ) \cdot \\& \ln( \frac{100c_2c_4 [\ln(c_2 c_4 / c_3^2)]^2}{c_3^2})
	\\
	\leq &  \frac{10 c_1 \sqrt{c_4} +c_4 }{c_3}  \cdot \ln(\frac{c_2c_4 }{c_3^2})  \cdot 3 \ln( \frac{100c_2c_4 }{c_3^2})
	\\
	\leq &  \frac{33 c_1 c_4 }{c_3}  \left [ \ln \left( \frac{100c_2c_4 }{c_3^2} \right) \right]^2 
	\\
	= &  \frac{132 c_1 c_4 }{c_3}  \left [ \ln \left( \frac{10 \sqrt{c_2c_4} }{c_3} \right) \right]^2 .
	\end{align*}
\end{proof}

\begin{lemma} \label{lemma:g_m_ub_mv}
	For $m\geq 2$, $c_1=2(5+\rho)\sqrt{ 2K }  $, $c_2=6 K  \mathbb{E}[N_0(\tau -1)]   $, $c_3=\Delta_0 + \alpha {\mv}_0 \in (2,  \rho)$, $c_4=8 \sqrt{2} K+12(5+\rho)^2 (H^\mv_1+4H^\mv_2) > 8 \sqrt{2} K+12(5+\rho)(K-1)+48(K-1)$ where $\rho>\frac{2}{\alpha \mu_0}>2$, $c_4>3c_1$, $c_4>12c_3$, define function $g_2(m)=-c_3m+c_1\sqrt{m}\ln(c_2 m)+c_4\ln(m)$.
	Then, $g_2(m)$ can be upper bounded by
	$$
	g_2(m) \leq \frac{48 c_1 c_4 }{c_3}   \left [ \ln( \frac{3 \sqrt{c_2} c_4}{c_3}) \right]^2  .
	$$
\end{lemma}

\begin{proof}
	Taking the derivative of $g_2(m)$, we obtain
	$$g_2'(m)=-c_3 + \frac{c_1 \ln(c_2 m) +2c_1}{2 \sqrt{m}} + \frac{c_4}{m} $$
	Let $\tilde{m}_1=\frac{c_4}{c_3}, \tilde{m}_2=\frac{9  c_4^2 [\ln( c_2 c_4^2 / c_3^2)]^2}{c_3^2}$. 
	Then, we have
	$$ g'_2(\tilde{m}_1)= \frac{c_1 \ln(c_2 \tilde{m}_1) +2c_1}{2 \sqrt{\tilde{m}_1}} >0 $$
	and
	\begin{align*}
		g'_2(\tilde{m}_2)= & -c_3 + \frac{c_1 \ln( \frac{9  c_2 c_4^2 [\ln( c_2 c_4^2 / c_3^2 )]^2}{ c_3^2 }) +2c_1}{ 3 c_4 \ln( c_2 c_4^2 / c_3^2 )} \cdot \frac{ c_3 }{2} \\& + \frac{c_3}{9c_4 [\ln( c_2 c_4^2 / c_3^2 )]^2} \cdot c_3 .
	\end{align*}
	Since 
	\begin{align*}
	&c_1 \ln( \frac{9  c_2 c_4^2 [\ln( c_2 c_4^2 / c_3^2 )]^2}{ c_3^2 }) +2c_1 \\
	=  & c_1 \ln( \frac{c_2 c_4^2 }{c_3^2}) +  c_1 \ln(9) +  2 c_1 \ln(  \ln( \frac{c_2 c_4^2 }{c_3^2}) ) +2c_1
	\\
	\leq & c_1 \ln( \frac{c_2 c_4^2 }{c_3^2}) +  c_1 \ln(9) +  2 c_1 \ln(   \frac{c_2 c_4^2 }{c_3^2} ) +2c_1
	\\
	= & 3 c_1 \ln( \frac{c_2 c_4^2 }{c_3^2}) +  ( 2+\ln(9) )c_1
	\\
	< & 9 c_1 \ln( \frac{c_2 c_4^2 }{c_3^2}) 
	\\
	< & 3 c_4 \ln( \frac{c_2 c_4^2 }{c_3^2}) ,
	\end{align*}
	we have
	$$
	\frac{c_1 \ln( \frac{9  c_2 c_4^2 [\ln( c_2 c_4^2 / c_3^2 )]^2}{ c_3^2 }) +2c_1}{ 3 c_4 \ln( c_2 c_4^2 / c_3^2 )}  < 1.
	$$
	In addition, it is clear that $\frac{c_3}{9c_4 [\ln( c_2 c_4^2 / c_3^2 )]^2} < \frac{1}{2} $. Thus, we have  
	\begin{align*}
	g'_2(\tilde{m}_2)= & -c_3 + \frac{c_1 \ln( \frac{9  c_2 c_4^2 [\ln( c_2 c_4^2 / c_3^2 )]^2}{ c_3^2 }) +2c_1}{ 3 c_4 \ln( c_2 c_4^2 / c_3^2 )} \cdot \frac{ c_3 }{2} \\& +  \frac{c_3}{9c_4 [\ln( c_2 c_4^2 / c_3^2 )]^2} \cdot c_3 <0 .
	\end{align*}
	Thus,
	\begin{align*}
	g_2(m) \leq & -c_3\tilde{m}_1+c_1\sqrt{\tilde{m}_2}\ln(c_2 \tilde{m}_2)+c_4\ln(\tilde{m}_2)
	\\
	\leq & ( c_1\sqrt{\tilde{m}_2}+c_4 )\ln(c_2 \tilde{m}_2)
	\\
	\leq & ( \frac{3 c_1 c_4 }{c_3} \ln(\frac{c_2 c_4^2 }{c_3^2}) +c_4 )\ln( \frac{9 c_2 c_4^2 [\ln( c_2 c_4^2 / c_3^2)]^2}{c_3^2})
	\\
	\leq &   \frac{4 c_1 c_4 }{c_3} \ln(\frac{c_2 c_4^2 }{c_3^2})  \cdot 3 \ln( \frac{9 c_2 c_4^2}{c_3^2})
	\\
	\leq &  \frac{12 c_1 c_4 }{c_3}   \left [ \ln( \frac{9 c_2 c_4^2}{c_3^2}) \right]^2 
	\\
	= &
	\frac{48 c_1 c_4 }{c_3}   \left [ \ln( \frac{3 \sqrt{c_2} c_4}{c_3}) \right]^2 .
	\end{align*}
\end{proof}

\begin{fact}[Lemma 9 in \cite{kazerouni2017conservative}] \label{fact:clucb_lemma9}
	For any $m \geq 2$ and $c_1,c_2,c_3>0$, the following holds
	$$
	-c_3m + c_1 \sqrt{m} \ln(c_2 m) \leq \frac{16 c_1^2}{9 c_3} \left[ \ln \left( \frac{2c_1\sqrt{c_2}e}{c_3} \right) \right]^2 .
	$$
\end{fact}

\begin{fact}[Lemma 10 in \cite{kazerouni2017conservative}] \label{fact:clucb_lemma10}
	Let $c_1$ and $c_2$ be two positive constants such that $\ln(c_1 c_2) \geq 1$. Then, any $z>0$ satisfying $z \leq c_1 \ln(c_2 z)$ also satisfies $z \leq 2c_1 \ln(c_1 c_2)$.
\end{fact}

\section{Algorithm Pseudo-code and Proof for CMAB} \label{apx:cmab}
Algorithm~\ref{alg:GenCUCB} presents the algorithm pseudo-code of \textsf{GenCB-CMAB} for CMAB, and we give the detailed proof of Theorem~\ref{thm:gen_con_ucb} in the following.

\begin{algorithm}[t]
	\caption{\textsf{GenCB-CMAB}} \label{alg:GenCUCB}
	\KwIn{Reugular arms $[K]$, default arm $x_0$ with reward $\mu_0$, parameter $\alpha$.}
		$\forall t \geq 0, \forall 0 \leq i \leq K,  N_i(t) \leftarrow 0$. $\forall t \geq 0, r_S(t) \leftarrow 0$. $m \leftarrow 0$\;
	\For{$t=1, 2, \dots$}
	{
		\If{$r_S(t-1)+N_0(t-1)\mu_0 \geq (1-\alpha) \mu_0 t$} 
		{
			$m \leftarrow m+1$\;
			$x_t \leftarrow \argmax_{i \in [K]} \left(\hat{\mu}_i + \sqrt{ \frac{2 \ln m}{ N_i(t-1) } } \right)$\;
			Play arm $x_t$, observe the random reward $r_{t,x_t}$ and update the empirical mean $\hat{\mu}_{x_t}$\;
			$N_{x_t}(t) \leftarrow N_{x_t}(t-1) + 1 $ and $\forall 0 \leq i \leq K, i \neq x_t, N_{i}(t) \leftarrow N_{i}(t-1) $\;
			$r_S(t) \leftarrow r_S(t-1)+r_{t,x_t}$\;
		}
		\Else
		{
			Play $x_0$ and receive reward $\mu_0$\;
			$N_0(t) \leftarrow N_0(t-1)+1$\;
			$r_S(t) \leftarrow r_S(t-1)$\;
			
		}
	}
\end{algorithm}

\begin{proof}
	First, we prove that \textsf{GenCB-CMAB} satisfies the sample-path reward constraint Eq.~\eqref{eq:perf_constraint} by induction.
	At timestep $t=1$, since the LHS of the \textsf{if} statement (in Line 3 of Algorithm \ref{alg:GenCUCB}) is zero and RHS is positive, \textsf{GenCB-CMAB} will pull the default arm $x_0$ and receive reward $\mu_0 \geq (1-\alpha) \mu_0$, which satisfies the constraint.
	Suppose that the sample-path reward constraint holds at timestep $t-1$. At time step $t$, if \textsf{GenCB-CMAB} plays $x_0$, it is clear that the constraint still holds for $t$. If \textsf{GenCB-CMAB} plays a regular arm $x_t$, which implies $r_S(t-1)+N_0(t-1)\mu_0 \geq (1-\alpha) \mu_0 t$, the received cumulative reward is $r_S(t-1)+N_0(t-1)\mu_0+r_{t,x_t} \geq r_S(t-1)+N_0(t-1)\mu_0 \geq (1-\alpha) \mu_0 t$, and thus the constraint still holds for $t$.
	
	Recall that $m_t$ denotes the number of times we played regular arms up to $t$, and $\cS_{t}$ denotes the set of timesteps when we played regular arms up to $t$. 
	
	Fix a time horizon $T$.
	Let $\tau \leq T$ denote the last timestep when \textsf{GenCB-CMAB} played arm $x_0$. 
	
	Define event 
	$$
	\cE:= \left \{ \sum_{t=1}^{m_{\tau-1}} r_{t,x_t} \geq \sum_{t=1}^{m_{\tau-1}} \mu_{x_t} - \sqrt{ 2 m_{\tau-1} \ln \left( 2m_{\tau-1}^3 \right) } \right \}.
	$$
	According to the Azuma-Hoeffding inequality, we have that for any $n \geq 1$,
	\begin{align*}
	& \Pr \mbr{ \bar{\cE},\ m_{\tau-1}=n }
	\\
	= & \Pr \Bigg[ \sum_{t=1}^{m_{\tau-1}} r_{t,x_t} < \sum_{t=1}^{m_{\tau-1}} \mu_{x_t} - \sqrt{ 2 m_{\tau-1} \ln \left( 2m_{\tau-1}^3 \right) },\\& \qquad \  m_{\tau-1}=n \Bigg]
	\\ 
	\leq & \frac{1}{2 n^3}.
	\end{align*}
	
	
	At the timestep $\tau$, we have the following three equivalent inequalities:
	\begin{align*}
	\sum_{t \in \cS_{\tau-1}} \!\!\! r_t \! + \! N_0(\tau-1)\mu_0  \!\! & <  \!\! (1-\alpha) \mu_0 \tau
	\nonumber
	\\
	\sum_{t \in \cS_{\tau-1}} \!\!\! r_t \! + \! N_0(\tau-1)\mu_0  \!\! & <  \!\! (1-\alpha) \mu_0 (\! N_0(\tau-1) \!+\! m_{\tau-1} \!+\! 1 \!)
	\nonumber
	\\
	\alpha \mu_0 N_0(\tau-1) \!\! & <  \!\!  (1-\alpha)  \mu_0 (m_{\tau-1} +1) - \! \sum_{t \in \cS_{\tau-1}} r_t  
	\end{align*}
	In the standard multi-armed bandit problem, we define the pseudo-regret of the well-known \textsf{UCB}~\cite{UCB_auer2002} algorithm for any time $m$ as
	$$
	\tilde{\cR}_m(\textsf{UCB}) = \mu_* m - \sum_{t=1}^{m} \mu_{x_t}.
	$$
	Thus, we have
	$$
	\sum_{t=1}^{m} \mu_{x_t} = \mu_* m - \tilde{\cR}_m(\textsf{UCB}).
	$$
	Then, we have
	\begin{align}
	\alpha \mu_0  N_0(\tau&-1)  <   (1-\alpha) \mu_0 (m_{\tau-1} +1) - \sum_{t \in \cS_{\tau-1}} r_t   \nonumber\\& -  \tilde{\cR}_{m_{\tau-1}}(\textsf{UCB})  +  \tilde{\cR}_{m_{\tau-1}}(\textsf{UCB})
	\nonumber
	\\
	 = & (1-\alpha) \mu_0 (m_{\tau-1} +1) - \sum_{t \in \cS_{\tau-1}} r_t   - \mu_* m_{\tau-1} \nonumber\\& + \sum_{t \in \cS_{\tau-1}} \mu_{x_t}  +  \tilde{\cR}_{m_{\tau-1}}(\textsf{UCB})
	\nonumber
	\\
	= &  -(\mu_*-(1-\alpha) \mu_0) (m_{\tau-1} +1) + \sum_{t \in \cS_{\tau-1}} \mu_{x_t}  \nonumber\\& - \sum_{t \in \cS_{\tau-1}} r_t    +  \tilde{\cR}_{m_{\tau-1}}(\textsf{UCB}) + \mu_* 
	\nonumber
	\\
	 = & -(\Delta_0 +\alpha \mu_0 ) (m_{\tau-1} +1) + \sum_{t \in \cS_{\tau-1}} \mu_{x_t}  \nonumber\\& - \sum_{t \in \cS_{\tau-1}} r_t    +  \tilde{\cR}_{m_{\tau-1}}(\textsf{UCB}) + \mu_*  \label{eq:introduce_pseudoR}
	\end{align}
	
	In the following analysis, we assume $m_{\tau-1} \geq 1, N_0(\tau-1) \geq 2$, since otherwise the theorem trivially holds. Since $\tau= m_{\tau-1}+N_0(\tau-1)+1 > 2$, we have $\mathbb{E}[m_{\tau-1}| \cE] \leq \frac{\mathbb{E}[m_{\tau-1}]}{\Pr[\cE]} \leq \frac{\mathbb{E}[m_{\tau-1}]}{1-\frac{1}{\tau}} <  2 \mathbb{E}[m_{\tau-1}]$ and $\mathbb{E}[N_0(\tau-1)| \cE] \leq \frac{\mathbb{E}[N_0(\tau-1)]}{\Pr[\cE]} \leq \frac{\mathbb{E}[N_0(\tau-1)]}{1-\frac{1}{\tau}} <  2 \mathbb{E}[N_0(\tau-1)]$.
	Taking expectation of both sides in \eqref{eq:introduce_pseudoR}, from  $\ex[\tilde{\cR}_{m_{\tau-1}}(\textsf{UCB})] \leq \ex[ 8H \ln(m_{\tau-1} +1) +  5K ]$ and Jensen's inequality, we have
	\begin{align}
	\alpha \mu_0 \mathbb{E}[N_0(\tau-1&)]  <  -(\Delta_0 +\alpha \mu_0 ) \mathbb{E}[m_{\tau-1} +1] \nonumber\\& + \mathbb{E}\left[ \sum_{t \in \cS_{\tau-1}} \mu_{x_t} - \sum_{t \in \cS_{\tau-1}} r_t| \cE \right]\Pr[\cE] 
	\nonumber\\& 
	+\! \sum_{n=1}^{\infty} \mathbb{E}\!\!\left[ \!\sum_{t \in \cS_{\tau-1}} \!\!\!\mu_{x_t} \!\!-\!\!\!\!\! \sum_{t \in \cS_{\tau-1}} \!\!\! r_t | \bar{\cE}, m_{\tau-1}\!=\!n \right] \!\!\cdot \nonumber\\& \quad \Pr[\bar{\cE}, m_{\tau-1}=n] \nonumber\\& + \mathbb{E}[\tilde{\cR}_{m_{\tau-1}}(\textsf{UCB})] + \mu_*   
	\nonumber
	\\
	\leq & -(\Delta_0 +\alpha \mu_0 ) \mathbb{E}[m_{\tau-1} +1] \nonumber\\& + \mathbb{E}\left[ 4 \sqrt{ m_{\tau-1} \ln \left( m_{\tau-1} \right) } | \cE \right] \!+\!\!  \sum_{n=1}^{\infty} \! n \!\cdot\! \frac{1}{2n^3}    \nonumber\\& +  \tilde{\cR}(\mathbb{E}[m_{\tau-1}]) + \mu_*
	\nonumber
	\\
	< & -(\Delta_0 +\alpha \mu_0 ) \mathbb{E}[m_{\tau-1} +1] \nonumber\\& + 4 \sqrt{ 2 \mathbb{E}[m_{\tau-1} +1]} \ln ( 2 \mathbb{E}[N_0(\tau-1)] \cdot \nonumber\\& \mathbb{E}[m_{\tau-1} +1] )   +    8H \ln(\mathbb{E}[m_{\tau-1} +1]) \nonumber\\& +  5K + 2 \label{eq:rhs_g_m}
	\end{align}
	Let $m=\mathbb{E}[m_{\tau-1} +1] \geq 2$, $c_1=4 \sqrt{2}$, $c_2=2\mathbb{E}[N_0(\tau-1)]  \geq 2$, $c_3=\Delta_0+ \alpha \mu_0 \in (0,1)$, $c_4=8H$ where $H\geq 2$. The RHS of Eq. \eqref{eq:rhs_g_m} can be written as a constant term plus
	$$
	g_1(m)=-c_3m+c_1\sqrt{m}\ln(c_2 m)+c_4\ln(m).
	$$

	According to Lemma \ref{lemma:g_m_ub}, we have
	$$
	g_1(m) \leq \frac{132 c_1 c_4 }{c_3}  \left [ \ln \left( \frac{10 \sqrt{c_2c_4} }{c_3} \right) \right]^2 .
	$$
	Then, we have
	\begin{align*}
	& \alpha \mu_0 \mathbb{E}[N_0(\tau-1)] \\
	< & \frac{ 4224\sqrt{2} H}{\Delta_0+ \alpha \mu_0} \left[\ln \left(\frac{40 \sqrt{ H \mathbb{E}[N_0(\tau-1)] } }{\Delta_0+ \alpha \mu_0} \right) \right]^2  +  5K + 2
	\\
	< & \frac{ 5981 H}{\Delta_0+ \alpha \mu_0} \left[\ln \left(\frac{40 \sqrt{ H \mathbb{E}[N_0(\tau-1)] } }{\Delta_0+ \alpha \mu_0} \right) \right]^2 .
	\end{align*}
	Thus, we have 
	\begin{align*}
	\mathbb{E}[N_0(\tau-1)]  < & \frac{ 5981 H}{\alpha \mu_0(\Delta_0+ \alpha \mu_0)} \cdot\\& \left[\ln \left(\frac{40 \sqrt{ H \mathbb{E}[N_0(\tau-1)] } }{\Delta_0+ \alpha \mu_0} \right) \right]^2 
	\\
	\sqrt{\mathbb{E}[N_0(\tau-1)]}  < & \sqrt{ \frac{  5981 H}{\alpha \mu_0 (\Delta_0+ \alpha \mu_0)} } \cdot\\& \ln \left(\frac{40 \sqrt{ H \mathbb{E}[N_0(\tau-1)] } }{\Delta_0+ \alpha \mu_0} \right)   
	\end{align*}

	According to Fact~\ref{fact:clucb_lemma10} (set $z=\sqrt{\mathbb{E}[N_0(\tau-1)]}, c_1=\sqrt{ \frac{  5981 H}{\alpha \mu_0 (\Delta_0+ \alpha \mu_0)} }, c_2=\frac{40 \sqrt{ H  } }{\Delta_0+ \alpha \mu_0} $), 
	\begin{align*}
	\sqrt{\mathbb{E}[N_0(\tau-1)]} \leq & 2 \sqrt{ \frac{  5981 H}{\alpha \mu_0 (\Delta_0+ \alpha \mu_0)} } \cdot \\& \ln \left( \sqrt{ \frac{  5981 H}{\alpha \mu_0 (\Delta_0+ \alpha \mu_0)} } \cdot \frac{40 \sqrt{ H  } }{\Delta_0+ \alpha \mu_0}  \right)
	\end{align*}
	Thus, 
	\begin{align*}
	\mathbb{E}[N_0(T)] = & \mathbb{E}[N_0(\tau)] 
	\\
	= & \mathbb{E}[N_0(\tau-1)] + 1 
	\\
	= &  O \! \Bigg( \!\! \frac{  H}{\alpha \mu_0 (\Delta_0+ \alpha \mu_0)} \!\! \left[ \! \ln \! \left(\! \frac{  H}{\alpha \mu_0 (\Delta_0+ \alpha \mu_0)} \!\right) \right]^2 \!\! \Bigg) \! .
	\end{align*}
	Theorem~\ref{thm:gen_con_ucb} follows from $ \ex[\cR_T(\textsf{GenCB-CMAB})] \leq  \ex[\cR_T(\textsf{UCB})] + \mathbb{E}[N_0(T)]\Delta_0 $.
\end{proof}

\section{Algorithm Pseudo-code and Proof for CLB} \label{apx:clb}
Algorithm~\ref{alg:GenCLUCB} presents the algorithm pseudo-code of \textsf{GenCB-CLB} for CLB, and we give the detailed proof of Theorem~\ref{thm:gen_con_linucb} in the following.

\begin{algorithm}[t]
	\caption{\textsf{GenCB-CLB}} \label{alg:GenCLUCB}
	\KwIn{Reugular arms $\cX \subseteq \mathbb{R}^d$, default arm $x_0$ with reward $\mu_0$, parameter $\alpha$, $L$, $S$, $\lambda \geq \max\{1, L^2\}$.}
	$\forall t \geq 0, N_0(t) \leftarrow 0, r_S(t) \leftarrow 0$. $m \leftarrow 0$. $V_0 \leftarrow \lambda I $. $b_0 \leftarrow \boldsymbol{0}^d $\;
	\For{$t=1, 2, \dots$}
	{
		\If{$r_S(t-1)+N_0(t-1)\mu_0 \geq (1-\alpha) \mu_0 t$} 
		{
			$m \leftarrow m+1$\;
			$\cC_{t} \leftarrow \{ \theta\in \mathbb{R}^d : \| \theta-\hat{\theta}_{t-1} \|_{V_{t-1}} \leq \sqrt{ d \ln ( 2 m^2 (1+mL^2/\lambda )  ) } + \sqrt{\lambda} S \}$\;
			$(x_t, \tilde{\theta}_t) \leftarrow \argmax_{(x, \theta) \in \cX \times \cC_t} x^{\top} \theta $\;
			Play arm $x_t$ and observe the random reward $r_{t,x_t}$\;
			$V_t \leftarrow V_{t-1} + x_t x_t^{\top}$, $b_t \leftarrow b_{t-1} + r_{t,x_t} x_t $\;
			$\hat{\theta}_t \leftarrow V_t^{-1} b_t$\;
			$r_S(t) \leftarrow r_S(t-1)+r_{t,x_t}$\;
		}
		\Else
		{
			Play $x_0$ and receive reward $\mu_0$\;
			$N_0(t) \leftarrow N_0(t-1)+1$\;
			$r_S(t) \leftarrow r_S(t-1)$\;
			
		}
	}
\end{algorithm}

\begin{proof}
	Since the proof of satisfaction on the performance constraint Eq.~\eqref{eq:perf_constraint} is the same to Theorem~\ref{thm:gen_con_ucb}, we mainly give the proof of regret bound here.
	
	
	Define event 
	$$
	\cE:= \left \{ \sum_{t=1}^{m_{\tau-1}} r_{x_t} \geq \sum_{t=1}^{m_{\tau-1}} \mu_{x_t} - \sqrt{ 2 m_{\tau-1} \ln \left( 2 m_{\tau-1}^3 \right) } \right \}.
	$$
	According to the Azuma-Hoeffding inequality, we have that for any $n \geq 1$,
	\begin{align*}
		& \Pr \mbr{ \bar{\cE},\ m_{\tau-1}=n }
		\\
		= & \Pr \Bigg[ \sum_{t=1}^{m_{\tau-1}} r_{t,x_t} < \sum_{t=1}^{m_{\tau-1}} \mu_{x_t} - \sqrt{ 2 m_{\tau-1} \ln \left( 2m_{\tau-1}^3 \right) },\\& \qquad \  m_{\tau-1}=n \Bigg]
		\\ 
		\leq & \frac{1}{2 n^3}.
	\end{align*}

	For the confidence ellipsoid in the \textsf{LinUCB}~\cite{linear_bandit_NIPS2011} algorithm, we set the confidence parameter $\delta_t=1/(2 t^2)$ and the confidence ellipsoid for timestep $t$ is 
	\begin{align*}
		\mathcal{C}_{t}= \{ \theta\in \mathbb{R}^d : & \| \theta-\hat{\theta}_{t-1} \|_{V_{t-1}} \leq \\& \sqrt{ d \ln ( 2 m_t^2 (1+ m_t L^2/\lambda )  ) } + \sqrt{\lambda} S \}.
	\end{align*}
	Then, we can obtain 
	\begin{align*}
	& \mathbb{E}[\tilde{\cR}_{m_{\tau-1}}(\textsf{LinUCB})]
	\\
	= & \ex\left[ \mu_* m_{\tau-1} - \sum_{t=1}^{m_{\tau-1}} \mu_{x_t} \right]
	\\
	\leq &  \ex\Bigg[ 4 \sqrt{ m_{\tau-1} d \ln \left( 1+ \frac{m_{\tau-1} L^2}{\lambda d} \right) } \cdot \\ & \Bigg(  \sqrt{\lambda} S  \!+ \!   \sqrt{d  \ln \left(  2 m_{\tau-1}^2 \!  \cdot \! \left(1+ \frac{  m_{\tau-1} L^2 }{\lambda} \right)  \right)} \Bigg)  + \Delta_{\textup{max}} \Bigg].
	\end{align*}
	
	Fix time horizon $T$. Recall that $\tau \leq T$ is the last timestep when we played $x_0$. Similar to the analysis in CMAB (Eq.~\eqref{eq:introduce_pseudoR}), at timestep $\tau$, we have
	\begin{align*}
	\alpha \mu_0 N_0(\tau-1)  < &  -(\Delta_0 +\alpha \mu_0 ) (m_{\tau-1} +1) + \sum_{t \in \cS_{\tau-1}} \mu_{x_t} \\& - \sum_{t \in \cS_{\tau-1}} r_t     +  \tilde{\cR}_{m_{\tau-1}}(\textsf{LinUCB}) + \mu_* .
	\end{align*}
	Taking expectation of both sides, we have
	\begin{align*}
	\alpha \mu_0 \mathbb{E}[N_0(&\tau-1)]  <  -(\Delta_0 +\alpha \mu_0 ) \mathbb{E}[m_{\tau-1} +1] \\& + \mathbb{E}\left[ \sum_{t \in \cS_{\tau-1}} \mu_{x_t} - \sum_{t \in \cS_{\tau-1}} r_t| \cE \right]\Pr[\cE] 
	\\ & +\! \sum_{n=1}^{\infty} \mathbb{E}\!\!\left[ \!\sum_{t \in \cS_{\tau-1}} \!\!\!\mu_{x_t} \!\!-\!\!\!\!\! \sum_{t \in \cS_{\tau-1}} \!\!\! r_t | \bar{\cE}, m_{\tau-1}\!=\!n \right] \!\!\cdot \nonumber\\& \quad \Pr[\bar{\cE}, m_{\tau-1}=n] \\&  +  \mathbb{E}[\tilde{\cR}_{m_{\tau-1}}(\textsf{LinUCB})] + \mu_*   
	\\
	\leq & -(\Delta_0 +\alpha \mu_0 ) \mathbb{E}[m_{\tau-1} +1] \\& + \mathbb{E}\left[ 4 \sqrt{ m_{\tau-1} \ln \left( m_{\tau-1} \right) } | \cE \right]  +  \sum_{n=1}^{\infty} n \!\cdot\! \frac{1}{2n^3}  \\&  +   \mathbb{E}[\tilde{\cR}_{m_{\tau-1}}(\textsf{LinUCB})] + \mu_*
	\\
	< & -(\Delta_0 +\alpha \mu_0 ) \mathbb{E}[m_{\tau-1} +1] \\& +\! 4 \sqrt{ 2 \mathbb{E}[m_{\tau-1}] \! \ln \! \left(  2 \mathbb{E}[N_0(\tau-1)]  \mathbb{E}[m_{\tau-1} +1] \right) }  \\ &+    4 \sqrt{ \mathbb{E}[m_{\tau-1}] d \ln \left( 1+ \frac{\mathbb{E}[m_{\tau-1}] L^2}{\lambda d} \right) } \cdot \\ \Bigg(  \sqrt{\lambda}& S  \!+\!   \sqrt{d \ln \left(  2 \mathbb{E}[m_{\tau-1}]^2 \! \cdot \! \left(1+ \frac{  \mathbb{E}[m_{\tau-1}] L^2 }{\lambda} \right)  \right)} \Bigg)  \\ & + \Delta_{\textup{max}} + 2
	\\
	< & -(\Delta_0+ \alpha \mu_0)\mathbb{E}[m_{\tau-1} +1]  \\& + 38 d \sqrt{\lambda} S \sqrt{ (\mathbb{E}[m_{\tau-1} +1] )} \cdot \\& \ln \left( 2 \mathbb{E}[N_0(\tau-1)] \cdot \mathbb{E}[m_{\tau-1} +1] \right) +3,
	\end{align*}
	where $\Delta_{\textup{max}}=\max_{x \in \cX} \Delta_{x}$.
	
	Let $m=\mathbb{E}[m_{\tau-1} +1]$, $c_1=38 d \sqrt{\lambda}S$, $c_2=2 \mathbb{E}[N_0(\tau-1)] $, $c_3=\Delta_0+ \alpha \mu_0 $. According to Fact~\ref{fact:clucb_lemma9}, we have
	\begin{align*}
	\alpha \mu_0 \mathbb{E}[N_0(\tau-1)]  \! < \! & \frac{ 23104 d^2 S^2 \lambda}{9(\Delta_0+ \alpha \mu_0)}  \cdot\\& \left[ \! \ln \! \left( \!  \frac{294 d \sqrt{\lambda} S\sqrt{\mathbb{E}[N_0(\tau-1)]} } {\Delta_0+ \alpha \mu_0} \! \right) \! \right]^2 \!\!\! + \! 3.
	\end{align*}
	Thus, we have 
	\begin{align*}
	\mathbb{E}[N_0(\tau-1)]  < &   \frac{ 2568 d^2 S^2 \lambda}{ \alpha \mu_0(\Delta_0+ \alpha \mu_0)}  \cdot\\& \left[\ln \!\! \left( \!\! \frac{294 d \sqrt{\lambda} S\sqrt{\mathbb{E}[N_0(\tau-1)]} } {\Delta_0+ \alpha \mu_0} \! \right) \! \right]^2 
	\\
	\sqrt{\mathbb{E}[N_0(\tau-1)]}  < &  \frac{  51 d \sqrt{\lambda} S }{ \sqrt{ \alpha \mu_0 (\Delta_0+ \alpha \mu_0)} }  \cdot\\& \ln \!\! \left(\frac{294 d \sqrt{\lambda} S\sqrt{\mathbb{E}[N_0(\tau-1)]} } {\Delta_0+ \alpha \mu_0} \! \right)   
	\end{align*}
	
	According to Fact~\ref{fact:clucb_lemma10} (set $z=\sqrt{\mathbb{E}[N_0(\tau-1)]}, c_1=\frac{  51 d \sqrt{\lambda} S }{ \sqrt{ \alpha \mu_0 (\Delta_0+ \alpha \mu_0)} } , c_2=\frac{294 d \sqrt{\lambda} S } {\Delta_0+ \alpha \mu_0} $), 
	\begin{align*}
	 \sqrt{\mathbb{E}[N_0(\tau-1)]}  & \leq   \frac{  102 d \sqrt{\lambda} S }{ \sqrt{ \alpha \mu_0 (\Delta_0+ \alpha \mu_0)} }  \cdot\\& \ln  \left(  \frac{14994 d^2 S^2 \lambda}{(\Delta_0+ \alpha \mu_0)^{\frac{3}{2}} \sqrt{\alpha \mu_0 } }  \right)   
	\end{align*}
	Thus,
	\begin{align*}
	& \ex[N_0(T)] 
	\\
	= & \ex[N_0(\tau)] 
	\\
	= & \ex[N_0(\tau-1)] + 1 
	\\
	= & O \left(  \frac{  d^2 S^2 \lambda }{\alpha \mu_0 (\Delta_0+ \alpha \mu_0)} \left[ \ln \left( \frac{  d S \sqrt{\lambda} }{\alpha \mu_0 (\Delta_0+ \alpha \mu_0)} \right) \right]^2  \right) .
	\end{align*}
	Theorem~\ref{thm:gen_con_linucb} follows from $ \ex[\cR_T(\textsf{GenCB-CLB})] \leq  \ex[\cR_T(\textsf{LinUCB})] + \mathbb{E}[N_0(T)]\Delta_0 $.
\end{proof}

\section{Algorithm Pseudo-code and Proof for CCCB} \label{apx:cccb}
Algorithm~\ref{alg:GenC3UCB} presents the algorithm pseudo-code of \textsf{GenCB-CCCB} for CCCB, and we give the detailed proof of Theorem~\ref{thm:gen_con_c2ucb} in the following.

\begin{algorithm}[t]
	\caption{\textsf{GenCB-CCCB}} \label{alg:GenC3UCB}
	\KwIn{Reugular arms (decision class) $\cX$, base arms $x_1, \dots, x_K \in \mathbb{R}^d$, default arm $x_0$ with reward $\mu_0$, parameter $\alpha$, $L$, $S$, $\lambda \geq \max\{1, L^2\}$.}
	$\forall t \geq 0, N_0(t) \leftarrow 0, r_S(t) \leftarrow 0$. $m \leftarrow 0$. $V_0 \leftarrow \lambda I $, $b_0 \leftarrow \boldsymbol{0}^d $\;
	\For{$t=1, 2, \dots$}
	{
		\If{$r_S(t-1)+N_0(t-1)\mu_0 \geq (1-\alpha) \mu_0 t$} 
		{
			$m \leftarrow m+1$\;
			$\hat{w}_{t,e} \leftarrow x_{e}^{\top} \hat{\theta}_{t-1}, \forall e \in [K]$\;
			$\bar{w}_{t,e} \leftarrow \hat{w}_{t,e} + (\sqrt{ d \ln ( 2 m^2 (1+mKL^2/\lambda )  ) }$ \\ \quad $+ \sqrt{\lambda} S) \| x_{e} \|_{V_{t-1}^{-1}}, \forall e \in [K]$\;
			$A_t \leftarrow \argmax_{A \in \cX} \bar{f}(A, \bar{\boldsymbol{w}}) $\;
			Play arm $A_t$ and observe the random reward $w_{t,e}$ for all $e \in A_t$\;
			$V_t\leftarrow\lambda I + \sum_{s=1}^{t} \sum_{e \in A_s} x_{e} x_{e}^{\top}$\;
			$b_t\leftarrow\sum_{s=1}^{t} \sum_{e \in A_s} w_{s,e} x_{e}$\;
			$\hat{\theta}_t \leftarrow V_t^{-1} b_t$\;
			$r_S(t) \leftarrow r_S(t-1)+r_{t,A_t}$\;
		}
		\Else
		{
			Play $x_0$ and receive reward $\mu_0$\;
			$N_0(t) \leftarrow N_0(t-1)+1$\;
			$r_S(t) \leftarrow r_S(t-1)$\;
			
		}
	}
\end{algorithm}

\begin{proof}
	Since the proof of satisfaction on the performance constraint Eq.~\eqref{eq:perf_constraint} is the same to Theorem~\ref{thm:gen_con_ucb}, we mainly give the proof of regret bound here.
	
	Define event 
	$$
	\cE\!\!:=\! \left \{\! \sum_{t=1}^{m_{\tau-1}} \!\! r_{t, A_t} \!\geq\!\!\! \sum_{t=1}^{m_{\tau-1}} \!\!\! f(A_t, \boldsymbol{w}^*) \!-\! K \!\sqrt{ 2 m_{\tau-1} \! \ln \! \left( 2 m_{\tau-1}^3 \right) } \! \right \} \!\!.
	$$
	According to the Azuma-Hoeffding inequality, we have that for any $n \geq 1$,
	\begin{align*}
		& \Pr \mbr{ \bar{\cE},\ m_{\tau-1}=n }
		\\
		= & \Pr \Bigg[ \sum_{t=1}^{m_{\tau-1}} \!\! r_{t, A_t} \!<\!\!\! \sum_{t=1}^{m_{\tau-1}} \!\!\! f(A_t, \boldsymbol{w}^*) \!-\! K \!\sqrt{ 2 m_{\tau-1} \! \ln \! \left( 2 m_{\tau-1}^3 \right) } ,\\& \qquad \  m_{\tau-1}=n \Bigg]
		\\ 
		\leq & \frac{1}{2 n^3}.
	\end{align*}

	For the confidence ellipsoid in the \textsf{C2UCB}~\cite{Contextual_Combinatorial_Bandit} algorithm, we set the confidence parameter $\delta_t=1/(2 t^2)$ and the confidence ellipsoid for timestep $t$ is  
	\begin{align*}
	\mathcal{C}_{t}= \big\{ \theta \in \mathbb{R}^d : & \| \theta - \hat{\theta}_{t-1} \|_{V_{t-1}} \leq \\& \sqrt{ d \ln ( 2 m^2 (1+mKL^2/\lambda )  ) } + \sqrt{\lambda} S \big\}.
	\end{align*}
	Then, we can obtain 
	\begin{align*}
	& \ex[\cR_{m_{\tau-1}} (\textsf{C2UCB})] 
	\\
	\leq & \ex\Bigg[ \!
	2P \!\! \sqrt{ 2 d m_{\tau-1}  \ln \left( 1+ \frac{m_{\tau-1} K L^2}{\lambda d} \right) } \cdot \\& \left( \! \sqrt{\lambda} S \! + \! \sqrt{d \ln \left( 2 m_{\tau-1}^2 \!\! \cdot \!\! \left( 1+ \frac{ m_{\tau-1} KL^2 }{\lambda} \right)  \right)} \right) \!+\! \Delta_{\textup{max}} \! \Bigg] \!,
	\end{align*}
	where $\Delta_{\textup{max}}=\max_{A \in \cX} (f(A_*, \boldsymbol{w}^*)-f(A, \boldsymbol{w}^*) )$.
	
	Fix time horizon $T$. Recall that $\tau \leq T$ is the last timestep when we played $x_0$.
	Similar to the conservative multi-armed bandit case (Eq.~\eqref{eq:introduce_pseudoR}), at timestep $\tau$, we have
	\begin{align*}
	\alpha \mu_0 N_0(\tau-1)   < & -(\Delta_0 +\alpha \mu_0 ) (m_{\tau-1} +1) \\& + \sum_{t \in S_{\tau-1}} f(A_t, \boldsymbol{w}^*) - \sum_{t \in S_{\tau-1}} r_{t,A_t}  \\&   +  \tilde{\cR}(m_{\tau-1}) + \mu_* .
	\end{align*}
	Taking expectation of both sides, we have
	\begin{align*}
	&\alpha \mu_0 \mathbb{E}[N_0(\tau-1)]  
	\\
	< & -(\Delta_0 +\alpha \mu_0 ) \mathbb{E}[m_{\tau-1} +1] \\& + \mathbb{E}\left[ \sum_{t \in S_{\tau-1}} f(A_t, \boldsymbol{w}^*) - \sum_{t \in S_{\tau-1}} r_{t,A_t} | \cE \right]\Pr[\cE] \\& 
	+\! \sum_{n=1}^{\infty} \mathbb{E}\!\!\left[ \sum_{t \in S_{\tau-1}} f(A_t, \boldsymbol{w}^*) - \sum_{t \in S_{\tau-1}} r_{t,A_t} | \bar{\cE}, m_{\tau-1}\!=\!n \right] \!\!\cdot \nonumber\\& \quad \Pr[\bar{\cE}, m_{\tau-1}=n]
	\\& + \mathbb{E}[\tilde{\cR}(m_{\tau-1})] + \mu_*   
	\\
	\leq & \!-\!(\Delta_0 \!+\! \alpha \mu_0 ) \mathbb{E}[m_{\tau-1} \!+\! 1] \!+\! \mathbb{E}\!\left[ 4 K\! \sqrt{ m_{\tau-1} \ln \left( m_{\tau-1} \right) } | \cE \right] \\& +  \sum_{n=1}^{\infty} \! n \!\cdot\! \frac{1}{2n^3}  + \mathbb{E}[\tilde{\cR}(m_{\tau-1})] + \mu_*
	\\
	< & -(\Delta_0 +\alpha \mu_0 ) \mathbb{E}[m_{\tau-1} +1] \\& +  4 K \sqrt{ 2 \mathbb{E}[m_{\tau-1}] \ln \left( 2 \mathbb{E}[N_0(\tau-1)]  \mathbb{E}[m_{\tau-1} +1] \right) }  \\& +  2P \sqrt{ 2 d \mathbb{E}[m_{\tau-1}]  \ln \left( 1+ \frac{\mathbb{E}[m_{\tau-1}] K L^2}{\lambda d} \right) } \cdot \\& \!\! \left( \!\! \sqrt{\lambda} S \!+\!  \sqrt{d \ln \left( 2 \mathbb{E}[m_{\tau-1}]^2 \cdot \left( 1+ \frac{ \mathbb{E}[m_{\tau-1}] KL^2 }{\lambda} \right)  \right)} \right)  \\ & + \Delta_{\textup{max}} + 2
	\\
	< & -(\Delta_0+ \alpha \mu_0)\mathbb{E}[m_{\tau-1} +1]   + (4\sqrt{2}K+10P\sqrt{\lambda}Sd) \cdot \\& \!\! \sqrt{ (\mathbb{E}[m_{\tau-1} +1] )} \ln \left( 2K  \mathbb{E}[N_0(\tau-1)] \! \cdot \! \mathbb{E}[m_{\tau-1} +1] \right) \!+\! 3.
	\end{align*}
	
	Let $m=\mathbb{E}[m_{\tau-1} +1]$, $c_1=4\sqrt{2}K+10P\sqrt{\lambda}Sd$, $c_2=2K  \mathbb{E}[N_0(\tau-1)] $, $c_3=\Delta_0+ \alpha \mu_0 $. According to Fact~\ref{fact:clucb_lemma9}, we have
	\begin{align*}
	&\alpha \mu_0 \mathbb{E}[N_0(\tau-1)]  <  \frac{ 16 (4\sqrt{2}K+10P\sqrt{\lambda}Sd)^2}{9(\Delta_0+ \alpha \mu_0)} \cdot \\& \left[\ln \left(\frac{8 (4\sqrt{2}K+10P\sqrt{\lambda}Sd) \sqrt{K \mathbb{E}[N_0(\tau-1)]} } {\Delta_0+ \alpha \mu_0} \right) \right]^2 +3.
	\end{align*}
	Thus, we have 
	\begin{align*}
	 &\mathbb{E}[N_0(\tau-1)]  <  \frac{ 5 (4\sqrt{2}K+10P\sqrt{\lambda}Sd)^2}{\alpha \mu_0(\Delta_0+ \alpha \mu_0)} \cdot \\   \Bigg[\ln &\left(\frac{8 K(4\sqrt{2}K+10P\sqrt{\lambda}Sd) \sqrt{ \mathbb{E}[N_0(\tau-1)]} } {\Delta_0+ \alpha \mu_0} \right) \Bigg]^2 
	\\
	 &\sqrt{\mathbb{E}[N_0(\tau-1)]}  <   \frac{ 3 (4\sqrt{2}K+10P\sqrt{\lambda}Sd)}{ \sqrt{\alpha \mu_0(\Delta_0+ \alpha \mu_0)} } \cdot \\& \ln \left(\frac{8 K(4\sqrt{2}K+10P\sqrt{\lambda}Sd) \sqrt{ \mathbb{E}[N_0(\tau-1)]} } {\Delta_0+ \alpha \mu_0} \right)  
	\end{align*}
	
	According to Fact~\ref{fact:clucb_lemma10} (set $z=\sqrt{\mathbb{E}[N_0(\tau-1)]}, c_1=\frac{ 3 (4\sqrt{2}K+10P\sqrt{\lambda}Sd)}{ \sqrt{\alpha \mu_0(\Delta_0+ \alpha \mu_0)} } , c_2=\frac{8 K(4\sqrt{2}K+10P\sqrt{\lambda}Sd) } {\Delta_0+ \alpha \mu_0} $), 
	\begin{align*}
	\sqrt{\mathbb{E}[N_0(\tau-1)]} \leq & \frac{ 12 (4\sqrt{2}K+10P\sqrt{\lambda}Sd)}{ \sqrt{\alpha \mu_0(\Delta_0+ \alpha \mu_0)} } \cdot \\& \ln \left( \frac{5 K(4\sqrt{2}K+10P\sqrt{\lambda}Sd)  } {(\Delta_0+ \alpha \mu_0) \sqrt{\alpha \mu_0(\Delta_0+ \alpha \mu_0)} } \right)   
	\end{align*}
	Thus, 
	\begin{align*}
	&\mathbb{E}[N_0(T)]
	\\
	= & \mathbb{E}[N_0(\tau)] 
	\\
	= & \mathbb{E}[N_0(\tau-1)] + 1 
	\\
	= & O \left( \frac{  (K+P\sqrt{\lambda}Sd)^2 }{ \alpha \mu_0(\Delta_0+ \alpha \mu_0) }  \left[ \ln \left( \frac{ K+P\sqrt{\lambda}Sd  } {\alpha \mu_0(\Delta_0+ \alpha \mu_0) } \right)  \right]^2  \right) .
	\end{align*}
	Theorem~\ref{thm:gen_con_c2ucb} follows from $ \ex[\cR_T(\textsf{GenCB-CCCB})] \leq  \ex[\cR_T(\textsf{C2UCB})] + \mathbb{E}[N_0(T)]\Delta_0 $.
\end{proof}

\section{Proof for MV-CBP} \label{apx:mv_cbp}
We give the detailed proof of Theorem~\ref{thm:con_mv_ucb} below.
\begin{proof}
In order to prove that \textsf{MV-CUCB} satisfies the sample-path reward constraint Eq.~\eqref{eq:mv_perf_constraint}, we give the following inequalities first. For any time horizon $T$, 
\begin{align}
&T \cdot \widehat{\mv}_T(\cA) 
\nonumber\\
= & T \cdot \left( \rho \hat{\mu}_{T}(\cA)-\hat{\sigma}_{T}^2(\cA) \right)
\nonumber\\
= & T \cdot  \frac{\rho}{T} \sum_{t=1}^{T} r_{t,x_t} - T \cdot \frac{1}{T} \sum_{t=1}^{T} r_{t,x_t}^2  + T \cdot \left(\frac{\sum_{t=1}^{T} r_{t,x_t}}{T} \right)^2 
\nonumber\\
= & \rho \sum_{t=1}^{T-1} r_{t,x_t}+ \rho r_{T,x_T} - \sum_{t=1}^{T-1} r_{t,x_t}^2 -  r_{T,x_T}^2 \nonumber\\& + \frac{1}{T} \left( \left(\sum_{t=1}^{T-1} r_{t,x_t}\right)^2 + r_{T,x_T}^2 + 2 \left( \sum_{t=1}^{T-1} r_{t,x_t} \right) r_{T,x_T} \right)
\nonumber\\
= & \rho \sum_{t=1}^{T-1} r_{t,x_t} - \sum_{t=1}^{T-1} r_{t,x_t}^2 + \frac{1}{T-1} \left(\sum_{t=1}^{T-1} r_{t,x_t}\right)^2 \nonumber\\& -\frac{1}{T-1} \left(\sum_{t=1}^{T-1} r_{t,x_t}\right)^2  + \rho r_{T,x_T}-  r_{T,x_T}^2 \nonumber\\& + \frac{1}{T} \left( \left(\sum_{t=1}^{T-1} r_{t,x_t}\right)^2 + r_{T,x_T}^2 + 2 \left( \sum_{t=1}^{T-1} r_{t,x_t} \right) r_{T,x_T} \right)
\nonumber\\
\geq & (T-1) \cdot \widehat{\mv}_{T-1}(\cA) -1 - \frac{1}{T(T-1)} \left(\sum_{t=1}^{T-1} r_{t,x_t}\right)^2 
\nonumber\\
\geq  & (T-1) \cdot \widehat{\mv}_{T-1}(\cA) -2 \label{eq:t_times_mv_bound}
\end{align}	
	
Now we prove that \textsf{MV-CUCB} satisfies the sample-path reward constraint Eq.~\eqref{eq:mv_perf_constraint} by induction.
At timestep $t=1$, since the LHS of the \textsf{if} statement (in Line 3 of Algorithm \ref{alg:GenCUCB}) is $-2$ and RHS is positive, \textsf{MV-CUCB} will pull the default arm $x_0$ and receive reward $\mu_0$. Then, we have $\widehat{\mv}_1(\cA) = {\mv}_0 \geq (1-\alpha) {\mv}_0$, which satisfies the constraint.
Suppose that the sample-path reward constraint holds at timestep $t-1$. At time step $t$, if \textsf{MV-CUCB} plays $x_0$, since the exploration risk caused by one pull is bounded by $2$ and $\alpha {\mv}_0>2$, the constraint still holds for $t$. If \textsf{MV-CUCB} plays a regular arm $x_t$, which implies $ (t-1)\widehat{\mv}_{t-1}(\cA) - 2 \geq (1-\alpha) {\mv}_0 t$, then from  Eq.~\eqref{eq:t_times_mv_bound} we have $t\widehat{\mv}_{t}(\cA) \geq (t-1)\widehat{\mv}_{t-1}(\cA) - 2 \geq (1-\alpha) {\mv}_0 t$, and thus the constraint still holds for $t$.

Next, we prove the regret bound of the \textsf{MV-CUCB} algorithm.
Fix a time horizon $T$.
Let $\tau \leq T$ denote the last timestep when the algorithm pulled arm $x_0$. For ease of notation, we use $N_{i,\tau -1}$ as a shorthand for $N_i(\tau -1)$, $\forall 0 \leq i \leq K$.

Define event 
\begin{align*}
	& \cF:= \Bigg\{ \forall i=1, \dots, K, |\hat{\mu}_{i,\tau-1} - \mu_i| \! \leq \!\! \sqrt{\frac{\ln (12Km_{\tau-1}^4)}{2N_{i,\tau-1}}} , \\& |\hat{\sigma}_{i,\tau-1}^2 - \sigma_i^2| \! \leq \!\! 5 \sqrt{\frac{\ln (12Km_{\tau-1}^4)}{2 N_{i,\tau-1}}}   \Bigg\} \!.
\end{align*}
Similar to Lemma 2 in \cite{sani2012risk}, for any $n \geq 1$ and $m_{\tau-1}=n$, using the Chernoff-Hoeffding inequality and a union bound over $N_{i,\tau-1} \in [n]$ and $i \in [K]$, we have 
$$
\Pr [ \bar{\cF},m_{\tau-1}=n ] \leq  \frac{1}{2 n^3} .
$$

Conditioning on $\cF$, we have
\begin{align*}
& \sum_{i=1}^{K} N_{i,\tau -1} \mv_i - \frac{2}{\tau -1} \sum_{i=1}^{K} \sum_{\begin{subarray}{l} j \neq i \\ j \neq 0 \end{subarray}} N_{i,\tau -1} N_{j,\tau -1} \Gamma_{i,j}^2  
\\ 
\leq & \sum_{i=1}^{K} N_{i,\tau -1} \left( \widehat{\mv}_i + 2(5+\rho)\sqrt{\frac{\log(12Km_{\tau-1}^4)}{2 N_{i,\tau -1}}} \right) \\ & 
- \Bigg( \frac{2}{\tau -1} \sum_{i=1}^{K} \sum_{\begin{subarray}{l} j \neq i \\ j \neq 0 \end{subarray}} N_{i,\tau -1} N_{j,\tau -1} \Gamma_{i,j}^2  \\&+ \frac{2\sqrt{2}}{\tau-1} \sum_{i=1}^{K} \sum_{\begin{subarray}{l} j \neq i \\ j \neq 0 \end{subarray}} N_{j,\tau -1}\log(12Km_{\tau-1}^4)  \\& + \frac{2\sqrt{2}}{\tau-1} \sum_{i=1}^{K} \sum_{\begin{subarray}{l} j \neq i \\ j \neq 0 \end{subarray}} N_{i,\tau -1}\log(12Km_{\tau-1}^4) \Bigg) \\& 
+ \frac{2\sqrt{2}}{\tau-1} \sum_{i=1}^{K} \sum_{\begin{subarray}{l} j \neq i \\ j \neq 0 \end{subarray}} N_{j,\tau -1}\log(12Km_{\tau-1}^4) \\& + \frac{2\sqrt{2}}{\tau-1} \sum_{i=1}^{K} \sum_{\begin{subarray}{l} j \neq i \\ j \neq 0 \end{subarray}} N_{i,\tau -1}\log(12Km_{\tau-1}^4)
\\ 
\leq & \sum_{i=1}^{K} N_{i,\tau -1}  \widehat{\mv}_i + (5+\rho)\sum_{i=1}^{K} \sqrt{2 N_{i,\tau -1} \log(12Km_{\tau-1}^4)}  \\ & 
- \frac{1}{\tau -1} \sum_{i=1}^{K} \sum_{\begin{subarray}{l} j \neq i \\ j \neq 0 \end{subarray}} N_{i,\tau -1} N_{j,\tau -1} \cdot \\& \left( |\Gamma_{i,j}|  + \sqrt{\frac{\log(12Km_{\tau-1}^4)}{2 N_{i,\tau -1}}}+ \sqrt{\frac{\log(12Km_{\tau-1}^4)}{2 N_{j,\tau -1}}}\right)^2 \\& + 4\sqrt{2}K\log(12Km_{\tau-1}^4) 
\\ 
\leq & \sum_{i=1}^{K} N_{i,\tau -1}  \widehat{\mv}_i - \frac{1}{\tau -1} \sum_{i=1}^{K} \sum_{\begin{subarray}{l} j \neq i \\ j \neq 0 \end{subarray}} N_{i,\tau -1} N_{j,\tau -1} \hat{\Gamma}_{i,j}^2 \\& + (5+\rho) \sqrt{2 K m_{\tau-1}   \log(12Km_{\tau-1}^4)}    \\& + 4\sqrt{2}K\log(12Km_{\tau-1}^4)
\end{align*}

Let $L=2$, ${\mv}_* \leq \rho$, $\Delta^\mv_{\textup{max}}\leq \frac{1}{4}+\rho$ and $\textup{GAP}_{\textup{max}} \leq \frac{5}{4}+\rho $. We set the confidence parameter $\delta_t=1/(12Kt^3)$ in the \textsf{MV-UCB}~\cite{sani2012risk} algorithm and use Jensen's inequality, and then we have
\begin{align*}
	& \mathbb{E}[m_{\tau-1}\tilde{\cR}_{m_{\tau-1}}(\textsf{MV-UCB})] \\ \leq & \ex \Bigg[ 12(5+\rho)^2 (H^\mv_1+4H^\mv_2) \ln(6 K m_{\tau-1} )  \\& + 288 (5+\rho)^4 H^\mv_3 \frac{ \ln^2(6 K m_{\tau-1} )}{m_{\tau-1}} + 9K + K \Delta^\mv_{\textup{max}} \Bigg] 
	\\ \leq & 12(5+\rho)^2 (H^\mv_1+4H^\mv_2) \ln(6 K \mathbb{E}[m_{\tau-1}] ) \\& + 288 (5+\rho)^4 H^\mv_3 \frac{ \ln^2(6 K \mathbb{E}[m_{\tau-1}] )}{\mathbb{E}[m_{\tau-1}]} + 9K + K \Delta^\mv_{\textup{max}} .
\end{align*}

At timestep $\tau$, we have
\begin{align*}
(\tau -1)\widehat{\mv}_{\tau-1}(\cA)-L & < (1-\alpha) {\mv}_0 \tau.
\end{align*}
Thus,
\begin{align*}
&\sum_{i=0}^{K} N_{i,\tau -1} \widehat{\mv}_i  - \frac{1}{\tau -1}  \sum_{i=0}^{K} \sum_{j \neq i} N_{i,\tau -1} N_{j,\tau -1} \hat{\Gamma}_{i,j}^2 -L \\ & < (1-\alpha) {\mv}_0 (m_{\tau-1} + N_0(\tau -1) +1)
\end{align*}
Rearranging the terms, we have
\begin{align*}
& \alpha {\mv}_0 N_0(\tau -1) 
\\
\leq &  (1-\alpha) {\mv}_0 (m_{\tau-1} + 1) \\& - \left(\! \sum_{i=1}^{K} N_{i,\tau -1} \widehat{\mv}_i - \frac{1}{\tau -1} \sum_{i=1}^{K} \sum_{\begin{subarray}{l} j \neq i \\ j \neq 0 \end{subarray}} N_{i,\tau -1} N_{j,\tau -1} \hat{\Gamma}_{i,j}^2 \!\right) \\ & + \frac{2}{\tau -1} N_0(\tau -1) \sum_{i=1}^{K} N_{i,\tau -1} \hat{\Gamma}_{0,i}^2  + L
\\ 
\leq &  (1-\alpha) {\mv}_0 (m_{\tau-1} + 1) \\& - \left(\! \sum_{i=1}^{K} N_{i,\tau -1} \widehat{\mv}_i - \frac{1}{\tau -1} \sum_{i=1}^{K} \sum_{\begin{subarray}{l} j \neq i \\ j \neq 0 \end{subarray}} N_{i,\tau -1} N_{j,\tau -1} \hat{\Gamma}_{i,j}^2 \!\right) \\ & - m_{\tau-1}\tilde{\cR}_{m_{\tau-1}}(\textsf{MV-UCB}) + m_{\tau-1}\tilde{\cR}_{m_{\tau-1}}(\textsf{MV-UCB}) \\& + 2 N_0(\tau -1)  + L
\\ 
\leq &  (1-\alpha) {\mv}_0 (m_{\tau-1} + 1) \\& - \left(\! \sum_{i=1}^{K} N_{i,\tau -1} \widehat{\mv}_i - \frac{1}{\tau -1} \sum_{i=1}^{K} \sum_{\begin{subarray}{l} j \neq i \\ j \neq 0 \end{subarray}} N_{i,\tau -1} N_{j,\tau -1} \hat{\Gamma}_{i,j}^2 \!\right) \\ & - {\mv}_* m_{\tau-1} + \sum_{i=1}^{K} N_{i,\tau -1} \mv_i \\& - \frac{2}{m_{\tau-1}} \sum_{i=1}^{K} \sum_{\begin{subarray}{l} j \neq i \\ j \neq 0 \end{subarray}} N_{i,\tau -1} N_{j,\tau -1} \Gamma_{i,j}^2 \\& + m_{\tau-1}\tilde{\cR}_{m_{\tau-1}}(\textsf{MV-UCB}) + 2 N_0(\tau -1)  + L
\\ 
\leq & - ({\mv}_*-(1-\alpha) {\mv}_0  ) (m_{\tau-1} + 1) \\& - \left(\! \sum_{i=1}^{K} N_{i,\tau -1} \widehat{\mv}_i - \frac{1}{\tau -1} \sum_{i=1}^{K} \sum_{\begin{subarray}{l} j \neq i \\ j \neq 0 \end{subarray}} N_{i,\tau -1} N_{j,\tau -1} \hat{\Gamma}_{i,j}^2 \!\right) \\ &  + \left(\! \sum_{i=1}^{K} N_{i,\tau -1} \mv_i - \frac{2}{\tau -1} \sum_{i=1}^{K} \sum_{\begin{subarray}{l} j \neq i \\ j \neq 0 \end{subarray}} N_{i,\tau -1} N_{j,\tau -1} \Gamma_{i,j}^2  \!\right) \\& + m_{\tau-1}\tilde{\cR}_{m_{\tau-1}}(\textsf{MV-UCB}) + 2 N_0(\tau -1)   +L + {\mv}_*
\\ 
\leq & - (\Delta^\mv_0 + \alpha {\mv}_0 ) (m_{\tau-1} + 1) \\& + \left(\! \sum_{i=1}^{K} N_{i,\tau -1} \mv_i - \frac{2}{\tau -1} \sum_{i=1}^{K} \sum_{\begin{subarray}{l} j \neq i \\ j \neq 0 \end{subarray}} N_{i,\tau -1} N_{j,\tau -1} \Gamma_{i,j}^2 \!\right) \\ &  - \left(\! \sum_{i=1}^{K} N_{i,\tau -1} \widehat{\mv}_i - \frac{1}{\tau -1} \sum_{i=1}^{K} \sum_{\begin{subarray}{l} j \neq i \\ j \neq 0 \end{subarray}} N_{i,\tau -1} N_{j,\tau -1} \hat{\Gamma}_{i,j}^2 \!\right) \\& + m_{\tau-1}\tilde{\cR}_{m_{\tau-1}}(\textsf{MV-UCB}) + 2 N_0(\tau -1)   +L + {\mv}_*
\end{align*}
Taking expectation of both sides, we have
\begin{align*}
& (\alpha {\mv}_0-2) \mathbb{E}[N_0(\tau -1)] 
\\
\leq & - (\Delta^\mv_0 + \alpha {\mv}_0 ) \mathbb{E}[m_{\tau-1} + 1]  \\& + \mathbb{E}\Bigg[\Bigg( \sum_{i=1}^{K} N_{i,\tau -1} \mv_i \\&- \frac{2}{\tau -1} \sum_{i=1}^{K} \sum_{\begin{subarray}{l} j \neq i \\ j \neq 0 \end{subarray}} N_{i,\tau -1} N_{j,\tau -1} \Gamma_{i,j}^2 \Bigg) \\ & \!\!-\!\! \Bigg(\sum_{i=1}^{K} N_{i,\tau -1} \widehat{\mv}_i \\&- \frac{1}{\tau -1} \sum_{i=1}^{K} \sum_{\begin{subarray}{l} j \neq i \\ j \neq 0 \end{subarray}} N_{i,\tau -1} N_{j,\tau -1} \hat{\Gamma}_{i,j}^2 \Bigg)   | \cF \Bigg] \!\! \Pr[\cF] \\ & +  \textup{GAP}_{\textup{max}}\sum_{n=1}^{\infty} \! n \!\cdot\! \frac{1}{2n^3} + \mathbb{E}[m_{\tau-1}\tilde{\cR}_{m_{\tau-1}}(\textsf{MV-UCB})] \\& +L + {\mv}_*
\\
\leq & - (\Delta^\mv_0 + \alpha {\mv}_0 ) \mathbb{E}[m_{\tau-1} + 1]  \\ & + \mathbb{E}[(5+\rho)\sqrt{ 2K m_{\tau-1} \ln(12Km_{\tau-1}^4) } \\&+ 4 \sqrt{2} K \ln (12Km_{\tau-1}^4) | \cF ] \\ & +  \textup{GAP}_{\textup{max}} + \mathbb{E}[m_{\tau-1}\tilde{\cR}_{m_{\tau-1}}(\textsf{MV-UCB})] \\& +L + {\mv}_*
\\
< & - (\Delta^\mv_0 + \alpha {\mv}_0 ) \mathbb{E}[m_{\tau-1} + 1] \\ & + 4(5+\rho) \cdot\\& \sqrt{ 2K \mathbb{E}[m_{\tau-1} + 1] \ln(6 K ( \mathbb{E}[N_0(\tau -1)]  \mathbb{E}[m_{\tau-1} + 1])  ) } \\ & + 16 \sqrt{2} K \ln (6 K  ( \mathbb{E}[N_0(\tau -1)]  \mathbb{E}[m_{\tau-1} + 1]) ) \\ &  + 12(5+\rho)^2 (H^\mv_1+4H^\mv_2) \ln(6 K \mathbb{E}[m_{\tau-1}] ) \\& + 288 (5+\rho)^4 H^\mv_3 \frac{ \ln^2(6 K \mathbb{E}[m_{\tau-1}] )}{\mathbb{E}[m_{\tau-1}]} \\ & + 9K + K \Delta^\mv_{\textup{max}}  +  \textup{GAP}_{\textup{max}} +L + {\mv}_*
\\
< & - (\Delta^\mv_0 + \alpha {\mv}_0 ) \mathbb{E}[m_{\tau-1} + 1] \\ & + 40K(5+\rho)  \cdot\\& \sqrt{ \mathbb{E}[m_{\tau-1} + 1] } \! \ln(6 K  \mathbb{E}[N_0(\tau -1)]  \mathbb{E}[m_{\tau-1} + 1]  )  \\ &   \!+\!  144K(5 \!+\! \rho)^2 (H^\mv_1 \!+\! 4H^\mv_2)  \ln(6 K \mathbb{E}[m_{\tau-1}] ) \\& + 864(5+\rho)^4 K H^\mv_3 + (13 + 3\rho)K 
\end{align*}
Let $m=\mathbb{E}[m_{\tau-1} +1] \geq 2$, $c_1=40K(5+\rho)$, $c_2=6 K  \mathbb{E}[N_0(\tau -1)] $, $c_3=\Delta^\mv_0 + \alpha {\mv}_0 \in (2,  \rho)$, $c_4=144K(5+\rho)^2 (H^\mv_1+4H^\mv_2)$ where $\rho>\frac{2}{\alpha \mu_0}>2$, $c_4>3c_1$, $c_4>12c_3$. The RHS of the above inequality can be written as a constant term plus
$$
g_2(m)=-c_3m+c_1\sqrt{m}\ln(c_2 m)+c_4\ln(m).
$$
According to Lemma~\ref{lemma:g_m_ub_mv}, we have
$$
g_2(m) \leq \frac{48 c_1 c_4 }{c_3}   \left [ \ln( \frac{3 \sqrt{c_2} c_4}{c_3}) \right]^2 .
$$
Then, we have
\begin{align*}
&(\alpha {\mv}_0-2) \mathbb{E}[N_0(\tau -1)] \\ \leq & \frac{48 \!\cdot\! 40K(5 \!+\! \rho) \cdot 144K(5 \!+\! \rho)^2 (H^\mv_1 \!+\! 4H^\mv_2)  }{\Delta^\mv_0 + \alpha {\mv}_0}  \cdot \\ &  \Bigg [ \ln \bigg( 3 \sqrt{ 6 K \mathbb{E}[N_0(\tau -1)]  } \cdot\\& 144K(5+\rho)^2 (H^\mv_1+4H^\mv_2) \bigg) \\&- \ln( \Delta^\mv_0 + \alpha {\mv}_0 )  \Bigg]^2  + 864(5+\rho)^4 K H^\mv_3 \\ &+ (13 + 3\rho)K 
\\
 \leq & \Bigg(\! \frac{48 \cdot 40 \cdot 144 (5+\rho)^3 K^2 (H^\mv_1 \!+\! 4H^\mv_2)   }{\Delta^\mv_0 + \alpha {\mv}_0}  \\& + 864(5+\rho)^4 K H^\mv_3 + (13 + 3\rho)K  \!\Bigg) \cdot   \Bigg[ \ln \bigg( 3 \sqrt{ 6 K } \cdot\\& 144K(5+\rho)^2 (H^\mv_1 \!+\! 4H^\mv_2) \!\sqrt{ \mathbb{E}[N_0(\tau -1)] }\bigg) \\&- \ln( \Delta^\mv_0 + \alpha {\mv}_0 )  \Bigg]^2 .
\end{align*}
Thus, we have 
\begin{align*}
& \mathbb{E}[N_0(\tau-1)] \\ \leq & \Bigg(\frac{48 \cdot 40 \cdot 144 (5+\rho)^3 K^2 (H^\mv_1 \!+\! 4H^\mv_2)   }{ (\alpha {\mv}_0-2) (\Delta^\mv_0 + \alpha {\mv}_0) } \\& + \frac{ 864(5+\rho)^4 K H^\mv_3 + (13 + 3\rho)K }{\alpha {\mv}_0-2} \Bigg) \cdot \\ &  \Bigg[ \ln \bigg( 3 \sqrt{ 6 K } \cdot 144K(5+\rho)^2 (H^\mv_1 \!+\! 4H^\mv_2) \cdot\\& \!\!\sqrt{ \mathbb{E}[N_0(\tau -1)] }\bigg) - \ln( \Delta^\mv_0 + \alpha {\mv}_0 )  \Bigg]^2 ,
\\
& \sqrt{\mathbb{E}[N_0(\tau-1)]} \\ \leq & \Bigg(\frac{48 \cdot 40 \cdot 144 (5+\rho)^3 K^2 (H^\mv_1 \!+\! 4H^\mv_2)   }{ (\alpha {\mv}_0-2) (\Delta^\mv_0 + \alpha {\mv}_0) } \\& + \frac{ 864(5+\rho)^4 K H^\mv_3 + (13 + 3\rho)K }{\alpha {\mv}_0-2} \Bigg)^{\frac{1}{2}}  \cdot \\ & \Bigg[ \ln \bigg( 3 \sqrt{ 6 K } \cdot 144K(5+\rho)^2 (H^\mv_1 \!+\! 4H^\mv_2) \cdot\\& \sqrt{ \mathbb{E}[N_0(\tau -1)] }\bigg) - \ln( \Delta^\mv_0 + \alpha {\mv}_0 )  \Bigg] .
\end{align*}
According to Fact~\ref{fact:clucb_lemma10} with \\ $z=\sqrt{\mathbb{E}[N_0(\tau-1)]}$, \\ $c_1=( \frac{48 \cdot 40 \cdot 144 (5+\rho)^3 K^2 (H^\mv_1+4H^\mv_2)   }{ (\alpha {\mv}_0-2) (\Delta^\mv_0 + \alpha {\mv}_0) } \\ \hspace*{2em} + \frac{ 864(5+\rho)^4 K H^\mv_3 + (13 + 3\rho)K }{\alpha {\mv}_0-2} )^{\frac{1}{2}}$ and \\ $c_2=\frac{3 \sqrt{ 6 K } \cdot 144K(5+\rho)^2 (H^\mv_1+4H^\mv_2)  }{ \Delta^\mv_0 + \alpha {\mv}_0 }$, 
we have
\begin{align*}
&\sqrt{\mathbb{E}[N_0(\tau-1)]} \\ \leq & 2 \Bigg(\frac{48 \cdot 40 \cdot 144 (5+\rho)^3 K^2 (H^\mv_1 \!+\! 4H^\mv_2)   }{ (\alpha {\mv}_0-2) (\Delta^\mv_0 + \alpha {\mv}_0) } \\& + \frac{ 864(5+\rho)^4 K H^\mv_3 + (13 + 3\rho)K }{\alpha {\mv}_0-2} \Bigg)^{\frac{1}{2}} \cdot \\ & \!\!\!\!\ln \! \Bigg( \!\! \Bigg( \!\!\frac{48 \cdot 40 \cdot 144 (5+\rho)^3 K^2 (H^\mv_1 \!+\! 4H^\mv_2)   }{ (\alpha {\mv}_0-2) (\Delta^\mv_0 + \alpha {\mv}_0) } \\& + \frac{ 864(5+\rho)^4 K H^\mv_3 + (13 + 3\rho)K }{\alpha {\mv}_0-2} \Bigg)^{\frac{1}{2}} \cdot \\ & \frac{3 \sqrt{ 6 K } \cdot 144K(5+\rho)^2 (H^\mv_1+4H^\mv_2)  }{ \Delta^\mv_0 + \alpha {\mv}_0 }  \Bigg) .
\end{align*}
Thus, 
\begin{align*}
& \mathbb{E}[N_0(\tau)] \\ = & \mathbb{E}[N_0(\tau-1)] + 1 
\\
= & O \! \Bigg( \!  \frac{  \rho^3 K^2 (H^\mv_1 \!\!+\! 4H^\mv_2) \!+\! ( \rho^4 K H^\mv_3 \!\!+\! \rho K)\tilde{\Delta}^\mv_0  }{ (\alpha {\mv}_0-2) \tilde{\Delta}^\mv_0 }   \cdot 
\\ & \Bigg[ \ln \Bigg( \frac{  \rho^3 K^2 (H^\mv_1+4H^\mv_2) + ( \rho^4 K H^\mv_3 +\rho K)\tilde{\Delta}^\mv_0  }{ (\alpha {\mv}_0-2) \tilde{\Delta}^\mv_0 } \cdot \\& \frac{ \rho^2 K\sqrt{K} (H^\mv_1+4H^\mv_2)  }{ \tilde{\Delta}^\mv_0 }  \Bigg) \Bigg]^2  \Bigg) ,
\end{align*}
where $\tilde{\Delta}^\mv_0=\Delta^\mv_0 + \alpha {\mv}_0$.

Theorem~\ref{thm:con_mv_ucb} follows from $ \ex[\cR_T(\textsf{MV-CUCB})] \leq  \ex[\cR_T(\textsf{MV-UCB})] + \frac{\mathbb{E}[N_0(T)]}{T} \Delta^\mv_0 $.
\end{proof}